\documentclass[11pt]{article}

\usepackage[a4paper,margin=1in]{geometry}
\usepackage[T1]{fontenc}
\usepackage[utf8]{inputenc}
\usepackage{lmodern}
\usepackage{microtype}
\usepackage{amsmath,amssymb,amsthm,mathtools,bm}
\usepackage{dsfont}
\usepackage{graphicx}
\usepackage{booktabs}
\usepackage{longtable}
\usepackage{array}
\usepackage{enumitem}
\usepackage{xcolor}
\usepackage[numbers,sort&compress]{natbib}
\usepackage[hidelinks]{hyperref}

\newif\ifusecolors
\usecolorstrue

\newcommand{\mgreen}[1]{\ifusecolors{\color{green!60!black}#1}\else{#1}\fi}
\newcommand{\mred}[1]{\ifusecolors{\color{red!70!black}#1}\else{#1}\fi}
\newcommand{\mblue}[1]{\ifusecolors{\color{blue!70!black}#1}\else{#1}\fi}
\newcommand{\mbgreen}[1]{\ifusecolors{\color{green!60!black}\bm{#1}}\else{\bm{#1}}\fi}
\newcommand{\mbred}[1]{\ifusecolors{\color{red!70!black}\bm{#1}}\else{\bm{#1}}\fi}
\newcommand{\mbblue}[1]{\ifusecolors{\color{blue!70!black}\bm{#1}}\else{\bm{#1}}\fi}

\DeclareMathOperator*{\argmin}{arg\,min}
\newcommand{\E}{\mathbb{E}}

\theoremstyle{plain}
\newtheorem{theorem}{Theorem}[section]
\newtheorem{lemma}[theorem]{Lemma}
\newtheorem{proposition}[theorem]{Proposition}
\newtheorem{corollary}[theorem]{Corollary}

\theoremstyle{definition}
\newtheorem{definition}[theorem]{Definition}

\title{Intensity Dot Product Graphs}
\author{
  Giulio Valentino Dalla Riva\thanks{Baffelan OU, Noumea, New Caledonia; \texttt{me@gvdallariva.net}}
  \and
  Matteo Dalla Riva\thanks{Dipartimento di Tecnica e Gestione dei Sistemi Industriali, Universit\`a di Padova, Vicenza, Italy}
}
\date{}

\begin{document}
\maketitle

\begin{abstract}
Latent-position random graph models usually treat the node set as fixed once the sample size is chosen, while graphon-based and random-measure constructions allow more randomness at the cost of weaker geometric interpretability. We introduce \emph{Intensity Dot Product Graphs} (IDPGs), which extend Random Dot Product Graphs by replacing a fixed collection of latent positions with a Poisson point process on a Euclidean latent space. This yields a model with random node populations, RDPG-style dot-product affinities, and a population-level intensity that links continuous latent structure to finite observed graphs. We define the heat map and the desire operator as continuous analogues of the probability matrix, prove a spectral consistency result connecting adjacency singular values to the operator spectrum, compare the construction with graphon and digraphon representations, and show how classical RDPGs arise in a concentrated limit. Because the model is parameterized by an evolving intensity, temporal extensions through partial differential equations arise naturally.
\end{abstract}

\section{Introduction}

Statistical network models face a fundamental modeling choice: what is random?
In the dominant paradigm~\cite{newman2018networks}, nodes are treated as fixed objects, such as people in a social network, species in a food web, or neurons in a connectome, while edges are the outcome of a probabilistic mechanism.
This asymmetry is built into the most widely studied families: Erd\H{o}s--R\'enyi models, stochastic block models~\cite{airoldi2008mixed}, latent position models, and Random Dot Product Graphs (RDPGs)~\cite{young2007random,athreya2018statistical}.
Yet in many applications the \emph{identity} and \emph{number} of interacting entities are themselves stochastic.
Ecological communities assemble through colonization and extinction; transient encounters on a transport network involve passengers who appear and disappear; neurons fire in overlapping but shifting ensembles.
In such settings, the nodes of the observed graph are better described as \emph{samples from a random process} than as a fixed roster.

Several existing frameworks address parts of this issue.
Graphon and digraphon models~\cite{lovasz2012large,borgs2008convergent} assign random latent labels to sampled nodes, and random-measure models~\cite{caron2017sparse,veitch2015class} generate sparse random graphs with random vertex populations.
Spatial point process models~\cite{daley2003introduction,last2017lectures} provide a mature theory for random point configurations.
However, these frameworks do not simultaneously provide an explicit finite-dimensional geometric latent space, RDPG-style dot-product affinities, and a continuous population-level object that links latent structure to finite observed graphs.
In graphon formulations, latent geometry is only defined up to measure-preserving rearrangement; in random-measure models, connectivity is not tied to Euclidean latent coordinates with the same direct interpretability.

In this paper, we introduce the family of \textbf{Intensity Dot Product Graphs} (IDPGs).
An IDPG extends the RDPG by replacing a fixed collection of latent positions with a Poisson point process on a latent position space $\mblue{\Omega} = \mgreen{B^d_+} \times \mred{B^d_+}$, governed by an intensity function $\mbblue{\rho}$.
Sampled individuals are located by their positions in the latent space, and the probability of a connection between two individuals is given by the dot product of their latent coordinates, preserving the interpretive structure of the RDPG.
The intensity $\mbblue{\rho}$ encodes the population-level distribution of interaction propensities, and the observed graph is a finite, noisy realization of this continuous object.

This construction yields several contributions.

\begin{enumerate}
\item \textbf{A generative framework bridging point processes and latent-position network models.}
We define IDPGs through two contrasting realization rules: \emph{perennial}, where long-lived entities can form all pairwise connections, and \emph{ephemeral}, where transient entities interact only in sampled pairs. We also introduce an intermediate regime based on entity lifetimes.
We derive closed-form expressions for expected edge counts and show that perennial graphs scale quadratically in the total intensity $\mblue{\Lambda}$, whereas ephemeral graphs scale linearly (Section~\ref{sec:expected-edges}).

\item \textbf{Rigorous comparison with graphon and digraphon theory.}
Every perennial IDPG admits a digraphon representation, but we prove that this representation necessarily destroys the local regularity of the latent space: any equivalent digraphon kernel fails to be of bounded variation, and global geometric coherence (Lipschitz continuity, Euclidean clustering) is lost.
This is not a technical inconvenience but a \emph{dimensional obstruction}: the one-dimensional label space of graphon theory cannot faithfully represent the higher-dimensional geometry of the IDPG latent space (Section~\ref{sec:graphon-obstruction}).

\item \textbf{The heat map: a measure-theoretic analogue of the probability matrix.}
We introduce the \emph{heat map} $\mathcal{H}$, a continuous operator that captures the full interaction structure of the model.
Its spectral decomposition reveals the dominant modes of interaction, and we prove a spectral consistency theorem: scaled singular values of the adjacency matrix converge to the singular values of the \emph{desire operator}, a normalized variant of the heat map, as the intensity grows (Section~\ref{sec:heatmap}).

\item \textbf{Temporal dynamics via PDEs on the intensity.}
Because the model is parameterized by a continuous intensity function, temporal evolution is naturally described by partial differential equations, including diffusion, advection, and pursuit-evasion dynamics, acting on $\mbblue{\rho}$.
We show that the ratio of perennial to ephemeral expected edges tracks the evolving intensity through time, regardless of the PDE regime, and verify this invariance computationally (Section~\ref{sec:time-indexing}).
\end{enumerate}

The framework is motivated by, and illustrated through, an ecological application: modeling food webs as IDPGs with mixture-of-products intensities representing distinct trophic species (Section~\ref{sec:foodwebs}).
In this context, the perennial/ephemeral distinction maps onto long-lived vs.\ transient ecological interactions, and the PDE dynamics describe shifts in community structure over time.

The paper is organized as follows.
Section~2 reviews Random Dot Product Graphs.
Section~3 defines Intensity Graphs, the perennial and ephemeral realization rules, and derives expected edge counts.
Section~4 establishes the relationship to graphons and digraphons, including the regularity obstruction theorems.
Section~5 introduces the heat map, its spectral decomposition, the desire operator, and the spectral consistency theorem.
Section~6 discusses the recovery of classical RDPGs as a limiting case.
Section~7 develops the ecological application.
Section~8 introduces PDE dynamics on the intensity.
Section~9 presents computational experiments verifying the theoretical predictions.
Sections~10 and~11 discuss inference and future directions.
Derivations and proofs are collected in the Appendix.

\section{Random Dot Product Graphs}

Let $G=(V,E)$ be a simple, directed graph, with nodes $i,j \in V = \{ 1, 2, \ldots \}$ and edges $(i \rightarrow j) \in E \subset V \times V$, where we consider only edges between distinct nodes ($i \neq j$).

\begin{figure}[htbp]
\centering
\begin{minipage}[t]{0.48\linewidth}
  \centering
  \includegraphics[height=0.67\linewidth]{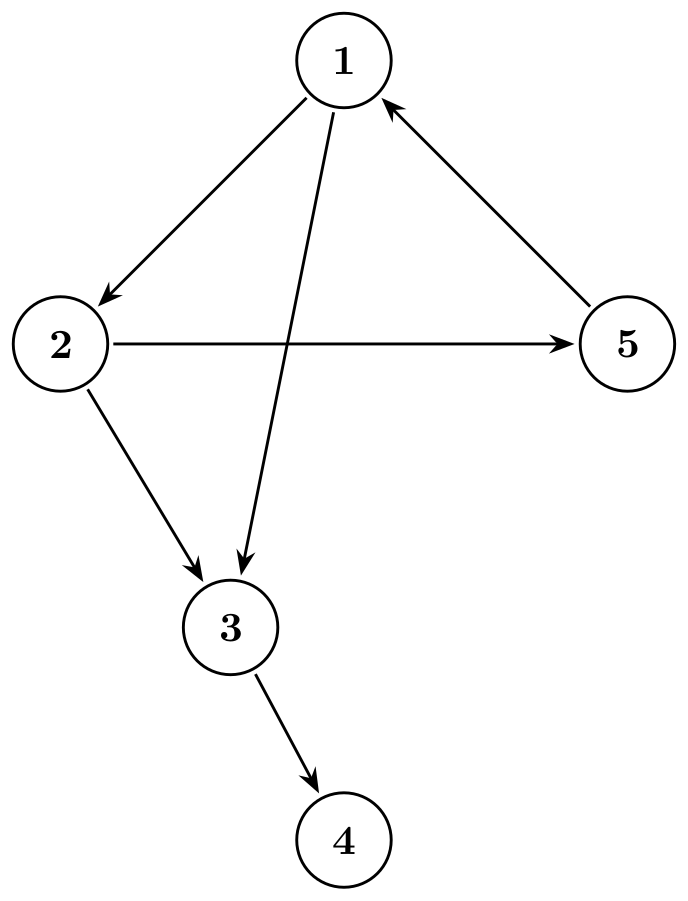}
  \caption{A simple, directed graph with 5 nodes and a bunch of edges.}
  \label{fig:graph}
\end{minipage}\hfill
\begin{minipage}[t]{0.48\linewidth}
  \centering
  \includegraphics[width=\linewidth]{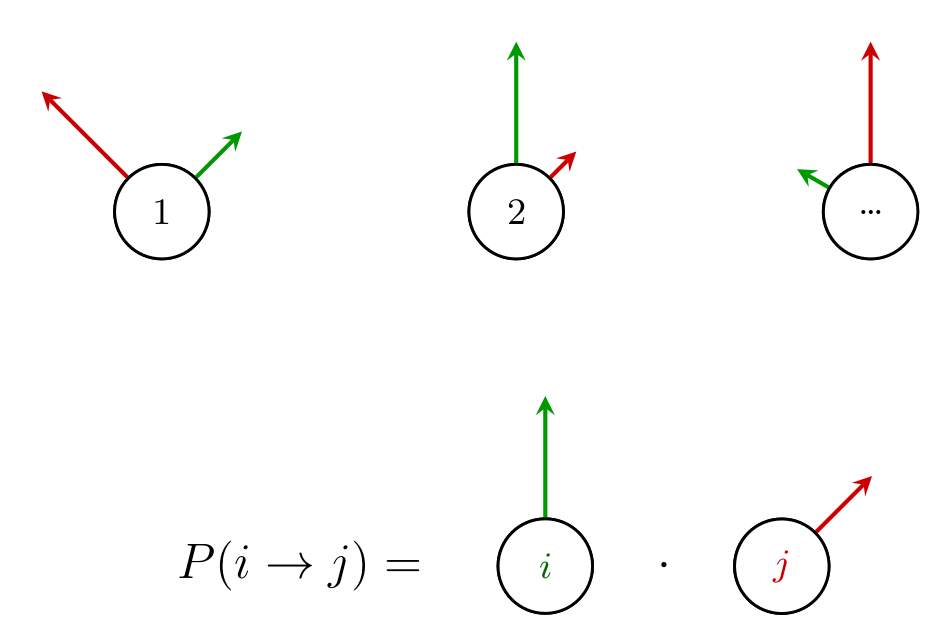}
  \caption{Each node of a graph is associated with a pair of vectors, one green and one red. The probability of observing an edge between two nodes is given by the dot product of a green and a red vector.}
  \label{fig:node_vectors}
\end{minipage}
\end{figure}

We'll consider graphs as the outcome of random processes.
This means that we associate to any possible graph a certain probability of being observed.
In particular, let $V$ be a given set $\{1, 2, \ldots, N\}$ of nodes, every ordered couple $(i,j)$ in $V \times V$ is in $E$ with a certain probability $p_{\text{ij}}$; we define the matrix of interaction probabilities $\mathbf{P}$ so that $\mathbf{P}_{\text{ij}} = p_{\text{ij}}$, and denote it $\mathbf{P}_G$ if we need to be explicit regarding what graph it is associated with.

Notice here that $\mathbf{P}$ completely determines the probability of observing a given graph $G=(V,E)$: the probability of $G$ will be given by the probability of observing exactly the links in $E$ and not observing the links not in $E$.

Random graph models are described by how they determine those interaction probabilities.

\subsection{RDPG as generating model}

In a Random Dot Product Graph (RDPG) model~\cite{young2007random}, each node is associated with two $d$-dimensional vectors, $\mgreen{\vec{g}_i}$ and $\mred{\vec{r}_i}$ (Figure~\ref{fig:node_vectors}). These vectors are chosen so that $\mgreen{\vec{g}_i} \cdot \mred{\vec{r}_j} \in [0,1]$ for every $(i,j) \in V \times V$. A pair of nodes $i,j$ is an edge in $E$ with probability $p_{i,j} = \mgreen{\vec{g}_i} \cdot \mred{\vec{r}_j}$.

We can then consider two matrices $\mbgreen{G}$ and $\mbred{R}$, where the rows $\mbgreen{G}_{i,\cdot}$ of $\mbgreen{G}$ are the vectors $g(i)$ and the columns $\mbred{R}_{\cdot,i}$ of $\mbred{R}$ are the vectors $r(i)$ for every $i$ in $V$. We have that the matrix multiplication
\begin{equation}
\mbgreen{G} \mbred{R} = \mathbf{P}
\end{equation}
and, hence, the two matrices $(\mbgreen{G}, \mbred{R})$ contain all the information of the random graph model (the number of the nodes is given by the number of rows of $\mbgreen{G}$, that is the number of columns of $\mbred{R}$).

It is convenient for our intuition to consider the vectors $\mgreen{\vec{g}_i}$ and $\mred{\vec{r}_i}$ as the node $i$ propensity to interact, either proposing or accepting a connection (Figure~\ref{fig:points}).
Furthermore, it is convenient to visualize each node $i$ as a pair of points in two $d$ dimensional metric spaces, that we will refer to (with some abuse of notation) as the green space $\mgreen{G}$ and the red space $\mred{R}$.
The coordinate of $i$ in these two spaces is given by $\mgreen{\vec{g}_i}$ and $\mred{\vec{r}_i}$. Hence, we can also see an RDPG model $(\mbgreen{G}, \mbred{R})$ as defined by a given set of $N$ points in the spaces $\mgreen{G}$ and $\mred{R}$, which offers a nice geometric representation of the random graph model.

\begin{figure}[htbp]
\centering
\includegraphics[width=0.7\linewidth]{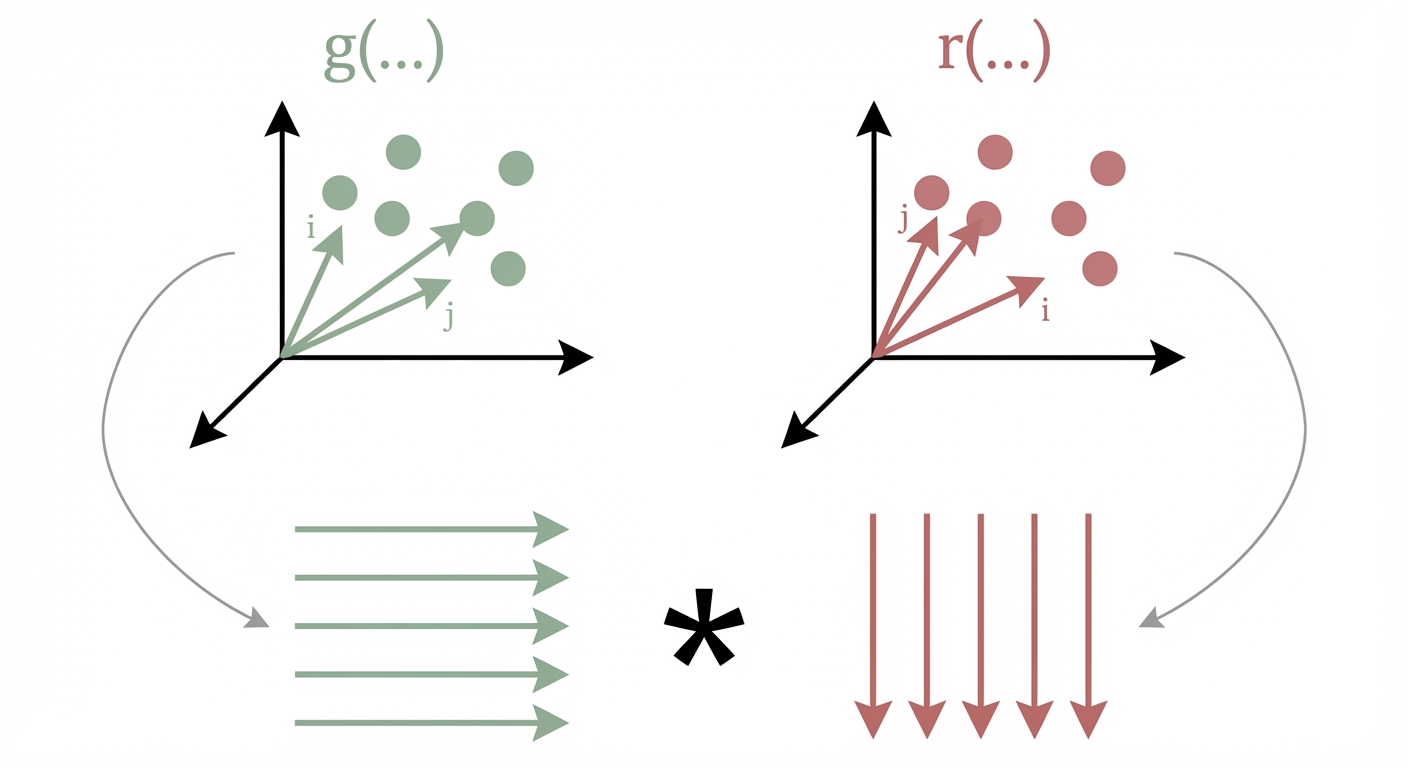}
\caption{The green and red vectors associated with the nodes define a set of points in two metric spaces. We define two matrices: one green, whose rows will be the coordinates of the green points, one red, whose columns will be the coordinates of the red points. Their matrix multiplication gives all the connection probabilities.}
\label{fig:points}
\end{figure}

To summarize, under a RDPG model, nodes are associated with points in a certain pair of spaces $\mgreen{G}$ and $\mred{R}$, and the probability of observing an edge between two points is given by the dot product $\mgreen{\vec{g}_i} \cdot \mred{\vec{r}_j}$.

\subsection{Inference of an RDPG model}

The inference of RDPG model parameters goes the other way round than the generation task.

Given an observed graph $G=(V,E)$ we are posed with the problem of identifying the two most likely matrices $(\mbgreen{G}, \mbred{R})$ that generated $G$.

This will be accomplished in two steps:
\begin{enumerate}
\item Infer the right dimension $d$ for the spaces $\mgreen{G}$ and $\mred{R}$ (notice: the dimension, not the number of points)
\item Infer the positions of the $N$ points in $\mgreen{G}$ and $\mred{R}$.
\end{enumerate}

First of all, we need to define the adjacent matrix of $G$. This is matrix $\mathbf{A}$ such that
\begin{equation}
\mathbf{A}_{i,j} = \begin{cases}
  1 & \text{if } i \rightarrow j \in E, \\
  0 & \text{otherwise.}
\end{cases}
\end{equation}

Classic results~\cite{athreya2018statistical} show that, under standard RDPG assumptions, both (i) and (ii) can be consistently estimated with elementary linear-algebraic tools based on the singular value decomposition of $\mathbf{A}$ (and indeed (i) is the most challenging!).

Let $\mathbf{A} = \mathbf{U} \boldsymbol{\Sigma} \mathbf{V}^T$ be the singular value decomposition of $\mathbf{A}$, that is $\mathbf{U}$ and $\mathbf{V}$ are orthogonal matrices and $\boldsymbol{\Sigma}$ is an $N \times N$ diagonal matrix with non-negative real coefficients on its diagonal in decreasing order. The elements $\sigma_i = \boldsymbol{\Sigma}_{i,i}$ are known as the singular values of $\mathbf{A}$.

An optimal dimension $\hat{d}$ can be inferred solely from the sequence of singular values $\sigma_i$. There are various techniques for doing it, and the technicality is left to the curious reader (see for example~\cite{zhu2006automatic,gavish2014optimal,chatterjee2015matrix}).

Let's define the two matrices $\mbgreen{\widetilde{G}} = \mathbf{U}|_{\hat{d}} \sqrt{\boldsymbol{\Sigma}|_{\hat{d}}}$ and $\mbred{\widetilde{R}} = \sqrt{\boldsymbol{\Sigma}|_{\hat{d}}} (\mathbf{V}|_{\hat{d}})^T$, where $M|_k$ is the truncation of a matrix $M$ to its first $k$ columns, and $\sqrt{\boldsymbol{\Sigma}}_{i,i} = \sqrt{\sigma_i}$ is the element-wise square root of $\boldsymbol{\Sigma}$.

Then, we have that
\begin{equation}
\begin{cases}
  \mbgreen{G} \approx \mbgreen{\widetilde{G}}, \\
  \mbred{R} \approx \mbred{\widetilde{R}}.
\end{cases}
\end{equation}
In particular, the matrix $\mathbf{\widetilde{A}} = \mbgreen{\widetilde{G}} \mbred{\widetilde{R}} \approx \mathbf{A}$
is optimal in the sense that it minimizes the distance to $\mathbf{A}$ in Frobenius norm, that is:
\begin{equation}
\mathbf{\widetilde{A}} = \argmin_{\mathbf{M} \text{ of rank } \hat{d}} \lVert \mathbf{A} - \mathbf{M} \rVert_F \;.
\end{equation}

To summarize, given an observed graph $G=(V,E)$, spectral methods based on singular value decomposition provide standard estimators of $(\mbgreen{G}, \mbred{R})$, up to the usual latent-space non-identifiabilities.

\section{Intensity Graphs}

Notice that in a RDPG model, while the edges are probabilistic, the nodes are not: their number and their identities, that is their propensities to propose and accept an edge, are completely determined by the model parameters.

Here we introduce the family of \emph{Intensity Graph} (IG) that start from a continuous setting in which nodes themselves are the outcome of a probability process.

\subsection{The latent space}

Before defining an IG, we address a technical constraint. In an RDPG, the vectors $\mgreen{g}(i)$ and $\mred{r}(j)$ must satisfy $\mgreen{g}(i) \cdot \mred{r}(j) \in [0,1]$ for all pairs of nodes, so that the dot product can be interpreted as an edge probability.

Since $\mgreen{\vec{g}_{\cdot}} \cdot \mred{\vec{r}_{\cdot}} = \|\mgreen{\vec{g}_{\cdot}}\| \cdot \|\mred{\vec{r}_{\cdot}}\| \cos \theta$, where $\theta$ is the angle between the vectors, two conditions suffice: (i) the norms are bounded by one, giving $\mgreen{\vec{g}_{\cdot}} \cdot \mred{\vec{r}_{\cdot}} \leq 1$; and (ii) the angle $\theta$ is at most $90^\circ$, ensuring $\mgreen{\vec{g}_{\cdot}} \cdot \mred{\vec{r}_{\cdot}} \geq 0$.

A canonical choice satisfying both conditions is the non-negative part of the closed unit ball:
\begin{equation}
B^d_+ = \{ x \in \mathbb{R}^d : x_k \geq 0 \text{ for all } k, \|x\| \leq 1 \}
\end{equation}
Any two vectors in the non-negative orthant $\mathbb{R}^d_+$ make an acute (or right) angle, satisfying (ii); the norm constraint satisfies (i).

Since all observable quantities depend only on inner products, the model is invariant under orthogonal transformations applied jointly to both the giving and receiving spaces. Restricting to $B^d_+$ is a convenient representation, not a fundamental constraint.

\subsection{Intensity Graphs as generating model}

In an Intensity Graph model, edges, nodes, and thus graphs emerge through stochastic processes in which individuals are sampled, and based on their affinity might establish connections.

\subsubsection{Individuals, positions, nodes}

Before defining an IDPG, we introduce some basic vocabulary and establish notation for the latent space. The intent is to link three different levels of discussion: the interpretation, the measure-theoretic/stochastic process, and the graph theory.

We express the model as revolving around the establishment of \emph{connections} between \emph{individuals}. These can be embodied in any form (from people listening and talking, to species consuming each other, to country and their commercial flows).

In terms of latent space, an individual is defined by its \emph{position} $i \in \mblue{\Omega}$, that is a point in the product space $\mblue{\Omega} = \mgreen{B^d_+} \times \mred{B^d_+}$.
Each individual's position has two coordinates:
\begin{itemize}
\item A \emph{green coordinate} $\mgreen{\vec{g}_i} \in B^d_+$: the propensity to \emph{give} connections (as a source)
\item A \emph{red coordinate} $\mred{\vec{r}_i} \in B^d_+$: the propensity to \emph{receive} connections (as a target)
\end{itemize}

We write $i = (\mgreen{\vec{g}_i}, \mred{\vec{r}_i})$ for the position of a specific individual, and $(\mgreen{\vec{g}}, \mred{\vec{r}})$ for a generic point in $\mblue{\Omega}$ when we don't refer to any specific individual.

In terms of graphs, an individual is represented as a node, and a connection as an edge.

\subsubsection{General definition}

An Intensity Graph separates three components: how interaction opportunities arise (realization rule), where individuals position concentrate (intensity), and how interaction are established as connection (affinity kernel). The realization rule and intensity together determine which pairs of individuals have the \emph{opportunity} to interact; the affinity kernel then determines which interaction become actual connections, and thus edges in the graph.

\begin{definition}[IDPG]\label{def:idpg}
An Intensity Graph is specified by:

\begin{enumerate}
\item A \emph{realization rule} $\mathbf{R}$ that determines how interactions, that is connection opportunities, arise. The rule specifies whether all sampled individuals can interact with each other (perennial, $\mathbf{R}_\infty$), only individuals sampled together as pairs can interact (ephemeral, $\mathbf{R}_0$), or something intermediate.

\item An \emph{intensity} $\mbblue{\rho}: \mblue{\Omega} \rightarrow \mathbb{R}_+$ describing the density of individuals' position across $\mblue{\Omega} = \mgreen{B^d_+} \times \mred{B^d_+}$. We will detail reasonable smoothness requirements below. In general, high $\mbblue{\rho(i)}=\mbblue{\rho}(\mgreen{\vec{g}_i}, \mred{\vec{r}_i})$ means individuals with positions near $i=(\mgreen{\vec{g}_i}, \mred{\vec{r}_i})$ are more likely to participate in interactions.

\item An \emph{affinity kernel} $K: \mblue{\Omega} \times \mblue{\Omega} \rightarrow [0,1]$ giving the probability that an interaction between two individuals becomes a realized edge. In particular, in an \textbf{Intensity Dot Product Graph} (IDPG), the kernel is given by the dot product: in an interaction between $s$ and $t$, being $s$ the individual proposing the connection (the source of the interaction), and $t$ the individual accepting it (the target of the interaction), we have that:
\begin{equation}
K(s,t) = \mgreen{\vec{g}_s} \cdot \mred{\vec{r}_t}
\end{equation}
The connection probability depends on the \emph{green} coordinate of the source $s$ and the \emph{red} coordinate of the target $t$.
\end{enumerate}

Given these components, an Intensity Graph generates a random graph in two stages:
\begin{itemize}
\item \emph{Stage 1 (Interactions)}: The realization rule $\mathbf{R}$, operating on the intensity $\mbblue{\rho}$, produces a random set of ordered pairs $(s, t)$ of positions that interact.
\item \emph{Stage 2 (Connections)}: Each interaction $(s, t)$ independently becomes a connection with probability $K(s,t)$.
\end{itemize}
\end{definition}

The realization rule $\mathbf{R}$ converts the total mass of the intensity into interactions; different rules produce different numbers and distributions of interactions from the same $\mbblue{\rho}$.

\subsubsection{The lifetime perspective}

The choice of realization rule can be understood through a unifying physical picture: \emph{individual lifetime}.

Imagine individuals are born over time, persist for some time $\tau$, and then disappear. Two individuals can interact only if their lifetimes overlap. The mean lifetime $\eta$ determines how many interactions arise:

\begin{itemize}
\item \textbf{Perennial individuals} ($\eta \rightarrow \infty$): All individuals coexist, so every pair can interact. With $N$ individuals, we get $N^2$ interactions.
\item \textbf{Intermediate lifetime} ($0 < \eta < \infty$): Some pairs overlap, some don't. The number of interactions interpolates between the extremes.
\item \textbf{Ephemeral individuals} ($\eta \rightarrow 0$): Individuals exist only instantaneously. Independent individuals never overlap; interactions occur only when individuals are ``born as pairs'' (that is, sampled directly as an interaction source and target). Each interaction consumes two individual-equivalents, yielding $2 N$ opportunities from total intensity that produced $N$ individuals.
\end{itemize}

We now define the two limiting realization rules, and the intermediate one, precisely. We denote $\mblue{\Lambda}$ the total mass of the intensity, that is
$\mblue{\Lambda} = \int_{\mblue{\Omega}} \mbblue{\rho}(s) \, ds = \int_{\mblue{\Omega}} \mbblue{\rho}(\mgreen{\vec{g}}, \mred{\vec{r}}) \, d\mgreen{\vec{g}} \, d\mred{\vec{r}}$.

\paragraph{Perennial rule ($\mathbf{R}_\infty$)}

Sample individuals from a Poisson Point Process (PPP) on $\mblue{\Omega}$ with intensity $\mbblue{\rho}$:
\begin{equation}
N \sim \text{Poisson}(\mblue{\Lambda}), \quad \text{positions } (\mgreen{\vec{g}_i}, \mred{\vec{r}_i}) \overset{\text{i.i.d.}}{\sim} \mbblue{\rho} / \mblue{\Lambda}
\end{equation}

The $N$ sampled individuals become the \emph{nodes} of the graph. Every ordered pair of individuals $(i, j)$ constitutes an interaction, with connection probability given by the affinity kernel $\mgreen{\vec{g}_i} \cdot \mred{\vec{r}_j}$.

All ordered pairs of nodes have a chance to interact, hence $N^2$ potential edges.

Interactions are \emph{conditionally independent given the node positions}, but \emph{marginally dependent}. Conditional on $(\mgreen{\vec{g}_i}, \mred{\vec{r}_i})_{i=1}^N$, each edge $i \rightarrow j$ is an independent Bernoulli trial. However, marginally (integrating over random positions), edges sharing a node are correlated: observing that node $i$ has high out-degree reveals information about $\mgreen{\vec{g}_i}$, which affects probabilities for other edges from $i$.

The perennial rule does not generate a PPP for the interactions. Indeed, their number is not Poisson in itself, but it is quadratic in a Poisson random variable (namely, the number of individuals), and interactions are marginally correlated through shared individuals.

In the perennial rule, the total intensity $\mblue{\Lambda}$ equals the expected number of sampled individuals: $\E[N] = \mblue{\Lambda}$
and the intensity $\mbblue{\rho}$ introduces a natural scale: multiplying $\mbblue{\rho}$ by a constant $c$ scales $\E[N]$ by $c$ and $\E[E]$ by $c^2$.

The perennial rule produces a classic random graph with persistent nodes:

\begin{itemize}
\item \textbf{Nodes can participate in multiple edges}: as source (via $\mgreen{\vec{g}}$) and as target (via $\mred{\vec{r}}$)
\item \textbf{Nontrivial topology}: paths, triangles, (large) connected components, varying degree distributions
\item \textbf{Isolated individuals}: A sampled individual may fail to form any connections, hence creating isolated nodes. Let $N_{\text{obs}}$ denote nodes with degree $\geq 1$. We have $N_{\text{obs}} \leq N$.
\end{itemize}

\begin{figure}[htbp]
\centering
\includegraphics[width=1.0\linewidth]{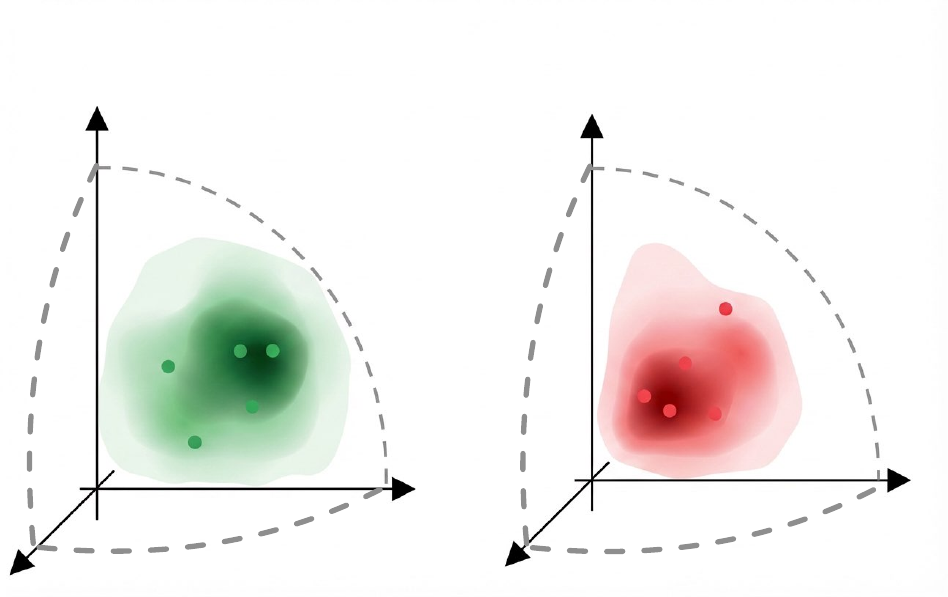}
\caption{From a set of points, we move toward intensity functions defining the expected density of individuals in the latent space. Regions of higher intensity will contain more individuals on average. The green coordinate describes propensity to propose connections; the red coordinate describes propensity to accept them.}
\label{fig:density}
\end{figure}

Depending on the fine modelling decision, the distinction between $N$ and $N_{\text{obs}}$ matters for inference: an observed graph reveals only nodes with positive degree. Nodes with weak propensities (small $\|\mgreen{\vec{g}}\|$ or $\|\mred{\vec{r}}\|$) have higher probability of isolation, so the observed population is biased toward nodes with stronger interaction propensities. This is analogous to zero-truncation in count data models.

\paragraph{Intermediate regime ($\mathbf{R}_\eta$): Finite lifetime}

A more complex case emerges when individuals live a finite, but not ephemeral, life. For example, individuals are sampled from a space-time PPP on $\mblue{\Omega} \times [0, W]$ (the latter being the observation window), with:
\begin{itemize}
\item Position intensity $\mbblue{\rho}$
\item Lifetime $\tau \sim \text{Exp}(\eta)$ (exponential with mean $\eta$, yet other choices are possible)
\end{itemize}

An individual born at time $T$ with lifetime $\tau$ is observed ``alive'' during $[T, T + \tau] \cap [0, W]$. Two individuals interact if and only if their lifetimes overlap.

Let $p_{\text{overlap}}(\eta, W)$ be the probability that two independently sampled individuals have overlapping lifetimes. For exponentially distributed lifetimes with mean $\eta$ and birth times uniform on $[0, W]$:
\begin{equation}
p_{\text{overlap}}(\eta, W) = \frac{2}{u^2} (u - 1 + e^{-u})
\end{equation}
where $u = W / \eta$.

Conditional on $N$ nodes being sampled, if we count ordered interaction opportunities including self-opportunities, the expected number of interacting opportunities is exactly $N^2 \cdot p_{\text{overlap}}(\eta, W)$ (excluding self-opportunities would replace $N^2$ by $N(N-1)$). Taking expectations over $N \sim \text{Poisson}(\mblue{\Lambda})$:
\begin{equation}
\E[\text{interactions}] = \E[N^2] \cdot p_{\text{overlap}}(\eta, W) = (\mblue{\Lambda}^2 + \mblue{\Lambda}) \cdot p_{\text{overlap}}(\eta, W)
\end{equation}

The limiting behavior confirms our interpretation:

As $\eta / W \rightarrow \infty$ (long-lived, $u \rightarrow 0$): Taylor expansion gives $p_{\text{overlap}} \rightarrow 1$, recovering the perennial scenario where all individuals coexist during the observation window.

As $\eta / W \rightarrow 0$ (ephemeral, $u \rightarrow \infty$): $p_{\text{overlap}} \approx 2 \eta / W \rightarrow 0$, and interaction opportunities vanish. The ephemeral rule emerges as the natural limiting model for instantaneous interactions. This interpolation is verified in Section~\ref{sec:simulations}.

\paragraph{Ephemeral rule ($\mathbf{R}_0$)}

In the ephemeral limit, individuals exist only instantaneously and can interact only if ``born together'' as a pair. We model this by sampling \emph{interaction pairs} rather than allowing all pairs to interact.

Sample $M \sim \text{Poisson}(\mblue{\Lambda} / 2)$ interaction pairs. For each pair, draw two independent positions:
\begin{equation}
(\mgreen{\vec{g}_i}, \mred{\vec{r}_i}), (\mgreen{\vec{g}_j}, \mred{\vec{r}_j}) \overset{\text{i.i.d.}}{\sim} \mbblue{\rho} / \mblue{\Lambda}
\end{equation}

The total number of individuals is $N = 2M$, so $\E[N] = \mblue{\Lambda}$ as in the perennial case.

\textbf{Connections within each sampled pair (ephemeral rule).} In the ephemeral rule, when two individuals $i$ and $j$ are sampled as an interaction pair, the following four potential edges are evaluated:
\begin{itemize}
\item $i \rightarrow j$ with probability $\mgreen{\vec{g}_i} \cdot \mred{\vec{r}_j}$
\item $j \rightarrow i$ with probability $\mgreen{\vec{g}_j} \cdot \mred{\vec{r}_i}$
\item $i \rightarrow i$ with probability $\mgreen{\vec{g}_i} \cdot \mred{\vec{r}_i}$ (self-loop)
\item $j \rightarrow j$ with probability $\mgreen{\vec{g}_j} \cdot \mred{\vec{r}_j}$ (self-loop)
\end{itemize}

In contrast, the perennial rule evaluates every ordered pair $(u, v)$ with
$u, v \in \{1, \ldots, N\}$, including self-pairs, so there are $N^2$
potential edges. The key distinction is which opportunities are evaluated:
perennial uses all ordered pairs, whereas ephemeral uses only the
$M = N/2$ sampled disjoint pairs.

The ephemeral rule produces a graph that decomposes into disconnected components with at most two nodes:

\begin{itemize}
\item \textbf{Disjoint pairs}: Each individual belongs to exactly one interaction pair; no node participates in interactions with multiple partners
\item \textbf{Rich local structure}: Within each pair, up to 4 edges can form (two cross-edges, two self-loops), yielding a non-trivial motif vocabulary
\item \textbf{No global connectivity}: Paths of length $> 1$ cannot exist; the graph is a disjoint union of small components
\end{itemize}

Yet, as we will see in the numerical results, meaningful aggregate structure emerges through the distribution of motif types across pairs and through post-hoc clustering or discretization of the latent space.

\subsection{Computing expected edges}\label{sec:expected-edges}

Despite their different generative mechanisms, both limiting rules admit clean formulas for expected edge counts.

\subsubsection{Perennial regime}

For the perennial rule, we use the \emph{second factorial moment formula}. For a Poisson process, the independence of counts in disjoint sets~\cite{daley2003introduction} implies:
\begin{equation}
\E\!\left[\sum_{i \neq j} f(x_i, x_j)\right] = \iint f(x, y) \, \lambda(dx) \, \lambda(dy)
\end{equation}

Applying this to edge counting with $f(s, t) = \mgreen{\vec{g}_s} \cdot \mred{\vec{r}_t}$:
\begin{equation}
\E[E]_{\mathbf{R}_\infty} = \iint_{\mblue{\Omega} \times \mblue{\Omega}} (\mgreen{\vec{g}_s} \cdot \mred{\vec{r}_t}) \, \mbblue{\rho}(s) \, \mbblue{\rho}(t) \, ds \, dt
\end{equation}

The cautious reader would have noticed that the summing is over $i \neq j$, and thus we are missing the contribution of self-connections $i \rightarrow i$. We acknowledge that, reassure the reader the contribution is linear, hence small for reasonably large $\mblue{\Lambda}$, and refer to Appendix~\ref{appendix:selfloops} for a more detailed discussion.

\subsubsection{Ephemeral regime}

For the ephemeral rule, we sum over the $M$ interaction pairs. Each pair $\{i, j\}$ contributes four potential edges with probabilities $\mgreen{\vec{g}_i} \cdot \mred{\vec{r}_j}$, $\mgreen{\vec{g}_j} \cdot \mred{\vec{r}_i}$, $\mgreen{\vec{g}_i} \cdot \mred{\vec{r}_i}$, and $\mgreen{\vec{g}_j} \cdot \mred{\vec{r}_j}$.

Taking expectations over positions drawn i.i.d.\ from $\mbblue{\rho} / \mblue{\Lambda}$:
\begin{equation}
\E[\text{edges per pair}] = \E[\mgreen{\vec{g}_i} \cdot \mred{\vec{r}_j}] + \E[\mgreen{\vec{g}_j} \cdot \mred{\vec{r}_i}] + \E[\mgreen{\vec{g}_i} \cdot \mred{\vec{r}_i}] + \E[\mgreen{\vec{g}_j} \cdot \mred{\vec{r}_j}]
\end{equation}

Since the positions are independent, the cross-terms give $\E[\mgreen{\vec{g}}] \cdot \E[\mred{\vec{r}}]$ and the self-loop terms give $\E[\mgreen{\vec{g}} \cdot \mred{\vec{r}}]$. With $M = N/2$ pairs:
\begin{equation}
\E[E]_{\mathbf{R}_0} = \E[M] \cdot \E[\text{edges per pair}]
\end{equation}

This scales linearly in the total intensity $\mblue{\Lambda}$, in contrast to the quadratic scaling of the perennial rule.

\subsubsection{The product case}

The case in which the intensity factorizes is easier to analyse mathematically. Let the intensity be a product of independent intensities on the giving and receiving coordinate spaces:
\begin{equation}
\mbblue{\rho}(\mgreen{\vec{g}}, \mred{\vec{r}}) = \mbgreen{\rho_G}(\mgreen{\vec{g}}) \cdot \mbred{\rho_R}(\mred{\vec{r}})
\end{equation}

The product form has a natural interpretation: an individual's propensity to propose connections (its position in $\mgreen{G}$) is independent of its propensity to accept connections (its position in $\mred{R}$).

Let us define:
\begin{itemize}
\item the marginal total intensities $\mgreen{c_G} = \int_{\mgreen{G}} \mbgreen{\rho_G}(\mgreen{\vec{g}}) \, d\mgreen{\vec{g}}$ and $\mred{c_R} = \int_{\mred{R}} \mbred{\rho_R}(\mred{\vec{r}}) \, d\mred{\vec{r}}$,
\item the intensity-weighted mean positions $\mbgreen{\mu_G} = \int_{\mgreen{G}} \mgreen{\vec{g}} \, \mbgreen{\rho_G}(\mgreen{\vec{g}}) \, d\mgreen{\vec{g}}$ and $\mbred{\mu_R} = \int_{\mred{R}} \mred{\vec{r}} \, \mbred{\rho_R}(\mred{\vec{r}}) \, d\mred{\vec{r}}$,
\item the normalized mean positions $\mbgreen{\widetilde{\mu}_G} = \mbgreen{\mu_G} / \mgreen{c_G}$ and $\mbred{\widetilde{\mu}_R} = \mbred{\mu_R} / \mred{c_R}$.
\end{itemize}

We can see that the total intensity is $\mblue{\Lambda} = \mgreen{c_G} \cdot \mred{c_R}$, so $\E[N] = \mgreen{c_G} \cdot \mred{c_R}$. All mathematical results are derived in Appendix~\ref{appendix:derivations}.

\paragraph{Perennial}

Using the second factorial moment formula we can derive an explicit formula for the expected number of edges, which links both the total intensity and the normalized mean positions:
\begin{equation}
\E[E]_{\mathbf{R}_\infty} = (\E[N])^2 \cdot (\mbgreen{\widetilde{\mu}_G} \cdot \mbred{\widetilde{\mu}_R})
\end{equation}

\paragraph{Ephemeral}

Each of the $M = N/2$ interaction pairs contributes four potential edges. The expected number of edges per pair, under the product assumption, is:
\begin{equation}
\E[\text{edges per pair}] = 2(\mbgreen{\widetilde{\mu}_G} \cdot \mbred{\widetilde{\mu}_R}) + 2(\mbgreen{\widetilde{\mu}_G} \cdot \mbred{\widetilde{\mu}_R}) = 4(\mbgreen{\widetilde{\mu}_G} \cdot \mbred{\widetilde{\mu}_R})
\end{equation}
where the first term accounts for the two cross-edges ($i \rightarrow j$ and $j \rightarrow i$) and the second for the two self-loops ($i \rightarrow i$ and $j \rightarrow j$). Therefore:
\begin{equation}
\E[E]_{\mathbf{R}_0} = \E[M] \cdot 4(\mbgreen{\widetilde{\mu}_G} \cdot \mbred{\widetilde{\mu}_R}) = 2 \E[N] \cdot (\mbgreen{\widetilde{\mu}_G} \cdot \mbred{\widetilde{\mu}_R})
\end{equation}

\paragraph{The ratio of expected edges}

The ratio of expected edges between the two rules is:
\begin{equation}
\frac{\E[E]_{\mathbf{R}_\infty}}{\E[E]_{\mathbf{R}_0}} = \frac{(\E[N])^2}{2 \E[N]} = \frac{\E[N]}{2} = \frac{\mblue{\Lambda}}{2}
\end{equation}

This expression uses the distinct-pair perennial convention ($i \neq j$). If perennial self-loops are included, the exact ratio is $\frac{\mblue{\Lambda}+1}{2}$ (see Appendix~\ref{appendix:selfloops}).

The fundamental scaling difference persists: perennial produces $O(N^2)$ edges (dense), while ephemeral produces $O(N)$ edges (sparse). Under the distinct-pair convention, the ratio $\mblue{\Lambda}/2$ grows linearly with population size.
\section{Relationship to graphons and digraphons}\label{sec:graphon-obstruction}

The model definition invites comparison with \emph{graphon theory}~\cite{lovasz2012large,borgs2008convergent}\footnote{Extensions to sparse graphs include graphex models~\cite{veitch2015class,caron2017sparse} and $L^p$ graphon theory~\cite{borgs2018sparse1}.}. Both IDPG and graphon frameworks capture interaction structure via kernels on continuous spaces: a graphon generates a random undirected graph by sampling uniformly at random $N$ labels (that is, numbers) $U_1, \dots, U_N$ from $[0,1]$ and connecting $i \leftrightarrow j$ with probability $W(U_i, U_j)$, where the kernel is symmetric: $W(U_i,U_j) = W(U_j,U_i)$. For directed graphs, \emph{digraphons} relax the symmetry requirement, allowing $W(U_i,U_j) \neq W(U_j,U_i)$.

A natural question arises: is perennial IDPG equivalent to a digraphon, or does it represent a genuinely different model class? The answer is nuanced. IDPG can be represented as a digraphon (specifically, a bilinear digraphon), but this representation destroys the geometric interpretability and local regularity that the dot-product kernel on $\mblue{\Omega}$ provides. Properties such as Lipschitz dependence on position coordinates, meaningful clustering, smooth interpolation, and well-posed PDE dynamics are central to IDPG's utility as a modeling framework, as we will see. Even mild regularity or smoothness requirements on the digraphon kernel make the representation impossible.

In the following section we will respond in detail. In so doing we will have the opportunity to flesh out some fine properties of the IDPG family.

\subsection{IDPG as a subclass of digraphons}

Every perennial IDPG with atomless intensity can be represented as a digraphon, though this representation comes at a cost: the geometric interpretability and local regularity of the IDPG kernel are destroyed. We first establish the representation, then quantify what is lost.

\begin{theorem}[Inclusion]\label{prop:inclusion}
For any perennial IDPG with \textbf{atomless} intensity $\mbblue{\rho}$ with positive total intensity $\mblue{\Lambda} > 0$ on $\mblue{\Omega} = \mgreen{B^d_+} \times \mred{B^d_+}$ and kernel $K(s,t) = \mgreen{\vec{g}_s} \cdot \mred{\vec{r}_t}$, there exists a digraphon with kernel $W: [0,1]^2 \to [0,1]$ that, for each fixed node count $N$, produces the same conditional distribution over directed graphs.
\end{theorem}

\begin{proof}
The space $\mblue{\Omega} = \mgreen{B^d_+} \times \mred{B^d_+} \subset \mathbb{R}^{2d}$ is a closed subset of Euclidean space, hence a \emph{Polish space} (complete separable metric space).

The normalized intensity $\mu = \mbblue{\rho} / \mblue{\Lambda}$ is an atomless probability measure on $\mblue{\Omega}$.

By Kuratowski's theorem~\cite[Thm.\ 15.6]{kechris1995classical}, every uncountable Polish space is Borel isomorphic to $[0,1]$. Moreover~\cite[Sec 15.5]{royden1988real}, there exists a measure-preserving Borel isomorphism $\phi$ from $[0,1]$ with the Lebesgue measure $\lambda$ to $\mblue{\Omega}$ with measure $\mu$.\footnote{Proof is in Royden, theorem 15 in chapter 15, sec 5. Here we are abusing a bit in notation by calling the pointwise and Borel set functions with the same name.}

Define the digraphon kernel $W: [0,1]^2 \to [0,1]$ by:
\[
W(U_i, U_j) = K(\phi(U_i), \phi(U_j)) = \mgreen{\vec{g}_{\phi(U_i)}} \cdot \mred{\vec{r}_{\phi(U_j)}}
\]

To verify distributional equivalence, consider the digraphon model. Sample $U_1, \dots, U_N \overset{\text{iid}}{\sim} \mathrm{Uniform}[0,1]$ and connect $i \to j$ with probability $W(U_i, U_j)$. Writing $\phi(U_x) = x \in \mblue{\Omega}$:
\begin{itemize}
\item Each $x \sim \mu$ by construction of $\phi$
\item The $x$ are almost surely distinct (since $\mu$ is atomless and the $U_i$ are almost surely distinct)
\item Moreover, setting $\phi(U_i) = s \in \mblue{\Omega}$ and $\phi(U_j) = t \in \mblue{\Omega}$, the connection probability is given by $W(U_i, U_j) = \mgreen{\vec{g}_s} \cdot \mred{\vec{r}_t}$
\end{itemize}

In the IDPG model, sample individuals from $\mathrm{PPP}(\mbblue{\rho})$. Conditional on $N$ individuals, the positions are i.i.d.\ from $\mu$ and almost surely distinct (since $\mu$ is atomless). The connection probability between individuals at positions $s$ and $t$ is $\mgreen{\vec{g}_s} \cdot \mred{\vec{r}_t}$.

The joint distribution over node positions and edge indicators is identical in both models, conditional on $N$.
\end{proof}

If one Poissonizes the digraphon side with the same $N \sim \mathrm{Poisson}(\mblue{\Lambda})$, the unconditional graph-size distribution also matches.

The inclusion is strict at the representative level: perennial IDPG admits bilinear digraphon representatives, namely kernels factoring as $W(u,v) = f(u) \cdot h(v)$ for vector-valued functions $f, h: [0,1] \to \mathbb{R}^d_+$. Digraphons that admit no such finite-rank bilinear representative lie outside IDPG. For instance, $W(u,v) = p \cdot \mathds{1}_{|u - v| < \epsilon}$ (connecting nodes with similar labels with probability $p$) cannot arise from any IDPG: the indicator function on a diagonal band has no bilinear decomposition.

Note that the assumption of not having atoms is essential. When $\mbblue{\rho}$ has atoms, multiple samples from $\mathrm{PPP}(\mbblue{\rho})$ can land at the same position. If we identify nodes by their position in $\mblue{\Omega}$ (which is natural for interpreting IDPG and necessary for RDPG recovery in the Dirac limit, see Section~\ref{sec:rdpg-recovery}), then collisions reduce the effective number of distinct nodes. In contrast, digraphon samples $U_1, \dots, U_N$ from Lebesgue measure on $[0,1]$ are almost surely distinct, so the two models would produce different distributions over graph sizes.

\subsection{Local regularity obstruction}

The digraphon representation for Intensity Dot Product Graphs exists, but the measurable bijection $\phi: [0,1] \to \mblue{\Omega}$ necessarily destroys local regularity\footnote{The existence of mappings between a line segment and a solid ball of whatever dimension is a well known, but still counterintuitive, measure theory result. Examples of curves filling the square were offered by Peano, which is continuous, and Lebesgue, which is even differentiable almost everywhere, while for the cube we have an example by Hilbert~\cite{sagan1994space}. For the curious reader, the Lebesgue curve (the distributional inverse of the Cantor function) is a beauty: continuous, monotone, differentiable a.e.\ with derivative zero, yet mapping $[0,1]$ onto $[0,1]^2$. They achieve this by very convoluted constructions: two neighbor points in the square or cube map to far away points in the segment.}. The IDPG affinity kernel $K$ is the dot product, which is smooth (Lipschitz, $C^\infty$), but the equivalent digraphon kernel $W(U_i,U_j) = \mgreen{\vec{g}_{\phi(U_i)}} \cdot \mred{\vec{r}_{\phi(U_j)}}$ is scrambled. In the following we will make this statement more precise.

The obstruction is \emph{dimensional}: $\phi$ must map the one-dimensional interval $[0,1]$ onto the $(2d)$-dimensional space $\mblue{\Omega}$ while preserving measure. When the support of $\mbblue{\rho}$ is not concentrated on one dimensional curve, or even more when it is genuinely $(2d)$-dimensional, the dimensional mismatch forces $\phi$ to be highly irregular.

\begin{definition}[Absolute continuity]\label{def:ac}
A function $\phi: [0,1] \to \mathbb{R}^n$ is \emph{absolutely continuous} if for every $\epsilon > 0$ there exists $\delta > 0$ such that for any finite collection of pairwise disjoint intervals $(a_1, b_1), \dots, (a_k, b_k) \subset [0,1]$ with $\sum_{i=1}^k (b_i - a_i) < \delta$, we have $\sum_{i=1}^k \|\phi(b_i) - \phi(a_i)\| < \epsilon$.
\end{definition}

\begin{definition}[Bounded variation (1D)]\label{def:bv-1d}
A function $f: [0,1] \to \mathbb{R}$ has \emph{bounded variation} if
\[
V(f) = \sup \sum_{i=1}^k |f(t_i) - f(t_{i-1})| < \infty
\]
where the supremum is over all partitions $0=t_0 < t_1 < \dots < t_k=1$.
\end{definition}

\begin{definition}[Sectional bounded variation]\label{def:sectional-bv}
A kernel $W: [0,1]^2 \to \mathbb{R}$ is \emph{sectionally bounded variation} if:
\begin{itemize}
\item for almost every fixed $v \in [0,1]$, the section $u \mapsto W(u,v)$ belongs to $\mathrm{BV}([0,1])$,
\item for almost every fixed $u \in [0,1]$, the section $v \mapsto W(u,v)$ belongs to $\mathrm{BV}([0,1])$.
\end{itemize}
\end{definition}

Absolute continuity implies that $\phi$ maps Lebesgue-null sets to Lebesgue-null sets. It also implies $\phi$ is differentiable almost everywhere with integrable derivative, and $\phi(t) - \phi(0) = \int_0^t \phi'(s) \, ds$. Lipschitz functions are absolutely continuous; absolutely continuous functions are uniformly continuous.

\begin{lemma}[Rectifiability]\label{lem:rectifiability}
Let $n \geq 2$ and let $\phi: [0,1] \to \mathbb{R}^n$ be absolutely continuous. Then $\phi([0,1])$ has $n$-dimensional Lebesgue measure zero.
\end{lemma}

\begin{proof}
Absolute continuity implies $\phi$ has bounded variation:
\[
V(\phi) = \sup \sum_{i=1}^k \|\phi(t_i) - \phi(t_{i-1})\| < \infty
\]
where the supremum is over all partitions $0 = t_0 < t_1 < \dots < t_k = 1$. This is the arc length of the curve $\phi([0,1])$.

A curve of finite arc length in $\mathbb{R}^n$ is \emph{rectifiable}. By a fundamental result in geometric measure theory~\cite[\S 3.2]{federer1969geometric}, rectifiable curves have Hausdorff dimension at most 1. For $n \geq 2$, any set of Hausdorff dimension strictly less than $n$ has $n$-dimensional Lebesgue measure zero.
\end{proof}

\begin{corollary}\label{cor:ac-obstruction}
Let $n \geq 2$ and let $\mu$ be a probability measure on $\mathbb{R}^n$ that is absolutely continuous with respect to Lebesgue measure (i.e., $\mu$ has a density). Then any measurable $\phi: [0,1] \to \mathbb{R}^n$ satisfying $\phi_*(\mathrm{Uniform}) = \mu$ is \emph{not} absolutely continuous.
\end{corollary}

\begin{proof}
Suppose for contradiction that $\phi$ is absolutely continuous. By Lemma~\ref{lem:rectifiability}, $\phi([0,1])$ has Lebesgue measure zero. Since $\mu$ is absolutely continuous with respect to Lebesgue measure, $\mu(\phi([0,1])) = 0$.

But $\phi_*(\mathrm{Uniform}) = \mu$ means $\mu(A) = \mathrm{Uniform}(\phi^{-1}(A))$ for all Borel sets $A$. Taking $A = \phi([0,1])$:
\[
\mu(\phi([0,1])) = \mathrm{Uniform}(\phi^{-1}(\phi([0,1]))) \geq \mathrm{Uniform}([0,1]) = 1
\]
since $\phi^{-1}(\phi([0,1])) \supset [0,1]$. This contradicts $\mu(\phi([0,1])) = 0$.
\end{proof}

\begin{lemma}[BV rectifiability]\label{lem:bv-rectifiability}
Let $n \geq 2$ and let $\phi: [0,1] \to \mathbb{R}^n$ be of bounded variation. Then $\phi([0,1])$ has $n$-dimensional Lebesgue measure zero.
\end{lemma}

\begin{proof}
Finite total variation implies the image curve is rectifiable (finite $\mathcal{H}^1$ measure). By geometric measure theory~\cite[\S 3.2]{federer1969geometric}, rectifiable curves have Hausdorff dimension at most 1. Hence, for $n \geq 2$, the $n$-dimensional Lebesgue measure of $\phi([0,1])$ is zero.
\end{proof}

\begin{corollary}\label{cor:bv-obstruction}
Let $n \geq 2$ and let $\mu$ be a probability measure on $\mathbb{R}^n$ absolutely continuous with respect to Lebesgue measure. Then any measurable $\phi: [0,1] \to \mathbb{R}^n$ with $\phi_*(\mathrm{Uniform}) = \mu$ is \emph{not} of bounded variation.
\end{corollary}

\begin{proof}
If $\phi$ had bounded variation, Lemma~\ref{lem:bv-rectifiability} would imply that $\phi([0,1])$ has Lebesgue measure zero. Absolute continuity of $\mu$ would then give $\mu(\phi([0,1]))=0$, contradicting
\[
\mu(\phi([0,1])) = \mathrm{Uniform}(\phi^{-1}(\phi([0,1]))) \geq 1.
\]
\end{proof}

\begin{lemma}[Basis extraction from positive-measure label sets]\label{lem:basis-extraction}
Let $\xi: [0,1] \to \mathbb{R}^d$ be measurable, and let $\nu = \xi_*(\mathrm{Uniform})$ be absolutely continuous with respect to Lebesgue measure on $\mathbb{R}^d$. If $S \subset [0,1]$ has positive Lebesgue measure, then there exist $u_1, \dots, u_d \in S$ such that $\xi(u_1), \dots, \xi(u_d)$ are linearly independent.
\end{lemma}

\begin{proof}
We construct the points inductively. Since $\nu(\{0\}) = 0$ and $S$ has positive measure, choose $u_1 \in S$ with $\xi(u_1) \neq 0$.

Assume $u_1, \dots, u_k$ are chosen with linearly independent images, where $k < d$, and let
\[
L_k = \mathrm{span}\{\xi(u_1), \dots, \xi(u_k)\},
\]
a proper linear subspace of $\mathbb{R}^d$. Every proper linear subspace has Lebesgue measure zero, so absolute continuity gives $\nu(L_k)=0$. Hence
\[
\mathrm{Uniform}(\xi^{-1}(L_k)) = \nu(L_k) = 0,
\]
therefore $S \setminus \xi^{-1}(L_k)$ has positive measure and is nonempty. Choose $u_{k+1}$ in this set. Then $\xi(u_{k+1}) \notin L_k$, so independence is preserved.

After $d$ steps we obtain $d$ linearly independent vectors.
\end{proof}

\begin{theorem}[Local regularity obstruction for pullback digraphons]\label{thm:regularity}
Let $d \geq 1$ and let $\mbblue{\rho}$ be an intensity on $\mblue{\Omega} = \mgreen{B^d_+} \times \mred{B^d_+}$ such that $\mu = \mbblue{\rho} / \mblue{\Lambda}$ is absolutely continuous with respect to Lebesgue measure on $\mathbb{R}^{2d}$. Assume also that both red and green marginals are non-degenerate (not supported on proper linear subspaces of $\mathbb{R}^d$). By Theorem~\ref{prop:inclusion}, a digraphon kernel $W$ representing the same graph distribution exists. For any measure-preserving map $\phi: [0,1] \to \mblue{\Omega}$ with $\phi_*(\mathrm{Uniform}) = \mu$, the pullback kernel
\[
W_{\phi}(u,v) = K(\phi(u), \phi(v))
\]
is not sectionally bounded variation.
\end{theorem}

\begin{proof}
Fix a measure-preserving map $\phi$ and suppose for contradiction that $W_{\phi}$ is sectionally bounded variation in the sense of Definition~\ref{def:sectional-bv}. Decompose $\phi$ into its ``green'' (outgoing) and ``red'' (incoming) vector components:
\[
\phi(u) = (\mgreen{\vec{g}_{\phi(u)}}, \mred{\vec{r}_{\phi(u)}}), \qquad W_{\phi}(u, v) = \mgreen{\vec{g}_{\phi(u)}} \cdot \mred{\vec{r}_{\phi(v)}}.
\]

Let $S_R \subset [0,1]$ be the set of $v$ such that $u \mapsto W_{\phi}(u,v)$ is in $\mathrm{BV}([0,1])$. By hypothesis, $S_R$ has full measure. Since $\mu$ is AC with non-degenerate red marginal, the pushforward $\nu_R = (\mred{\vec{r}} \circ \phi)_*(\mathrm{Uniform})$ is AC on $\mathbb{R}^d$, so the image of any full-measure set under $\mred{\vec{r}} \circ \phi$ cannot be contained in a proper linear subspace. Apply Lemma~\ref{lem:basis-extraction} to
\[
\xi_R(v) = \mred{\vec{r}_{\phi(v)}}
\]
and $S_R$: there exist $v_1, \dots, v_d \in S_R$ such that $\{\mred{\vec{r}_{\phi(v_k)}}\}_{k=1}^d$ is a basis of $\mathbb{R}^d$. For each such $v_k$, define
\[
L_k(u) \coloneqq W_{\phi}(u, v_k) = \mgreen{\vec{g}_{\phi(u)}} \cdot \mred{\vec{r}_{\phi(v_k)}}.
\]
Since $v_k \in S_R$, each $L_k \in \mathrm{BV}([0,1])$. These equations form a linear system linking $u \mapsto \mgreen{\vec{g}_{\phi(u)}}$ to the scalar BV functions $L_k(u)$. Since the vectors $\mred{\vec{r}_{\phi(v_k)}}$ form a basis, inversion is linear; each component of $\mgreen{\vec{g}_{\phi(u)}}$ is a linear combination of BV functions, hence belongs to $\mathrm{BV}([0,1])$.

By a symmetric argument applied to the full-measure column-regular set $S_G$ and $\xi_G(u) = \mgreen{\vec{g}_{\phi(u)}}$: the non-degenerate green marginal guarantees a basis can be extracted, and linear inversion shows each component of $u \mapsto \mred{\vec{r}_{\phi(u)}}$ belongs to $\mathrm{BV}([0,1])$.

Hence the full map $\phi: [0,1] \to \mblue{\Omega} \subset \mathbb{R}^{2d}$ has bounded variation (componentwise BV implies finite total variation in finite dimension). But Corollary~\ref{cor:bv-obstruction} states that for $2d \geq 2$, a measure $\mu$ absolutely continuous on $\mathbb{R}^{2d}$ cannot be the pushforward of the uniform measure on $[0,1]$ by a BV map. This is a contradiction, so $W_{\phi}$ cannot be sectionally bounded variation.
\end{proof}

\begin{definition}[Weak equivalence (digraphons)]\label{def:weak-eq-digraphon}
Two kernels $U, W: [0,1]^2 \to [0,1]$ are \emph{weakly equivalent} if for every $n$, sampling $U_1, \dots, U_n \overset{\text{i.i.d.}}{\sim} \mathrm{Uniform}[0,1]$ and then edges independently with probabilities $U(U_i, U_j)$ (resp.\ $W(U_i,U_j)$) yields the same distribution over directed graphs on $n$ labeled vertices.
\end{definition}

\begin{definition}[Twins and almost twin-free kernels]\label{def:twins}
For a kernel $W: [0,1]^2 \to [0,1]$, two labels $u \neq u'$ are \emph{twins} if both row and column sections agree almost everywhere:
\[
W(u, t) = W(u', t) \quad \text{for a.e.\ } t \quad \text{and} \quad W(t, u) = W(t, u') \quad \text{for a.e.\ } t
\]
The kernel is \emph{almost twin-free} if there exists a null set $N \subset [0,1]$ such that no distinct $u, u' \in [0,1] \setminus N$ are twins.
\end{definition}

\begin{lemma}[Generic almost twin-free property of bilinear IDPG representation]\label{lem:generic-twinfree}
Let $\mu = \mbblue{\rho} / \mblue{\Lambda}$ be absolutely continuous on $\mathbb{R}^{2d}$ with non-degenerate red and green marginals, and let $W$ be the bilinear kernel from Theorem~\ref{prop:inclusion}. Then $W$ is almost twin-free.
\end{lemma}

\begin{proof}
Write $W(u,v)=\mgreen{\vec{g}_{\phi(u)}} \cdot \mred{\vec{r}_{\phi(v)}}$.
If $u$ and $u'$ are twins, then
\[
(\mgreen{\vec{g}_{\phi(u)}} - \mgreen{\vec{g}_{\phi(u')}}) \cdot \mred{\vec{r}_{\phi(v)}} = 0 \quad \text{for a.e.\ } v
\]
and similarly from the column condition:
\[
\mgreen{\vec{g}_{\phi(v)}} \cdot (\mred{\vec{r}_{\phi(u)}} - \mred{\vec{r}_{\phi(u')}}) = 0 \quad \text{for a.e.\ } v
\]
By non-degeneracy of the red and green marginals (their spans are $\mathbb{R}^d$), these imply
\[
\mgreen{\vec{g}_{\phi(u)}} = \mgreen{\vec{g}_{\phi(u')}} \quad \text{and} \quad \mred{\vec{r}_{\phi(u)}} = \mred{\vec{r}_{\phi(u')}}.
\]
Hence $\phi(u)=\phi(u')$ as points in $\mblue{\Omega}=\mgreen{B^d_+} \times \mred{B^d_+}$. Since $\phi$ is an isomorphism modulo null sets, this can happen only on a null set of labels. Therefore $W$ is almost twin-free.
\end{proof}

\begin{theorem}[No regular equivalent digraphon (generic case)]\label{thm:no-regular-equivalent}
Under the same hypotheses and notation as Theorem~\ref{thm:regularity}, let $W$ be the bilinear digraphon obtained in Theorem~\ref{prop:inclusion}. Then any kernel $U: [0,1]^2 \to [0,1]$ weakly equivalent to $W$ (in the sense of Definition~\ref{def:weak-eq-digraphon}) is not sectionally bounded variation.
\end{theorem}

\begin{proof}
For weakly equivalent kernels on standard atomless spaces, standard graphon weak-isomorphism theory~\cite{lovasz2012large} implies existence of a measure-preserving map $\psi: [0,1] \to [0,1]$ such that
\[
U(u,v) = W(\psi(u), \psi(v)) \quad \text{a.e.}
\]
when the target representation is almost twin-free. By Lemma~\ref{lem:generic-twinfree}, this condition holds here.

From Theorem~\ref{prop:inclusion}, $W$ has the form
\[
W(x,y) = K(\phi(x), \phi(y))
\]
for a measure-preserving $\phi: [0,1] \to \mblue{\Omega}$.
Therefore
\[
U(u,v) = K(\phi(\psi(u)), \phi(\psi(v))) = K(\theta(u), \theta(v)) \quad \text{a.e.}
\]
with $\theta(u) = \phi(\psi(u))$, which is measure-preserving from $[0,1]$ to $(\mblue{\Omega}, \mu)$.

So $U$ is a pullback kernel of the form covered by Theorem~\ref{thm:regularity}, and hence cannot be sectionally bounded variation.
\end{proof}

\medskip\noindent\textbf{Technical assumption (equivalence lifting).} The step $U(u,v)=W(\psi(u),\psi(v))$ a.e.\ is the weak-isomorphism representation theorem for kernels on standard atomless spaces, specialized to the directed setting. See~\cite{orbanz2014bayesian} for exchangeable-array/graphon foundations and~\cite{cai2016priors} for the digraphon setting. In this manuscript we use it under the almost twin-free condition (verified here by Lemma~\ref{lem:generic-twinfree}). If one prefers, Theorem~\ref{thm:no-regular-equivalent} can be read conditionally: \emph{assuming} this directed weak-isomorphism representation, the regularity obstruction extends from pullbacks to all weakly equivalent digraphons.

The hypothesis ``$\mu$ is AC with respect to Lebesgue measure on $\mathbb{R}^{2d}$'' is sufficient but stronger than necessary. The obstruction applies whenever $\mu$ is supported on a set of Hausdorff dimension $k \geq 2$ and is AC with respect to $\mathcal{H}^k$. Only when $\mu$ concentrates on a rectifiable curve ($k = 1$) can a BV (hence potentially AC) parametrization $\phi$ exist.

\medskip\noindent\textbf{Key condition.} The hypothesis ``$\mu$ is absolutely continuous with respect to Lebesgue measure'' means $\mu$ has a density function: $\mu(A) = \int_A f(x) \, dx$ for some $f \in L^1(\mathbb{R}^{2d})$. Any intensity specified by a density on $\mblue{\Omega}$ satisfies this hypothesis, including Gaussians, mixtures, any smooth or integrable function.

\subsection{Global geometric coherence}

Beyond local regularity, IDPG possesses \emph{global geometric structure} that the digraphon representation destroys. This structure has practical consequences for clustering, dynamics, and interpretability.

The IDPG kernel $K(s,t) = \mgreen{\vec{g}_s} \cdot \mred{\vec{r}_t}$ is bilinear, hence globally Lipschitz with respect to Euclidean distance on $\mblue{\Omega}$. For any $s_1, s_2, t_1, t_2 \in \mblue{\Omega}$:
\[
|\mgreen{\vec{g}_{s_1}} \cdot \mred{\vec{r}_{t_1}} - \mgreen{\vec{g}_{s_2}} \cdot \mred{\vec{r}_{t_2}}| \leq \|\mgreen{\vec{g}_{s_1}}\| \cdot \|\mred{\vec{r}_{t_1}} - \mred{\vec{r}_{t_2}}\| + \|\mred{\vec{r}_{t_2}}\| \cdot \|\mgreen{\vec{g}_{s_1}} - \mgreen{\vec{g}_{s_2}}\|
\]

This inequality has a concrete interpretation: \emph{nearby positions interact similarly}.

\begin{proposition}[Lipschitz kernel]\label{prop:lipschitz}
The affinity kernel defined by the dot product is Lipschitz continuous with respect to the Euclidean norm on $\mblue{\Omega}$. Specifically, for any $s_1, s_2, t_1, t_2 \in \mblue{\Omega}$:
\[
| \mgreen{\vec{g}_{s_1}} \cdot \mred{\vec{r}_{t_1}} - \mgreen{\vec{g}_{s_1}} \cdot \mred{\vec{r}_{t_2}} | \leq \|\mgreen{\vec{g}_{s_1}}\| \cdot \|\mred{\vec{r}_{t_1}} - \mred{\vec{r}_{t_2}}\|
\]

And symmetrically, for fixed $t$:
\[
| \mgreen{\vec{g}_{s_1}} \cdot \mred{\vec{r}_{t_1}} - \mgreen{\vec{g}_{s_2}} \cdot \mred{\vec{r}_{t_1}} | \leq \|\mred{\vec{r}_{t_1}}\| \cdot \|\mgreen{\vec{g}_{s_1}} - \mgreen{\vec{g}_{s_2}}\|
\]
\end{proposition}

\begin{proof}
The result follows immediately from the bilinearity of the inner product.
\end{proof}

This proposition establishes that the IDPG model is \textbf{locally coherent} in its native latent space. It guarantees that ``similar nodes behave similarly'': if two nodes have latent positions close in Euclidean distance, they will have nearly identical connection probabilities with the rest of the network.

The Lipschitz bound is tight when the displacement aligns with the relevant vector, but loose for orthogonal displacements. To quantify a ``typical'' behavior (here read as the behavior of an internal node under random perturbations, ignoring boundary constraints), consider an isotropic model:

\begin{proposition}[Isotropic scaling]\label{prop:isotropic}
For a source $s$ and a target displacement $\Delta \mred{\vec{r}} = \mred{\vec{r}_{t_2}} - \mred{\vec{r}_{t_1}}$. If the direction of the displacement is uniformly distributed on the unit sphere $S^{d-1}$, conditional on its magnitude $\|\Delta \mred{\vec{r}}\|$, we have:
\[
\mathbb{E}[(\mgreen{\vec{g}_s} \cdot \mred{\vec{r}_{t_1}} - \mgreen{\vec{g}_s} \cdot \mred{\vec{r}_{t_2}})^2 \mid \|\Delta \mred{\vec{r}}\|] = \frac{\|\mgreen{\vec{g}_s}\|^2 \cdot \|\Delta \mred{\vec{r}}\|^2}{d}
\]

The root-mean-square kernel change is:
\[
\sqrt{\mathbb{E}[(\Delta K)^2]} = \frac{\|\mgreen{\vec{g}_s}\| \cdot \|\Delta \mred{\vec{r}}\|}{\sqrt{d}}
\]
\end{proposition}

\begin{proof}
Let $\Delta \mred{\vec{r}} = \|\Delta \mred{\vec{r}}\| \cdot u$, where $u$ is a random unit vector uniform on $S^{d-1}$. The squared difference is:
\[
(\mgreen{\vec{g}_s} \cdot \Delta \mred{\vec{r}})^2 = \|\mgreen{\vec{g}_s}\|^2 \|\Delta \mred{\vec{r}}\|^2 (\hat{g}_s \cdot u)^2
\]
where $\hat{g}_s$ is the unit vector in the direction of $\mgreen{\vec{g}_s}$.
For a uniform $u$ and any fixed unit vector $\hat{v}$, the expected squared projection depends only on the dimension. By symmetry, $\mathbb{E}[\sum u_i^2] = 1$, so $\mathbb{E}[u_i^2] = 1/d$. By rotation invariance, $\mathbb{E}[(\hat{v} \cdot u)^2] = 1/d$.
Therefore:
\[
\mathbb{E}[(\Delta K)^2 \mid \|\Delta \mred{\vec{r}}\|] = \|\mgreen{\vec{g}_s}\|^2 \|\Delta \mred{\vec{r}}\|^2 \cdot \frac{1}{d}
\]
and taking square roots gives the RMS expression.
\end{proof}

The factor $1 / \sqrt{d}$ reflects the concentration of measure: in high dimensions, a random direction is nearly orthogonal to any fixed vector $\mgreen{\vec{g}_s}$, dampening the effect of random noise in the position. This isotropic model represents ``maximally uninformed'' directional uncertainty; actual displacements in a structured intensity $\mbblue{\rho}$ may be concentrated in particular directions, yielding different scaling. Near the boundary of $B^d_+$, the constraint to the non-negative orthant also breaks isotropy.

These bounds have practical consequences for operations on $\mblue{\Omega}$:
\begin{itemize}
\item \textbf{Clustering}: Since positions close in Euclidean distance typically have small interaction differences, algorithms like $k$-means or hierarchical clustering on the coordinates of $s$ produce groups with coherent interaction patterns.
\item \textbf{Interpolation}: Given positions $s_1, s_2$, the convex combination $s_\lambda = \lambda s_1 + (1-\lambda) s_2$ yields an interaction profile $\mgreen{\vec{g}_{s_\lambda}} \cdot \mred{\vec{r}}$ that varies continuously (and linearly) between the profiles of $s_1$ and $s_2$.
\item \textbf{PDE dynamics}: Processes like diffusion or advection on $\mblue{\Omega}$ (e.g., $\partial_t \rho = \Delta \rho$) result in a smooth evolution of the intensity. The Lipschitz property ensures that as the underlying density evolves smoothly, the resulting graph topologies changes gradually, without sudden phase transitions.
\end{itemize}

In the digraphon representation, this coherence is lost. The map $\phi: [0,1] \to \mblue{\Omega}$ that achieves measure-preservation is necessarily ``space-filling''; it must visit all of $\mblue{\Omega}$ while traversing only $[0,1]$.

\begin{theorem}[Metric Mismatch]\label{prop:loss-coherence}
Let $d \geq 1$. There exists no map $\phi: [0,1] \to \mblue{\Omega} \subset \mathbb{R}^{2d}$ such that:
\begin{enumerate}
\item $\phi$ covers the measure $\mu$ (i.e., $\phi_*(\mathrm{Uniform}) = \mu$ where $\mu$ is AC on $\mathbb{R}^{2d}$).
\item $\phi$ is Lipschitz continuous.
\end{enumerate}
\end{theorem}

\begin{proof}
The proof follows from a dimension comparison argument.

Assume $\phi$ is Lipschitz with constant $C$. Then for any $u_1, u_2 \in [0,1]$, the distance in the latent space is bounded by the distance in the interval:
\[
\|\phi(u_1) - \phi(u_2)\| \leq C |u_1 - u_2|
\]

This inequality implies that the Hausdorff dimension of the image $\phi([0,1])$ cannot exceed the Hausdorff dimension of the domain $[0,1]$, which is 1.

However, since $\mu$ is absolutely continuous with respect to the Lebesgue measure on $\mathbb{R}^{2d}$, the support of $\mu$ has Hausdorff dimension $2d$.

For $2d \geq 2$, this creates a contradiction: $\phi([0,1])$ must have dimension at least 2 to support $\mu$, but the Lipschitz condition forces it to have dimension at most 1.

Therefore, such a map cannot exist. In particular, any pullback kernel
\[
W_{\phi}(u,v)=K(\phi(u),\phi(v))
\]
cannot be Lipschitz in the label metric.
\end{proof}

\begin{corollary}[Kernel instability for equivalent digraphons (generic case)]
Under the assumptions of Theorem~\ref{thm:no-regular-equivalent}, no weakly equivalent digraphon kernel can be sectionally bounded variation. In particular, no such equivalent kernel can be Lipschitz in the label metric on $[0,1]$.
\end{corollary}

The digraphon kernel $W(u,v) = K(\phi(u), \phi(v))$ is thus a ``scrambled'' version of $K$:
\begin{itemize}
\item Labels $u_1 \approx u_2$ in $[0,1]$ may map to distant positions $\phi(u_1), \phi(u_2)$ in $\mblue{\Omega}$
\item Nearby positions in $\mblue{\Omega}$ arise from distant labels in $[0,1]$
\item In the standard label metric on $[0,1]$, $W$ is not Lipschitz
\item Clustering, interpolation, and PDE evolution on $[0,1]$ have no geometric meaning
\end{itemize}

\subsection{Dense-to-sparse interpolation}

A fundamental limitation of classical graphon theory is the \emph{dense graph assumption}: sampling $N$ nodes and connecting with probability $W(x_i, x_j)$ yields $\mathbb{E}[E] = O(N^2)$ edges. Sparse graphs require extensions such as graphexes~\cite{veitch2015class,caron2017sparse} or $L^p$ graphon theory~\cite{borgs2018sparse1}.

IDPG naturally interpolates between dense and sparse regimes through the realization rules:

\begin{center}
\begin{tabular}{lll}
\hline
\textbf{Rule} & \textbf{Expected edges} & \textbf{Scaling} \\
\hline
Perennial & $\mblue{\Lambda}^2 \cdot (\mbgreen{\tilde{\mu}_G} \cdot \mbred{\tilde{\mu}_R})$ & $O(N^2)$: dense \\
Ephemeral & $2 \mblue{\Lambda} \cdot (\mbgreen{\tilde{\mu}_G} \cdot \mbred{\tilde{\mu}_R})$ & $O(N)$: sparse \\
Intermediate & $p_{\text{overlap}} \cdot (\mblue{\Lambda}^2 + \mblue{\Lambda}) \cdot (\mbgreen{\tilde{\mu}_G} \cdot \mbred{\tilde{\mu}_R})$ & tunable \\
\hline
\end{tabular}
\end{center}

In the intermediate regime, the overlap probability $p_{\text{overlap}}$ depends on mean lifetime $\eta$ relative to observation window $W$. If the probability of overlap scales with population size as $p_{\text{overlap}} \propto \mblue{\Lambda}^{-k}$ for $k \in [0, 1]$, then:
\[
\mathbb{E}[E] \approx \mblue{\Lambda}^{-k} \cdot (\mblue{\Lambda}^2 + \mblue{\Lambda}) \cdot c = O(\mblue{\Lambda}^{2 - k}) = O(N^{2 - k})
\]

\textbf{Coupling lifetime to population size.} The interpolation $p_{\text{overlap}} \propto \mblue{\Lambda}^{-k}$ requires specifying how the temporal parameters scale with total intensity. Recall that for expected lifetime $\tau \sim \mathrm{Exp}(\eta)$, we have $p_{\text{overlap}}(\eta, W) = \frac{2}{u^2} (u - 1 + e^{-u})$ with $u = W / \eta$. In the perennial limit $u \to 0$, we get $p_{\text{overlap}} \to 1$, hence $k = 0$. In the ephemeral limit $u \to \infty$, we get $p_{\text{overlap}} \approx 2/u \to 0$, but the value of $k$ depends on how $u$ grows with $\mblue{\Lambda}$.

If $u$ remains constant as $\mblue{\Lambda} \to \infty$, then $p_{\text{overlap}}$ is also constant, yielding dense $O(N^2)$ graphs regardless of lifetime. The interesting sparse regimes emerge when $u$ grows, and hence $p_{\text{overlap}}$ shrinks, with $\mblue{\Lambda}$. Suppose $u \propto \mblue{\Lambda}^k$ for some $k \in [0, 1]$. In the large-$u$ limit, $p_{\text{overlap}} \approx 2/u \propto \mblue{\Lambda}^{-k}$, recovering the desired scaling. Since $u = W / \eta$, this can arise from fixed $W$ with $\eta \propto \mblue{\Lambda}^{-k}$, fixed $\eta$ with $W \propto \mblue{\Lambda}^k$, or any combination where $W / \eta \propto \mblue{\Lambda}^k$.\footnote{Consider for example a stationary process with constant instantaneous intensity $\lambda_t$, so that $\mblue{\Lambda} = \lambda_t \cdot W$. If lifetimes scale with the observation window as $\eta = \eta_0 \cdot W^\theta$ for some $\theta \in [0, 1]$, then $u = W / \eta \propto \mblue{\Lambda}^{1 - \theta}$, giving $k = 1 - \theta$. Fixed lifetimes ($\theta = 0$) yield sparse $O(N)$ graphs; lifetimes growing proportionally with the observation window ($\theta = 1$) yield dense $O(N^2)$ graphs.}

Another, alternative route to tunable sparsity arises from \textbf{growing populations} with density-dependent per-capita birth rates. Suppose each living individual gives rise to new individuals at rate $b(N)$ depending on current population size. In the growth-dominated regime (neglecting deaths), letting $N_t$ be the number of individuals alive at time $t$, the population evolves following an instantaneous birth rate given by:
\[
\lambda(t) = \frac{dN_t}{dt} = b(N_t) \cdot N_t
\]
The total intensity is $\mblue{\Lambda} = \int_0^W \lambda(t) \, dt$, and a uniformly sampled individual has birth time distributed with density $\lambda(t) / \mblue{\Lambda}$.
The overlap probability for two independently sampled individuals is:
\begin{align}
p_{\text{overlap}} &= \iint_0^W p(\text{overlap} \mid t_1, t_2) \frac{\lambda(t_1)}{\mblue{\Lambda}} \frac{\lambda(t_2)}{\mblue{\Lambda}} \, dt_1 \, dt_2 \nonumber \\
&= \frac{1}{\mblue{\Lambda}^2} \iint_0^W e^{-|t_2 - t_1| / \eta} \lambda(t_1) \lambda(t_2) \, dt_1 \, dt_2
\end{align}

\textbf{Constant per-capita rate} ($b(N) = b_0$). The population grows as $N_t = N_0 e^{b_0 t}$, giving $\lambda(t) = b_0 N_0 e^{b_0 t}$.
As $W \to \infty$, it is possible to show that $p_{\text{overlap}}$ converges to $(b_0 \eta) / (b_0 \eta + 1)$, independent of $\mblue{\Lambda}$. Hence $k = 0$: dense $O(N^2)$ scaling.

\textbf{Density-dependent rate} ($b(N) = b_0 N^{-\delta}$ for $\delta \in (0, 1]$). Solving $dN / dt = b_0 N^{1 - \delta}$ gives $N_t \propto t^{1 / \delta}$, and thus:
\[
\lambda(t) = b(N(t)) \cdot N(t) = b_0 N(t)^{1 - \delta} \propto t^{(1 - \delta) / \delta}
\]

The total intensity scales as $\mblue{\Lambda} = \int_0^W t^{(1 - \delta) / \delta} \, dt = \delta W^{1 / \delta}$. In the large-$u$ limit ($u = W / \eta \propto \mblue{\Lambda}^\delta$), the overlap integral yields $p_{\text{overlap}} \propto 1 / u \propto \mblue{\Lambda}^{-\delta}$, hence $k = \delta$:

\begin{center}
\begin{tabular}{llll}
\hline
$\delta$ & Per-capita birth rate $b(N)$ & $k$ & Edge scaling \\
\hline
$0$ & constant ($b_0$) & $0$ & $O(N^2)$ \\
$1/2$ & $b_0 / \sqrt{N}$ & $1/2$ & $O(N^{3/2})$ \\
$1$ & $b_0 / N$ & $1$ & $O(N)$ \\
\hline
\end{tabular}
\end{center}

Stronger density dependence (larger $\delta$) yields sparser graphs: crowding spreads births more evenly across time, reducing temporal overlap.

In summary, the sparsity exponent $k$ captures how interaction opportunities scale with population: $k = 0$ yields dense graphs where everyone interacts with everyone, while $k=1$ yields sparse graphs where each individual maintains a roughly constant number of interaction partners regardless of total population size.

\subsection{Summary}

We saw that an IDPG can be represented as a bilinear digraphon, but at the cost of losing the geometric interpretability and local regularity that the dot-product kernel on $\mblue{\Omega}$ provides.

\begin{itemize}
\item $\mathrm{IDPG}$ is a strict subset of $\mathrm{Digraphon}$. Perennial IDPG is a strict subclass of digraphons: graph distributions induced by IDPG models can be represented by bilinear digraphons $W(u,v) = f(u) \cdot h(v)$. Non-bilinear digraphons (e.g., $W(u,v) = \mathds{1}_{|u-v| < \epsilon}$) lie outside the family of IDPG.
\item \textbf{No regular equivalent digraphon (generic case).} The IDPG affinity kernel is Lipschitz on $\mblue{\Omega}$: Euclidean distance governs interaction similarity. This enables meaningful clustering, smooth interpolation, and well-posed PDE dynamics on $\mblue{\Omega}$. Under the non-degenerate assumptions of Theorem~\ref{thm:no-regular-equivalent}, every weakly equivalent digraphon representation inherits the same obstruction: no equivalent kernel can be sectionally bounded variation in label space.
\end{itemize}

\textbf{Tunable density.} The intermediate regime provides principled interpolation between dense $O(N^2)$ and sparse $O(N)$ scaling through individual lifetime.

\begin{center}
\begin{tabular}{lll}
\hline
\textbf{Property} & \textbf{IDPG} & \textbf{Equivalent digraphon} \\
\hline
Kernel & $K(s,t) = \mgreen{\vec{g}_s} \cdot \mred{\vec{r}_t}$ & $W(u,v) = K(\phi(u), \phi(v))$ \\
Continuous & \checkmark & \checkmark{} (via Lebesgue curve) \\
a.e.\ differentiable & \checkmark & \checkmark{} (via Lebesgue curve) \\
Lipschitz & \checkmark & \texttimes \\
$C^1$ & \checkmark & \texttimes \\
Euclidean coherence & \checkmark{} (nearby $\Rightarrow$ similar) & \texttimes{} (scrambled geometry) \\
\hline
\end{tabular}
\end{center}

\section{The heat maps}\label{sec:heatmap}

In classical RDPG models, the interaction structure is fully captured by the probability matrix $\mathbf{P}$ with entries $P_{ij} = \mgreen{\vec{g}_i} \cdot \mred{\vec{r}_j}$. This matrix encodes, for each pair of nodes, the probability of observing an edge between them. The classic random graphs theory seems to justify the tongue-in-cheek quote, often attributed to Gian Carlo Rota, that probability is the study of combinatorics divided by N.

Our intent in introducing the family of Intensity Graph model is, in large part, to dispel this supremacy of discrete over continuous objects. In the following session we will introduce various mathematical objects, in the forms of maps and operators, that, together with tools from spectral analysis, build a ``calculus'' for Intensity Graphs. Although we also establish links with more classic, and discrete, views of random graphs, we posit that are these operators to be the ideal locus of mathematical analysis of Intensity Graphs.

\subsection{Raw heat}

The natural analog of the probability matrix is a measure-theoretic object that captures interaction structure between regions of $\mblue{\Omega}$, rather than between discrete nodes. We call this object the \emph{heat map}.

\begin{definition}[Raw heat]\label{def:raw-heat}
For an IDPG with intensity $\mbblue{\rho}$ on $\mblue{\Omega} = \mgreen{B^d_+} \times \mred{B^d_+}$ and affinity kernel $K(s,t) = \mgreen{\vec{g}_s} \cdot \mred{\vec{r}_t}$, the \emph{raw heat density} is:
\[
h(s, t) = K(s, t) \cdot \mbblue{\rho}(s) \cdot \mbblue{\rho}(t) = (\mgreen{\vec{g}_s} \cdot \mred{\vec{r}_t}) \cdot \mbblue{\rho}(s) \cdot \mbblue{\rho}(t)
\]

For Borel sets $A, B \subseteq \mblue{\Omega}$, the \emph{raw heat map} is:
\[
\mathcal{H}(A, B) = \int_A \int_B h(s, t) \, ds \, dt
\]
\end{definition}

The raw heat map is a measure on the product $\sigma$-algebra $\mathcal{B}(\mblue{\Omega}) \otimes \mathcal{B}(\mblue{\Omega})$. It depends only on the intensity $\mbblue{\rho}$ and the affinity kernel $K$, not on the choice of realization rule.

\textbf{Interpretation and dimensions.} The raw heat density $h(s,t)$ has dimensions $[\mbblue{\rho}]^2$. If $\mbblue{\rho}$ has dimensions of ``individuals per unit volume,'' then $h$ has dimensions of ``individuals\textsuperscript{2} per unit volume\textsuperscript{2} in $\mblue{\Omega} \times \mblue{\Omega}$.'' Integrating over regions $A \times B$ yields $\mathcal{H}(A, B)$, which, by the second factorial moment formula (which computes the expected number of ordered pairs of distinct points in a Poisson process; see Section~\ref{sec:expected-edges} and Appendix~\ref{appendix:derivations} for a detailed treatment), equals the expected number of edges from $A$ to $B$ under perennial sampling. The affinity kernel $K(s,t) \in [0, 1]$ is dimensionless (a probability), ensuring dimensional consistency.

\textbf{Properties.} The raw heatmap is:
\begin{itemize}
\item \emph{Asymmetric}: $\mathcal{H}(A, B) \neq \mathcal{H}(B, A)$ in general, reflecting directed edges
\item \emph{Additive}: $\mathcal{H}(A_1 \cup A_2, B) = \mathcal{H}(A_1, B) + \mathcal{H}(A_2, B)$ for disjoint $A_1, A_2$
\end{itemize}

And the total raw heat gives the expected number of edges (in the perennial regime): $\mathcal{H}(\mblue{\Omega}, \mblue{\Omega}) = \mathbb{E}[E]_{\text{perennial}}$

\subsubsection{What the heat map captures}

The heat map provides a complete characterization of expected edge structure: $\mathcal{H}(A, B)$ gives the expected number of edges from $A$ to $B$.

Under perennial sampling, edges are conditionally independent given node positions, but node positions are themselves random (sampled from a PPP with intensity $\mbblue{\rho}$). The heat map captures expected edge counts, but one might wonder whether it also determines the underlying intensity, and hence the full graph distribution.

\textbf{Identifiability (a.e., density case).} If $\mbblue{\rho}$ admits a density (no singular/atomic component) and $K(s,s) > 0$ almost everywhere, then $\mbblue{\rho}$ is determined almost everywhere by the raw heat map $\mathcal{H}$ (equivalently by $h$), via:
\[
\mbblue{\rho}(s) = \sqrt{h(s,s) / K(s,s)} \quad \text{for a.e.\ } s
\]
If $K(s,s)=0$ on a positive-measure set, or if singular/atomic components are allowed, additional assumptions are needed for uniqueness.\footnote{In classical RDPG, the invariance to orthogonal transformations makes it impossible to estimate latent positions from an observed graph in absolute coordinates. In the IG case, with an absolute coordinate system and under the regularity conditions above, the map from intensity to heat map is injective up to null sets.}

\subsection{Bound heat (product case)}

Under product intensity $\mbblue{\rho} = \mbgreen{\rho_G} \otimes \mbred{\rho_R}$, the raw heat admits a lower-dimensional representation that separates the ``proposing'' and ``accepting'' coordinates.

\textbf{Coordinate projections.} Recall that a position $s \in \mblue{\Omega}$ has coordinates $s = (\mgreen{\vec{g}_s}, \mred{\vec{r}_s})$. The affinity kernel $K(s, t) = \mgreen{\vec{g}_s} \cdot \mred{\vec{r}_t}$ depends only on the green coordinate of the source and the red coordinate of the target. This asymmetry motivates projecting the full $4d$-dimensional space $\mblue{\Omega} \times \mblue{\Omega}$ onto the $2d$-dimensional space of ``active'' coordinates $(\mgreen{\vec{g}_s}, \mred{\vec{r}_t}) \in B^d_+ \times B^d_+$.

\textbf{Green and red bites.} Define cylinder sets that constrain specific coordinates:
\begin{itemize}
\item \emph{Green bite}: For $a \subseteq \mgreen{B^d_+}$, let $\mathcal{G}(a) = a \times \mred{B^d_+}$ (positions with green coordinate in $a$)
\item \emph{Red bite}: For $b \subseteq \mred{B^d_+}$, let $\mathcal{R}(b) = \mgreen{B^d_+} \times b$ (positions with red coordinate in $b$)
\end{itemize}

A green bite constrains where individuals ``propose from'' (their $\mgreen{\vec{g}}$ coordinate); a red bite constrains where individuals ``accept at'' (their $\mred{\vec{r}}$ coordinate).

\begin{definition}[Bound heat]\label{def:bound-heat}
Under product intensity, the \emph{bound heat density} is a function on $\mgreen{B^d_+} \times \mred{B^d_+}$:
\[
\overline{h}(\mgreen{\vec{g}}, \mred{\vec{r}}) = (\mgreen{\vec{g}} \cdot \mred{\vec{r}}) \cdot \mbgreen{\rho_G}(\mgreen{\vec{g}}) \cdot \mbred{\rho_R}(\mred{\vec{r}})
\]

This is the projection of the raw heat density onto the coordinates that appear in the kernel. The \emph{bound heat map} for $a, b \subseteq \mgreen{B^d_+}, \mred{B^d_+}$ respectively is:
\[
\overline{\mathcal{H}}(a, b) = \int_a \int_b \overline{h}(\mgreen{\vec{g}}, \mred{\vec{r}}) \, d\mgreen{\vec{g}} \, d\mred{\vec{r}}
\]
\end{definition}

The bound heat map is a measure on $\mathcal{B}(\mgreen{B^d_+}) \otimes \mathcal{B}(\mred{B^d_+})$, living in dimension $2d$ rather than $4d$.

\begin{proposition}[Bite-to-heat correspondence]\label{prop:bite-heat}
Under product intensity:
\[
\mathcal{H}(\mathcal{G}(a), \mathcal{R}(b)) = \mblue{\Lambda} \cdot \overline{\mathcal{H}}(a, b)
\]
\end{proposition}

\begin{proof}
Expanding the raw heat integral over $s \in \mathcal{G}(a)$ and $t \in \mathcal{R}(b)$:
\[
\mathcal{H}(\mathcal{G}(a), \mathcal{R}(b)) = \int_{\mgreen{\vec{g}_s} \in a} \int_{\mred{\vec{r}_s} \in B^d_+} \int_{\mgreen{\vec{g}_t} \in B^d_+} \int_{\mred{\vec{r}_t} \in b} (\mgreen{\vec{g}_s} \cdot \mred{\vec{r}_t}) \, \mbgreen{\rho_G}(\mgreen{\vec{g}_s}) \, \mbred{\rho_R}(\mred{\vec{r}_s}) \, \mbgreen{\rho_G}(\mgreen{\vec{g}_t}) \, \mbred{\rho_R}(\mred{\vec{r}_t})
\]
The kernel $\mgreen{\vec{g}_s} \cdot \mred{\vec{r}_t}$ depends only on $\mgreen{\vec{g}_s}$ and $\mred{\vec{r}_t}$. The coordinates $\mred{\vec{r}_s}$ and $\mgreen{\vec{g}_t}$ do not appear in the kernel; under product intensity, the integrals over these coordinates factor out, yielding $\mred{c_R}$ and $\mgreen{c_G}$ respectively:
\[
= \mred{c_R} \cdot \mgreen{c_G} \cdot \int_a \int_b (\mgreen{\vec{g}} \cdot \mred{\vec{r}}) \, \mbgreen{\rho_G}(\mgreen{\vec{g}}) \, \mbred{\rho_R}(\mred{\vec{r}}) \, d\mgreen{\vec{g}} \, d\mred{\vec{r}} = \mblue{\Lambda} \cdot \overline{\mathcal{H}}(a, b)
\]
\end{proof}

The factor $\mblue{\Lambda} = \mgreen{c_G} \cdot \mred{c_R}$ arises from integrating out the coordinates that do not appear in the kernel.

\begin{figure}[htbp]
\centering
\includegraphics[width=\textwidth]{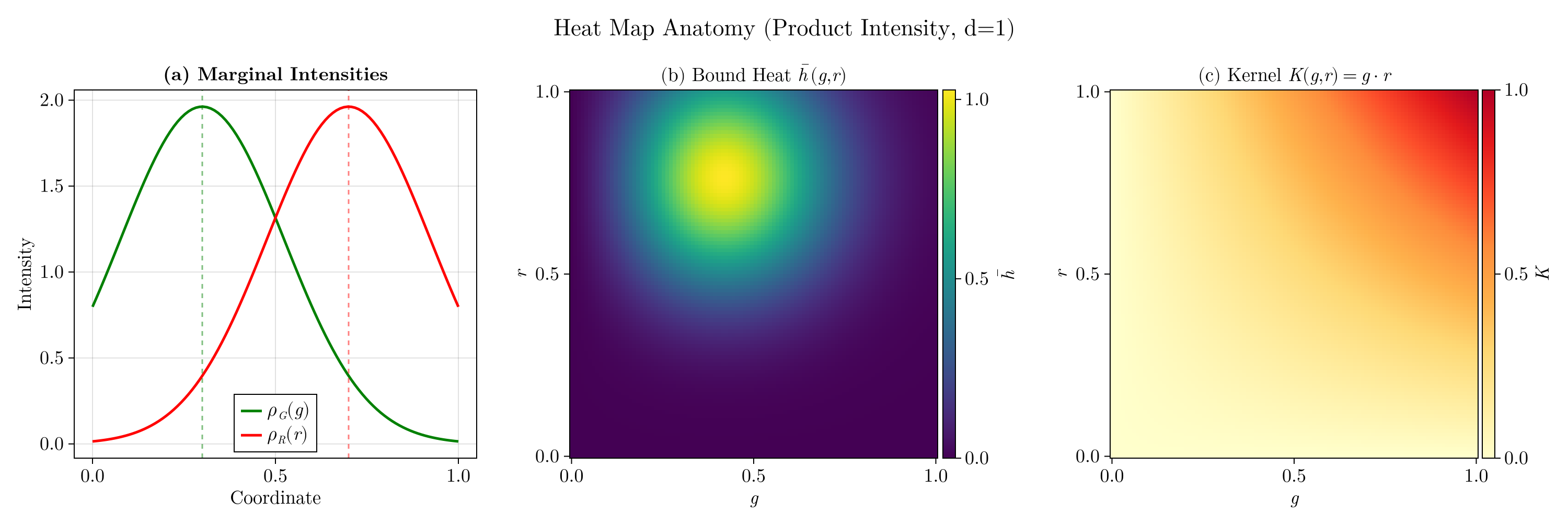}
\caption{\textbf{Heat map anatomy under product intensity ($d = 1$).} (a) Component intensities $\mbgreen{\rho_G}(g)$ and $\mbred{\rho_R}(r)$: the primitive objects defining the product intensity $\mbblue{\rho}(g,r) = \mbgreen{\rho_G}(g) \cdot \mbred{\rho_R}(r)$. Dashed lines mark the centers $g_0 = 0.3$ and $r_0 = 0.7$. (b) Bound heat $\overline{h}(g,r) = K(g,r) \cdot \mbgreen{\rho_G}(g) \cdot \mbred{\rho_R}(r)$, showing how the interaction landscape concentrates where both intensities and kernel are large. (c) Affinity kernel $K(g,r) = g \cdot r$. Under product intensity, raw heat $\mathcal{H}$ and bound heat $\overline{\mathcal{H}}$ have the same shape, differing only by a scalar factor: $\mathcal{H} = \mblue{\Lambda} \cdot \overline{\mathcal{H}}$.}
\label{fig:heat_anatomy}
\end{figure}

\subsubsection{Heat between other bite combinations}

For completeness, we record the raw heat between all combinations of green and red bites. Let $\mgreen{c_G}(a) = \int_a \mbgreen{\rho_G}$ denote the mass in green region $a$, and $\mbgreen{\mu_G}(a) = \int_a \mgreen{\vec{g}} \, \mbgreen{\rho_G}(\mgreen{\vec{g}}) \, d\mgreen{\vec{g}}$ the (unnormalized) mean green position in $a$. Define $\mred{c_R}(b)$ and $\mbred{\mu_R}(b)$ analogously.

\begin{center}
\begin{tabular}{lll}
\hline
\textbf{Source} & \textbf{Target} & \textbf{Raw heat $\mathcal{H}$} \\
\hline
$\mathcal{G}(a)$ & $\mathcal{R}(b)$ & $\mblue{\Lambda} \cdot \overline{\mathcal{H}}(a, b)$ \quad \textbf{(bound heat)} \\
$\mathcal{G}(a)$ & $\mathcal{G}(a')$ & $\mred{c_R} (\mbgreen{\mu_G}(a) \cdot \mbred{\mu_R}) \mgreen{c_G}(a')$ \\
$\mathcal{R}(b)$ & $\mathcal{R}(b')$ & $\mgreen{c_G} (\mbgreen{\mu_G} \cdot \mbred{\mu_R}(b')) \mred{c_R}(b)$ \\
$\mathcal{R}(b)$ & $\mathcal{G}(a)$ & $(\mbgreen{\mu_G} \cdot \mbred{\mu_R}) \mred{c_R}(b) \mgreen{c_G}(a)$ \\
\hline
\end{tabular}
\end{center}

Only the green-to-red combination ($\mathcal{G}(a) \to \mathcal{R}(b)$) captures local interaction structure through the bound heat. The other combinations involve global intensity-weighted means: they encode how much mass is in each region, weighted by average interaction propensity, but not the fine spatial structure of who-connects-to-whom.

This asymmetry reflects the kernel structure: $K(s,t) = \mgreen{\vec{g}_s} \cdot \mred{\vec{r}_t}$ uses the green coordinate of the source and the red coordinate of the target. Green bites constrain sources ``where they act from''; red bites constrain targets ``where they receive.''

\subsubsection{When bound heat fails: non-product intensity}

The bound heat representation relies on the product structure $\mbblue{\rho} = \mbgreen{\rho_G} \otimes \mbred{\rho_R}$. Without it, the ``inactive'' coordinates $(\mred{\vec{r}_s}, \mgreen{\vec{g}_t})$ do not factor out, and no $2d$-dimensional summary exists.

\begin{figure}[htbp]
\centering
\includegraphics[width=\textwidth]{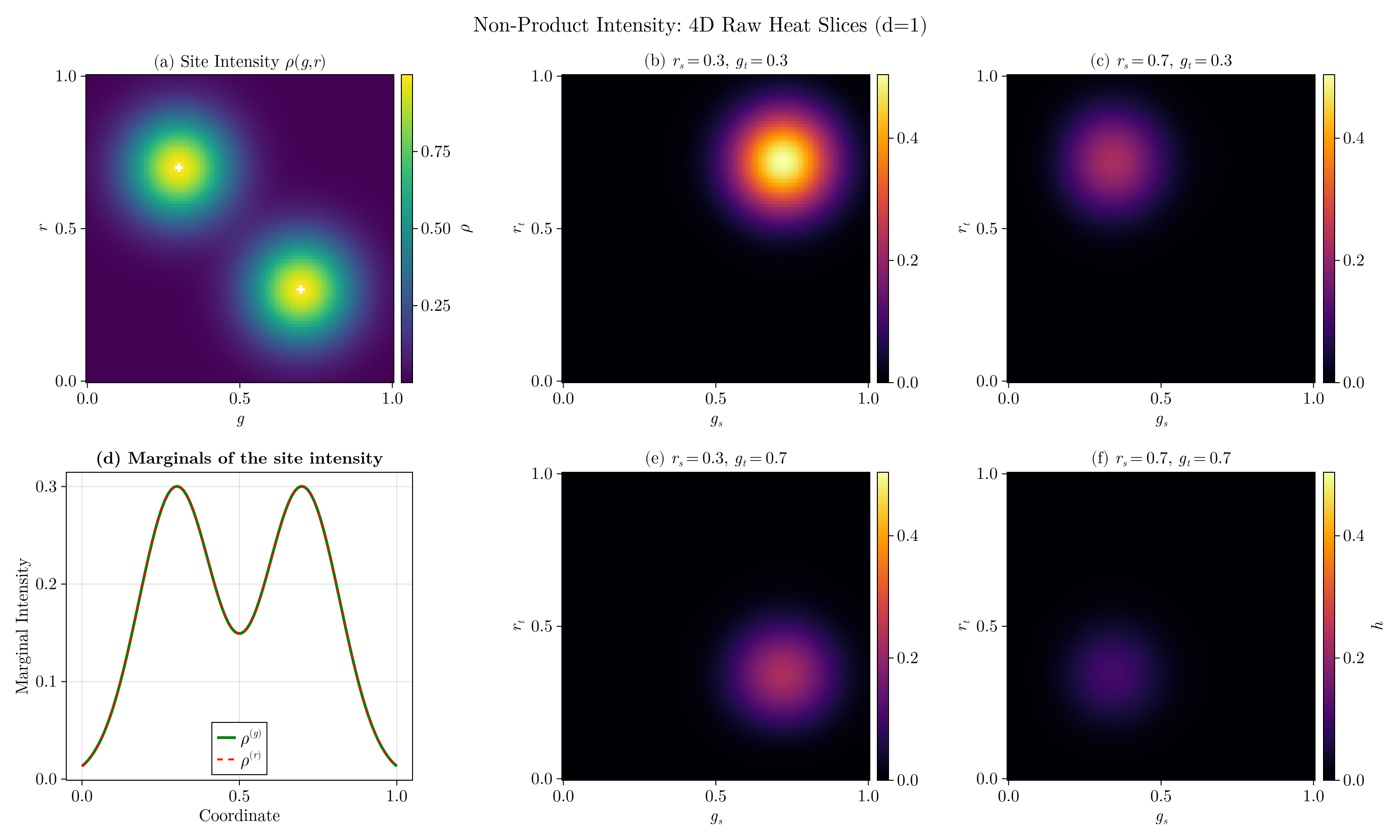}
\caption{\textbf{Non-product intensity and the failure of dimensional reduction.} (a) Intensity $\mbblue{\rho}(g,r)$ consisting of two Gaussian blobs at $(0.3, 0.7)$ and $(0.7, 0.3)$. (d) Marginals $\rho^{(g)}(g) = \int \mbblue{\rho} \, dr$ and $\rho^{(r)}(r) = \int \mbblue{\rho} \, dg$ of the intensity. The marginals are identical by symmetry of the blob placement, yet the joint does not factor: $\mbblue{\rho}(g,r) \neq \rho^{(g)}(g) \cdot \rho^{(r)}(r)$. The product of marginals would produce four blobs; the actual joint has only two. (b,c,e,f) Slices of the 4D raw heat $h(g_s, r_t \mid r_s, g_t)$ at four combinations of inactive coordinates. Different slices have different shapes because the inactive coordinates do not factor out. Under product intensity (Figure~\ref{fig:heat_anatomy}), all slices would be proportional, enabling the 2D bound heat summary.}
\label{fig:nonproduct}
\end{figure}

\subsection{Spectral structure of heat}

The heat map, viewed as an integral kernel, defines a compact operator whose spectral decomposition reveals the dominant modes of interaction. The relevant mathematical framework is the spectral theory of Hilbert--Schmidt operators~\cite[Ch.\ VI]{reed1980methods}\cite[Ch.\ 28]{lax2002functional}.

\textbf{The bound heat operator}: Under product intensity, define the \emph{bound heat operator} $\overline{T}: L^2(\mred{B^d_+}) \to L^2(\mgreen{B^d_+})$ by:
\[
(\overline{T} f)(\mgreen{\vec{g}}) = \int_{\mred{B^d_+}} \overline{h}(\mgreen{\vec{g}}, \mred{\vec{r}}) \, f(\mred{\vec{r}}) \, d\mred{\vec{r}}
\]

This maps functions on $\mred{\text{red space}}$ to functions on $\mgreen{\text{green space}}$, encoding how accepting-propensities translate to proposing-propensities through the interaction structure. The operator $\overline{T}$ is Hilbert--Schmidt (hence compact) whenever $\overline{h} \in L^2(\mgreen{B^d_+} \times \mred{B^d_+})$, which holds for any bounded intensity on the bounded domain $B^d_+$.

\subsubsection{Finite rank from the dot product kernel}

The affinity kernel $K(s,t) = \mgreen{\vec{g}_s} \cdot \mred{\vec{r}_t} = \sum_{k=1}^d (\mgreen{g}_s)_k (\mred{r}_t)_k$ is a sum of $d$ rank-1 terms. Consequently, $\overline{T}$ has \emph{rank at most $d$}. This finite-rank property, inherited from the dot product structure, distinguishes IDPG from models based on infinite-rank kernels such as Gaussian RBF.

Explicitly, write:
\[
\overline{h}(\mgreen{\vec{g}}, \mred{\vec{r}}) = \sum_{k=1}^d \underbrace{\mgreen{g}_k \, \mbgreen{\rho_G}(\mgreen{\vec{g}})}_{\alpha_k(\mgreen{\vec{g}})} \cdot \underbrace{\mred{r}_k \, \mbred{\rho_R}(\mred{\vec{r}})}_{\beta_k(\mred{\vec{r}})}
\]

The spectral structure of $\overline{T}$ is determined by the \emph{Gram matrices}. Let
\[
\langle f, g \rangle_{L^2} = \int_{B^d_+} f(x) \, g(x) \, dx
\]
denote the standard $L^2$ inner product. Then:
\[
A_{jk} = \langle \alpha_j, \alpha_k \rangle_{L^2} = \int_{\mgreen{B^d_+}} \mgreen{g}_j \, \mgreen{g}_k \, \mbgreen{\rho_G}(\mgreen{\vec{g}})^2 \, d\mgreen{\vec{g}}
\]
\[
B_{jk} = \langle \beta_j, \beta_k \rangle_{L^2} = \int_{\mred{B^d_+}} \mred{r}_j \, \mred{r}_k \, \mbred{\rho_R}(\mred{\vec{r}})^2 \, d\mred{\vec{r}}
\]

These $d \times d$ matrices encode the ``shape'' of the intensity in each coordinate direction. See Appendix~\ref{appendix:spectral} for the full derivation.

\subsubsection{Singular value decomposition}

Since the kernel is non-symmetric ($\overline{h}(\mgreen{\vec{g}}, \mred{\vec{r}}) \neq \overline{h}(\mred{\vec{r}}, \mgreen{\vec{g}})$ in general), we use \emph{singular value decomposition} rather than eigendecomposition. By the Schmidt decomposition theorem~\cite{schmidt1907theorie}\cite[Ch.\ 28]{lax2002functional}, there exist orthonormal left singular functions $\{u_n\} \subset L^2(\mgreen{B^d_+})$, right singular functions $\{v_n\} \subset L^2(\mred{B^d_+})$, and singular values $\sigma_1 \geq \sigma_2 \geq \cdots \geq \sigma_d \geq 0$ such that:
\[
\overline{h}(\mgreen{\vec{g}}, \mred{\vec{r}}) = \sum_{n=1}^d \sigma_n \, u_n(\mgreen{\vec{g}}) \, v_n(\mred{\vec{r}})
\]

The singular values satisfy $\overline{T} v_n = \sigma_n u_n$ and $\overline{T}^* u_n = \sigma_n v_n$. For symmetric positive-definite kernels, this reduces to the eigendecomposition given by Mercer's theorem~\cite{mercer1909functions}\cite[Ch.\ 28]{lax2002functional}; the non-symmetric case requires the more general singular value decomposition (SVD) framework.

\subsubsection{Interpretation of the spectrum}

The singular values encode interaction structure:

\begin{itemize}
\item $\sigma_1$: The dominant mode: the direction in latent space along which most interaction intensity concentrates
\item $\sum_k \sigma_k^2 = \|\overline{h}\|_{HS}^2$: Total interaction intensity (Hilbert--Schmidt norm squared)
\item Decay rate of $\sigma_k$: How ``low-dimensional'' the effective interaction is. Rapid decay means interactions are well-approximated by fewer than $d$ modes.
\end{itemize}

For concentrated intensities (approaching Dirac masses), the spectrum reflects the positions of the point masses. For diffuse intensities, the spectrum reflects the geometric overlap between $\mbgreen{\rho_G}$ and $\mbred{\rho_R}$ in the latent space.

Perturbation theory~\cite{kato1995perturbation} guarantees stability: if the intensity $\mbblue{\rho}$ changes smoothly, the singular values change continuously. Specifically, Weyl's inequality~\cite{weyl1912asymptotische}\cite[Ch.\ 3]{horn2013matrix} bounds $|\sigma_k(\overline{T}_1) - \sigma_k(\overline{T}_2)| \leq \|\overline{T}_1 - \overline{T}_2\|_{\text{op}}$.

\begin{figure}[htbp]
\centering
\includegraphics[width=0.9\textwidth]{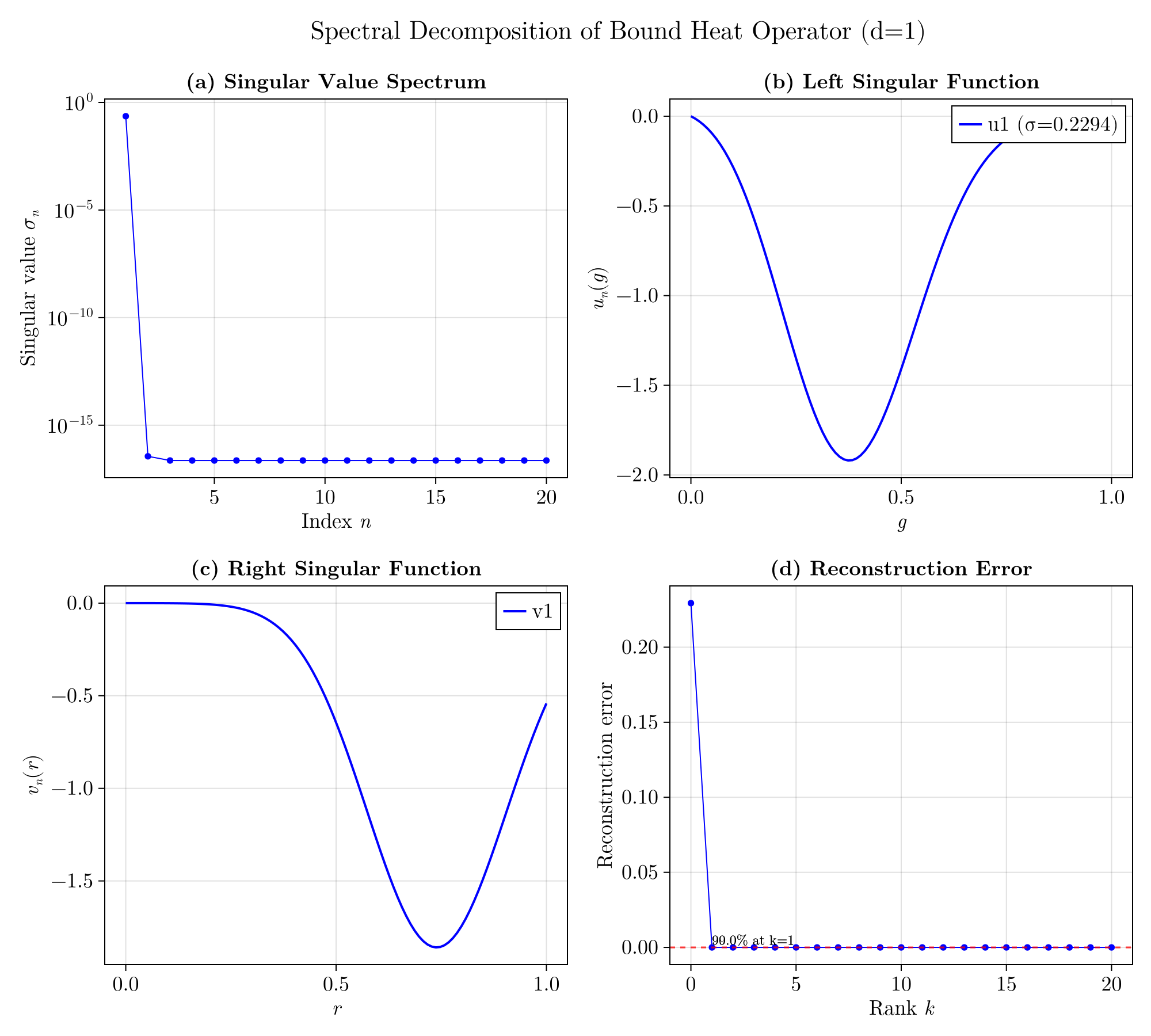}
\caption{\textbf{Spectral decomposition of the bound heat operator ($d = 1$).} (a) Singular value spectrum showing exactly one non-zero singular value $\sigma_1 \approx 0.23$, confirming the theoretical rank bound $\operatorname{rank}(\overline{T}) \leq d$. (b,c) The unique left and right singular functions $u_1(g)$ and $v_1(r)$, whose shapes reflect the intensity-weighted coordinate functions. (d) Reconstruction error as a function of rank $k$; the first mode captures 99\% of the Hilbert--Schmidt norm.}
\label{fig:spectral_1d}
\end{figure}

\subsection{The desire operator}

The bound heat operator $\overline{T}$ incorporates the intensity directly into its kernel: $\overline{h}(g, r) = (g \cdot r) \rho_G(g) \rho_R(r)$, so that both the interaction structure and the population density are entangled. An alternative formulation uses the normalized intensities $\tilde{\rho}_G = \rho_G / c_G$ and $\tilde{\rho}_R = \rho_R / c_R$ as probability distributions describing where individuals are located, and keeps the affinity kernel $K(s,t) = \mgreen{\vec{g}_s} \cdot \mred{\vec{r}_t}$ separate.

\begin{definition}
The \emph{desire operator} $\tilde{D}: L^2(B^d_+, \tilde{\rho}_R) \to L^2(B^d_+, \tilde{\rho}_G)$ is defined by:
\[
(\tilde{D} f)(g) = \int_{B^d_+} (g \cdot r) \, f(r) \, \tilde{\rho}_R(r) \, dr
\]
where $\tilde{\rho}_G = \rho_G / c_G$ and $\tilde{\rho}_R = \rho_R / c_R$ are the normalized marginal intensities (probability densities).
\end{definition}

\textbf{Interpretation.} Given a distribution of receiver profiles $f$ weighted by the population $\tilde{\rho}_R$, the desire operator computes the resulting ``giving desire'' landscape over $G$-space. The value $(\tilde{D} f)(g)$ measures the expected affinity of a node at $g$ towards the population of receivers $f$.
The adjoint $\tilde{D}^*$ computes the reverse: conditional on a certain distribution of givers, where would receivers want to be?

\subsubsection{Spectral structure and Gram matrices}

Like the bound heat operator, $\tilde{D}$ has finite rank (at most $d$). However, its spectrum is governed by the geometry of the weighted $L^2$ spaces induced by the population densities.

Define the weighted inner products for the proposal and acceptance spaces:
\[
\langle u, v \rangle_{\tilde{\rho}_G} = \int_{\mgreen{B^d_+}} u(\mgreen{\vec{g}}) \, v(\mgreen{\vec{g}}) \, \tilde{\rho}_G(\mgreen{\vec{g}}) \, d\mgreen{\vec{g}}
\]
\[
\langle u, v \rangle_{\tilde{\rho}_R} = \int_{\mred{B^d_+}} u(\mred{\vec{r}}) \, v(\mred{\vec{r}}) \, \tilde{\rho}_R(\mred{\vec{r}}) \, d\mred{\vec{r}}
\]

The \emph{desire Gram matrices} are the Gramians of the coordinate functions with respect to these weighted inner products:
\[
(\Sigma_G)_{jk} = \langle \mgreen{g}_j, \mgreen{g}_k \rangle_{\tilde{\rho}_G}
\]
\[
(\Sigma_R)_{jk} = \langle \mred{r}_j, \mred{r}_k \rangle_{\tilde{\rho}_R}
\]

Explicitly, $(\Sigma_G)_{jk} = \int \mgreen{g}_j \, \mgreen{g}_k \, \tilde{\rho}_G \, d\mgreen{\vec{g}} = \mathbb{E}[\mgreen{g}_j \, \mgreen{g}_k]$ and likewise for $\Sigma_R$. These are the \emph{second moment matrices} of the latent positions under the normalized intensities: unlike centred covariance matrices, they do not subtract the mean, so they capture both the spread and the location of the population in the latent space. The singular values of the desire operator are determined by the alignment of these two geometries:
\[
\sigma_k(\tilde{D}) = \sqrt{\lambda_k(\Sigma_G \Sigma_R)}
\]
where $\lambda_k(M)$ denotes the $k$-th largest eigenvalue of the matrix $M$. Note that while the product $\Sigma_G \Sigma_R$ is not symmetric, it is similar to a symmetric positive semi-definite matrix, ensuring real non-negative eigenvalues (see Appendix~\ref{appendix:spectral-desire}).

\subsubsection{Sample estimation}

The desire operator provides a direct link between the measure-centric operators and the observed, discrete, graphs. Under mild conditions, we can see that the spectral decomposition of the observed graphs (that determine important structural property of these discrete objects) converge to the spectral decomposition of the desire operators.

\begin{theorem}[Spectral Consistency of the Adjacency Matrix]\label{thm:spectral-consistency-adjacency}
Let $A_N$ be the adjacency matrix of a graph sampled from the perennial IDPG model \emph{conditional on having $N$ nodes} (equivalently: sample $N$ latent positions i.i.d.\ from $\mu = \mbblue{\rho} / \mblue{\Lambda}$, then sample edges conditionally independently). Let
\[
P_N = \mathbb{E}[A_N \mid (\mgreen{\vec{g}_i}, \mred{\vec{r}_i})_{i=1}^N]
\]
with entries $(P_N)_{ij} = \mgreen{\vec{g}_i} \cdot \mred{\vec{r}_j}$. As $N \to \infty$:
\[
\frac{\sigma_k(A_N)}{N} \xrightarrow{p} \sigma_k(\tilde{D})
\]

where $\xrightarrow{p}$ denotes convergence in probability. That is, for every $\epsilon > 0$,
\[
P\left(\left|\frac{\sigma_k(A_N)}{N} - \sigma_k(\tilde{D})\right| > \epsilon\right) \to 0
\quad \text{as } N \to \infty.
\]
Equivalently, the scaled singular values of the observed graph converge to the singular values of the desire operator.
\end{theorem}

\begin{proof}
The proof relies on decomposing the error into a ``sampling error'' (approximating integrals with $P$) and a ``Bernoulli noise'' error (realizing edges in $A$).

By the triangle inequality:
\[
\left|\frac{\sigma_k(A_N)}{N} - \sigma_k(\tilde{D})\right| \leq \underbrace{\left|\frac{\sigma_k(A_N)}{N} - \frac{\sigma_k(P_N)}{N}\right|}_{\text{Bernoulli Noise}} + \underbrace{\left|\frac{\sigma_k(P_N)}{N} - \sigma_k(\tilde{D})\right|}_{\text{Discretization Error}}
\]

\begin{enumerate}
\item \textbf{Discretization Error:} Conditional on $N$, we established that $\sigma_k(P_N) / N$ converges to $\sigma_k(\tilde{D})$. This follows from the Law of Large Numbers applied to the Gram matrices; see Appendix~\ref{appendix:spectral-desire}.

\item \textbf{Bernoulli Noise:} By Weyl's inequality, $|\sigma_k(A_N) - \sigma_k(P_N)| \leq \|A_N - P_N\|_{\text{op}}$. The matrix $E_N = A_N - P_N$ consists of independent centered bounded random variables. Standard concentration results for the spectral norm of random matrices~\cite{athreya2018statistical} guarantee that $\|E_N\|_{\text{op}} \leq C \sqrt{N}$ with high probability if the maximum expected degree in the graph grows fast enough~\cite[Sec.\ 3.1]{athreya2018statistical}.
\end{enumerate}

Consequently, the noise term behaves as:
\[
\frac{\|A_N - P_N\|_{\text{op}}}{N} \leq \frac{C \sqrt{N}}{N} = \frac{C}{\sqrt{N}} \to 0
\]

Since both error terms vanish, the result follows.
\end{proof}

This is a conditional-on-size asymptotic statement. In the original PPP formulation with random size $N \sim \text{Poisson}(\mblue{\Lambda})$, the same limit follows along $\mblue{\Lambda} \to \infty$ because $N / \mblue{\Lambda} \to 1$ in probability.

In the scenario of the above theorem we observe only one graph sampled from a certain IDPG model, although a very large one.

We have a similar, albeit weaker, result in the scenario in which we observe many \emph{small} independent graphs, sampled from the same IDPG model.

Even if the graphs have different vertex sets and sizes, the \emph{average} of their scaled singular values are linked to the operator spectrum. In practice we use a non-empty-graph sampling protocol (empty realizations carry no spectral information). Yet, the singular values of a finite graph are biased estimators of the operator spectrum. Averaging over $m$ graphs reduces the \textbf{variance} of the estimate, but not this inherent \textbf{bias}.
\[
|\overline{\sigma}_k - \sigma_k(\tilde{D})| = \mathcal{O}(1/\sqrt{\mblue{\Lambda}}) + \mathcal{O}_p(1/\sqrt{m \mblue{\Lambda}})
\]
Thus, accurate recovery requires the total intensity $\mblue{\Lambda}$, and hence the expected number of nodes, to be sufficiently large. If $\Lambda$ is large, the error is dominated by fluctuations, and observing $m$ graphs reduces the error variance by a factor of $1/m$, effectively scaling the precision with the total number of observed nodes. See Proposition~\ref{prop:multi-graph-consistency} in the Appendix for the rigorous derivation.

\subsubsection{Numerical verification}

Figure~\ref{fig:desire-spectral} verifies these spectral consistency results through Monte Carlo simulation. We generate IDPGs in $d = 4$ dimensions using a two-component mixture intensity on $B^4_+ \times B^4_+$, with theoretical singular values $\sigma_1(\tilde{D}) = 0.509$, $\sigma_2(\tilde{D}) = 0.107$, $\sigma_3(\tilde{D}) \approx 0.032$, and $\sigma_4(\tilde{D}) \approx 0.030$.

Panel (a) demonstrates single-graph convergence: the scaled singular values $\sigma_k(A)/N$ approach their theoretical limits as the total intensity $\Lambda$ (and hence expected node count) increases. The leading singular values $\sigma_1$ and $\sigma_2$ converge rapidly, while $\sigma_3$ and $\sigma_4$, being close to the noise floor $1/\sqrt{\Lambda}$, require larger graphs to distinguish from the fifth singular value, which should be theoretically zero.

Panel (b) confirms that the finite-$\Lambda$ bias scales as $\mathcal{O}(1/\sqrt{\Lambda})$, consistent with the CLT-based argument in the proof. Panel (c) verifies the multi-graph consistency result (Proposition~\ref{prop:multi-graph-consistency} in the Appendix): at fixed $\Lambda = 300$, averaging $m$ independent graphs reduces the standard deviation of $\overline{\sigma}_k$ as $\mathcal{O}(1/\sqrt{m})$, while the bias (set by $\Lambda$) remains unchanged.

A key insight emerges around $\Lambda \approx 10^3$: this is where the noise floor $1/\sqrt{\Lambda} \approx 0.03$ drops below the smallest true singular values, allowing the rank-$d$ structure of the desire operator to become empirically distinguishable. Before this threshold, the signal from $\sigma_3$ and $\sigma_4$ is confounded with noise; after it, the full spectral structure emerges.

\begin{figure}[htbp]
\centering
\includegraphics[width=\textwidth]{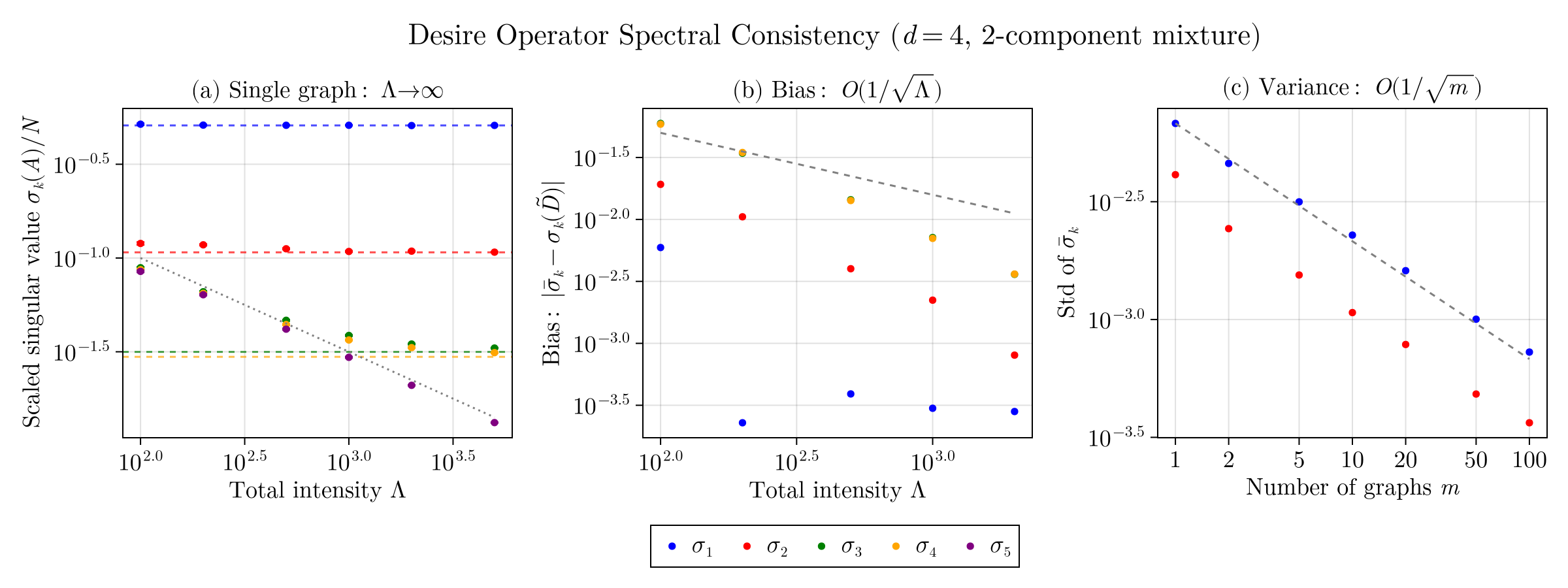}
\caption{\textbf{Numerical verification of spectral consistency.} A two-component mixture intensity in $d = 4$ dimensions with theoretical singular values $\sigma_k(\tilde{D}) = (0.509, 0.107, 0.032, 0.030)$. (a) Single-graph convergence: scaled singular values $\sigma_k(A)/N$ approach theoretical limits as $\Lambda \to \infty$. The noise floor $1/\sqrt{\Lambda}$ (dotted) bounds the estimation precision for small singular values. (b) Finite-intensity bias scales as $\mathcal{O}(1/\sqrt{\Lambda})$, with $\sigma_1$ showing negligible bias. (c) Multi-graph averaging at $\Lambda = 300$: standard deviation decreases as $1/\sqrt{m}$ while bias remains fixed.}
\label{fig:desire-spectral}
\end{figure}

\subsubsection{Relationship to bound heat}

The two operators encode complementary views of the system:

\begin{itemize}
\item \textbf{Bound Heat $\overline{T}$:} Represents the \textbf{total interaction mass}. Its Gram matrices ($A, B$) involve integrals of $\rho^2$. It answers ``Where are the edges located in space?''.
\item \textbf{Desire $\tilde{D}$:} Represents the \textbf{per-capita interaction structure}. Its Gram matrices ($\Sigma_G, \Sigma_R$) involve integrals of $\rho$ (first moments). It answers ``What is the connectivity rule for an average node?''.
\end{itemize}

\subsection{The measure-theoretic Laplacian}

The classic Laplacian is the matrix $L = D - A$~\cite[Ch.\ 1]{chung1997spectral} where $D$ is a diagonal matrix which entries are the nodes degree, and $A$ is the adjacency matrix. The classic Laplacian is central to many results in graph and complex network theory. Here we introduce a measure-theoretic analogous.

Define the \emph{out-degree function}:
\[
d_{\text{out}}(s) = \int_{\mblue{\Omega}} h(s, t) \, dt
\]

Under product intensity, this becomes:
\[
d_{\text{out}}(s) = \mbgreen{\rho_G}(\mgreen{\vec{g}_s}) \, \mbred{\rho_R}(\mred{\vec{r}_s}) \cdot \mred{c_R} \cdot (\mgreen{\vec{g}_s} \cdot \mbred{\tilde{\mu}_R})
\]

The \emph{continuous Laplacian} is the operator $\mathcal{L}: L^2(\mblue{\Omega}) \to L^2(\mblue{\Omega})$:
\[
(\mathcal{L} f)(s) = d_{\text{out}}(s) \, f(s) - \int_{\mblue{\Omega}} h(s, t) \, f(t) \, dt
\]

Equivalently, define the \emph{raw heat operator} $T: L^2(\mblue{\Omega}) \to L^2(\mblue{\Omega})$ by
\[
(T f)(s) = \int h(s,t) \, f(t) \, dt,
\]
and the degree operator $D$ by $(D f)(s) = d_{\text{out}}(s) \, f(s)$. Then $\mathcal{L} = D - T$.

In symmetric/reversible variants (e.g., after symmetrization), the spectral gap controls spreading rates and admits Cheeger-type interpretations~\cite{cheeger1970lower}\cite[Ch.\ 2]{chung1997spectral}. For the general directed non-self-adjoint operator, the spectral interpretation is subtler and is left to future work. In related geometric-random-graph settings, operator-to-graph Laplacian convergence results are known~\cite{hein2007graph,trillos2020error}; establishing the exact analogue for the present directed IDPG construction is left open.\footnote{A full development of the Laplacian's spectral properties (niceties such as Cheeger inequalities, clustering from eigenvectors, connection to random walks) is deferred to future work. The key point here is that the heat map provides the natural kernel for defining such operators.}


\section{Recovery of RDPG}
\label{sec:rdpg-recovery}

The reader might be tempted to read a classic RDPG as a limiting case for an IDPG whose intensities become supported pointwise. The heat map framework allows us to make this analogy more robust, but the two models remain distinct.

\textbf{The Dirac limit.} Consider a sequence of increasingly concentrated intensities converging to point masses. Let $\mbblue{\rho}^{(\epsilon)} = \sum_{i=1}^{N} \rho_i^{(\epsilon)}$ where each $\rho_i^{(\epsilon)}$ is a truncated Gaussian centered at $s_i = (\mgreen{\vec{g}_i}, \mred{\vec{r}_i})$ with variance $\epsilon^2$, normalized so $\int \rho_i^{(\epsilon)} = 1$. Then $\mblue{\Lambda}^{(\epsilon)} = N$ for all $\epsilon$.

As $\epsilon \to 0$, the intensity converges weakly to a sum of Dirac measures:
\[
\mbblue{\rho}^{(\epsilon)} \to \mbblue{\rho} = \sum_{i=1}^{N} \delta_{s_i}
\]

\textbf{Measure-theoretic interpretation.} When $\mbblue{\rho}$ is a sum of Dirac measures, the raw heat becomes a discrete measure on $\mblue{\Omega} \times \mblue{\Omega}$. For singletons $\{s_i\}$ and $\{s_j\}$, the Dirac measure satisfies $\delta_{s_i}(\{s_i\}) = 1$, so:
\[
\mathcal{H}(\{s_i\}, \{s_j\}) = (\mgreen{\vec{g}_i} \cdot \mred{\vec{r}_j}) \cdot \delta_{s_i}(\{s_i\}) \cdot \delta_{s_j}(\{s_j\}) = \mgreen{\vec{g}_i} \cdot \mred{\vec{r}_j} = P_{ij}
\]

The raw heat at point masses recovers the entries of the RDPG probability matrix $\mathbf{P}$.

\textbf{Relationship to RDPG.} The correspondence illuminates the structure but does not establish containment in either direction:

\begin{itemize}
\item In RDPG, $P_{ij}$ is a probability (between 0 and 1) and nodes are fixed.
\item In IDPG, $\mathcal{H}(\{s_i\}, \{s_j\}) = P_{ij}$ in the unit-weight Dirac limit recovers an RDPG-style probability matrix, but in general IDPG there is no finite matrix $P$: interaction mass is encoded by measures (raw heat / bound heat).
\item The Dirac limit of IDPG produces point masses with unit weight; weighted Diracs $\sum_i w_i \delta_{s_i}$ with $w_i \neq 1$ give $\mathcal{H}(\{s_i\}, \{s_j\}) = w_i w_j P_{ij}$, which has no direct RDPG analog.
\end{itemize}

In an RDPG the interaction structure is encoded in a kernel evaluated at latent positions. The heat map framework generalizes the intuition behind RDPG to a continuous, measure-theoretic setting. IDPG points toward RDPG in the concentrated limit, but the two models remain fundamentally distinct.
\section{An ecological motivating example}\label{sec:foodwebs}

A \emph{food web} is an epitome example of an ecological network in which nodes represent species and edges represent consumption, predation, or, in brief, \emph{who eats whom}. Usually, the set of species in a food web represents a certain ecological community or an ecosystem. And the edges are often determined by painstaking laborious field work by ecologists.

Yet, despite their fruitful application, it is quite clear that it is not a species eating another species: it's a certain individual of a species, say a cow, eating one or more individuals of another species. And, from a Darwinian point of a view, a species is a population (we might say a collection satisfying certain genealogical conditions) of individuals. Crucially, the individuals are not all identical, as various mechanisms bring some variance in the genetic (and phenotypical, and thus ecological) identity of individuals.

Hence, it is rather common in evolutionary sciences to describe the variability of individuals in a species with a certain probability distribution $\mu$ in some space where the metric represents genetic similarity (and often the distributions tend to be multivariate Gaussians: most individuals will have a genome very similar to each other with few mutations, some will vary more, a few will be rather atypical, \ldots).

The environment, its biotic and abiotic components, and the phenotypic expression of an individual concur to determine the individual ecological role in an ecosystem (that is, the individual propensities to establish different, ecologically relevant, connections with other individuals from the same or other species). This mapping of a genome to an ecological role corresponds to a mapping from the distribution $\mu$, which takes values in the genetic space, to a distribution $\mathcal{E}(\mu)$ taking values in a theoretical space of ecological roles. In the case of a food web, where the ecological interactions are trophic relationships (who is a food resource for whom, who is a consumer of whom) we can ideally project $\mathcal{E}(\mu)$ into two subspaces $\mathcal{E}_g(\mu)$ and $\mathcal{E}_r(\mu)$ of \emph{ecological role as a resource} and \emph{ecological role as a consumer} (see \cite{dalla2016exploring} for such an analysis, although at the level of species).

An interaction between two individuals of different species will hence depend on the probability of those two individuals being ``there'', the propensity of one individual to eat the other, and the propensity of the latter to be eaten.

We can thus represent an ecological network, and in particular a food web, as an IDPG under the \emph{ephemeral} interpretation. The intensity $\mbblue{\rho}(\mgreen{\vec{g}}, \mred{\vec{r}})$ captures the density of potential individuals characterized by their resource-role $\mgreen{\vec{g}}$ and consumer-role $\mred{\vec{r}}$. Edge opportunities arise when pairs of ephemeral (in the time scale of evolutionary processes) individuals encounter each other; the probability that a consumer at $\mgreen{\vec{g}_s}$ successfully consumes a resource at $\mred{\vec{r}_t}$ is given by the dot product $\mgreen{\vec{g}_s} \cdot \mred{\vec{r}_t}$.

In this sense, a classic representation of a food web should be seen as a statistical summary of an ephemeral IDPG in which individuals are aggregated through an appropriate clustering procedure into species nodes. The necessity of considering food webs as probabilistic in nature has been recognized by \cite{poisot2016structure}.

When aggregating from the continuous ephemeral representation to discrete species-level graphs, different amounts of information may be retained per sampled edge:

\begin{itemize}
  \item \textbf{Minimal}: only $(\mgreen{\vec{g}_s}, \mred{\vec{r}_t})$, the coordinates used in the affinity kernel
  \item \textbf{Full}: each edge carries all four coordinates $(\mgreen{\vec{g}_s}, \mred{\vec{r}_s}, \mgreen{\vec{g}_t}, \mred{\vec{r}_t})$
\end{itemize}

Full position information enables clustering by the complete $(\mgreen{\vec{g}}, \mred{\vec{r}})$ signature, essential for ecological applications where species are identified by their combined resource-consumer profile.

Interestingly, ecological and evolutionary processes do really happen at the level of the underlying intensity functions, by the gradual movement of a species across the space, as fully recognized by various disciplines as population genetics or adaptive dynamics. The timescale of these evolutionary changes relative to the persistence of individual organisms determines whether the perennial or ephemeral view is more appropriate for a given analysis.

\subsection{Mixture of product intensities for species}\label{sec:mixture-products}

For food webs with multiple distinct species, a natural representation arises when the intensity is a \emph{sum} of species-specific products:
\[
  \mbblue{\rho}(\mgreen{\vec{g}}, \mred{\vec{r}}) = \sum_{m=1}^{M} \rho_{G,m}(\mgreen{\vec{g}}) \cdot \rho_{R,m}(\mred{\vec{r}})
\]

Each species $m$ has its own marginal intensities $\rho_{G,m}$ and $\rho_{R,m}$ defining its niche in the giving and receiving spaces.

\textbf{Key distinction from product intensity.} Compare the two forms:

\begin{table}[htbp]
\centering
\begin{tabular}{|p{0.45\textwidth}|p{0.45\textwidth}|}
\hline
\textbf{Product of Mixtures} & \textbf{Mixture of Products} \\
\hline
$[\sum_m \rho_{G,m}] \times [\sum_m \rho_{R,m}]$ & $\sum_m [\rho_{G,m} \times \rho_{R,m}]$ \\
\hline
$\mgreen{\vec{g}}$ and $\mred{\vec{r}}$ independent & $\mgreen{\vec{g}}$ and $\mred{\vec{r}}$ coupled by species \\
\hline
Cross-species mixing in $(\mgreen{\vec{g}}, \mred{\vec{r}})$ & No cross-species mixing \\
\hline
\end{tabular}
\end{table}

In the mixture of products, a sampled individual's $\mgreen{\vec{g}}$ and $\mred{\vec{r}}$ come from the \emph{same} species: they are coupled by species identity. This structure is essential for ecological modeling where species have characteristic profiles in both resource and consumer roles.

\textbf{Species identity as model state.} In the mixture model, species identity $m$ is \emph{part of the sampled state}, not inferrable from position alone. If species have overlapping supports in $\mblue{\Omega}$, an individual at position $(\mgreen{\vec{g}}, \mred{\vec{r}})$ could belong to multiple species. The sampling procedure makes species identity explicit:

\textbf{Sampling from mixture of products:}
\begin{enumerate}
  \item Compute species contributions $\gamma_m = c_{G,m} \cdot c_{R,m}$ where $c_{G,m} = \int \rho_{G,m}$ and $c_{R,m} = \int \rho_{R,m}$
  \item Total intensity $\mblue{\Lambda} = \sum_m \gamma_m$
  \item Sample $N \sim \mathrm{Poisson}(\mblue{\Lambda})$
  \item For each individual:
  \begin{enumerate}
    \item Select species $m$ with probability $\gamma_m / \mblue{\Lambda}$
    \item Sample $\mgreen{\vec{g}}$ from $\rho_{G,m} / c_{G,m}$ (normalized)
    \item Sample $\mred{\vec{r}}$ from $\rho_{R,m} / c_{R,m}$ (normalized)
    \item The individual carries species label $m$ as part of its state
  \end{enumerate}
\end{enumerate}

The species label enables clustering and analysis at the species level, bridging the continuous latent space with discrete taxonomic structure.

\subsection{Source-target asymmetry in food webs}

\footnote{This section describes a variant of the ephemeral model where source and target individuals are drawn from potentially different distributions. This ``asymmetric ephemeral'' model differs from the symmetric ephemeral rule defined earlier.}

In the basic ephemeral model, both individuals in a pair are drawn from the same intensity $\mbblue{\rho}$. For food webs, we may want \emph{asymmetric} source and target distributions: producers should more often be targets (eaten) than sources (eating), while apex predators should be the reverse.

\subsubsection{General edge intensity}

The asymmetric ephemeral rule requires only that $\mbblue{\rho_{\mathcal{E}}}: \mathcal{E} \to \mathbb{R}_+$ with $\int \mbblue{\rho_{\mathcal{E}}} = \mblue{\Lambda} / 2$. The product form $\mbblue{\rho_{\mathcal{E}}}(s,t) \propto \mbblue{\rho}(s) \mbblue{\rho}(t)$ is a special case. More generally:
\[
  \mbblue{\rho_{\mathcal{E}}}(s, t) = f(s, t)
\]
for any non-negative function $f$ with the correct total mass. This includes:

\begin{itemize}
  \item \textbf{Factored asymmetric}: $\mbblue{\rho_{\mathcal{E}}}(s,t) \propto \rho_S(s) \cdot \rho_T(t)$ with different source and target intensities
  \item \textbf{Fully coupled}: $\mbblue{\rho_{\mathcal{E}}}(s,t)$ that cannot be written as a product
\end{itemize}

\subsubsection{Species-weighted asymmetry}

For mixture-of-products intensities (Section~\ref{sec:mixture-products}), a natural asymmetric form uses species-specific source and target weights. Let $\rho_m(\mgreen{\vec{g}}, \mred{\vec{r}}) = \rho_{G,m}(\mgreen{\vec{g}}) \cdot \rho_{R,m}(\mred{\vec{r}})$ denote the intensity for species $m$.

Define source and target intensities:
\[
  \rho_S(s) = \sum_m w_{S,m} \cdot \rho_m(s), \quad \rho_T(t) = \sum_m w_{T,m} \cdot \rho_m(t)
\]

where $w_{S,m}$ is the propensity of species $m$ to appear as a source (consumer) and $w_{T,m}$ its propensity to appear as a target (resource). For producers: $w_{S,m} \approx 0$, $w_{T,m}$ large. For apex predators: the reverse.

The edge intensity is then:
\[
  \mbblue{\rho_{\mathcal{E}}}(s, t) = \frac{\rho_S(s) \cdot \rho_T(t)}{2 M_S M_T / \mblue{\Lambda}}
\]
where $M_S = \int \rho_S$ and $M_T = \int \rho_T$. This ensures $\int \int \mbblue{\rho_{\mathcal{E}}} = \mblue{\Lambda} / 2$.

\subsubsection{Coordinate-dependent weights and kernel absorption}

A different mechanism uses weights that depend on position rather than species. If the weight depends only on $\mgreen{\vec{g}}$ for sources and only on $\mred{\vec{r}}$ for targets, the weights can be \emph{absorbed into the affinity kernel}:
\[
  w_S(\mgreen{\vec{g}_s}) \cdot w_T(\mred{\vec{r}_t}) \cdot (\mgreen{\vec{g}_s} \cdot \mred{\vec{r}_t}) = \tilde{\mgreen{\vec{g}}}_s \cdot \tilde{\mred{\vec{r}}}_t
\]
where $\tilde{\mgreen{\vec{g}}} = w_S(\mgreen{\vec{g}}) \cdot \mgreen{\vec{g}}$ is a rescaled green coordinate.
This absorption is valid only if the transformed coordinates remain admissible for the probability model (e.g., inside $B^d_+$ so that all induced dot products stay in $[0,1]$); otherwise an additional renormalization or projection step is required.

However, if weights depend on the \emph{full position} $(\mgreen{\vec{g}}, \mred{\vec{r}})$, e.g., when trophic roles are determined by both coordinates jointly, they cannot be absorbed into the kernel. This requires genuine asymmetry in the edge intensity $\mbblue{\rho_{\mathcal{E}}}$.

The distinction matters: kernel absorption preserves the dot-product form of connection probabilities, while edge-intensity asymmetry modifies the distribution of edge opportunities without changing the affinity kernel.

\begin{figure}[htbp]
  \centering
  \includegraphics[width=0.95\textwidth]{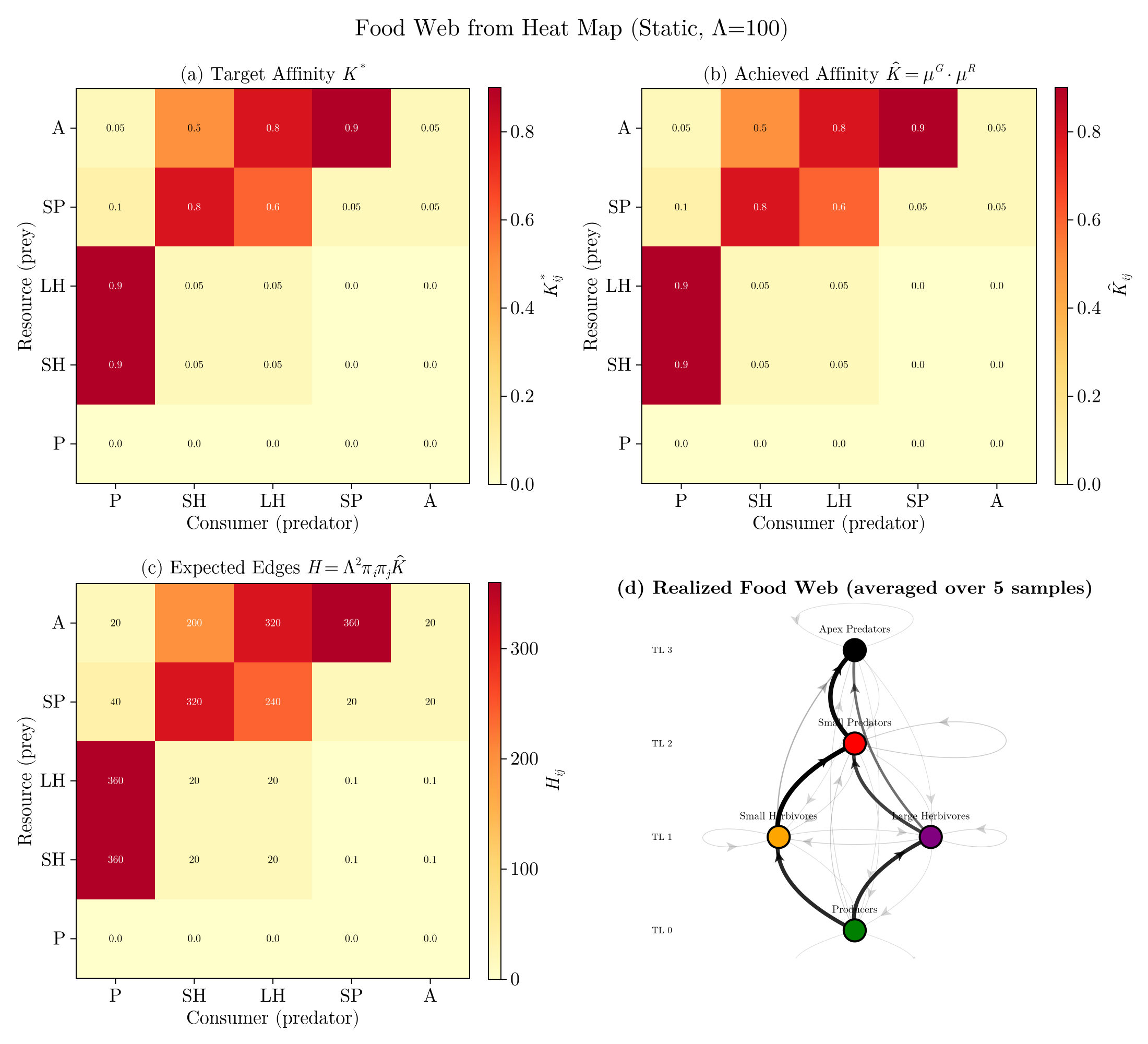}
  \caption{\textbf{Food web generation from IDPG ($\mblue{\Lambda} = 100$, five guilds).} (a) Target affinity matrix $K^*_{ij}$ specifying desired interaction structure between guilds: producers (P) are consumed by herbivores (SH, LH), which are consumed by predators (SP, A). (b) Achieved affinity $\hat{K}_{ij} = \mbgreen{\mu^G_i} \cdot \mbred{\mu^R_j}$ from optimized guild centroids in $d = 4$ dimensional latent space. (c) Expected edges $\mathcal{H}_{ij} = \mblue{\Lambda}^2 \cdot \pi_i \cdot \pi_j \cdot \hat{K}_{ij}$ combining guild abundances with interaction propensities. (d) Realized food web (averaged over 5 samples) showing trophic structure: arrow width proportional to edge count, vertical position indicates trophic level.}
  \label{fig:foodweb_static}
\end{figure}

\begin{figure}[htbp]
  \centering
  \includegraphics[width=0.90\textwidth]{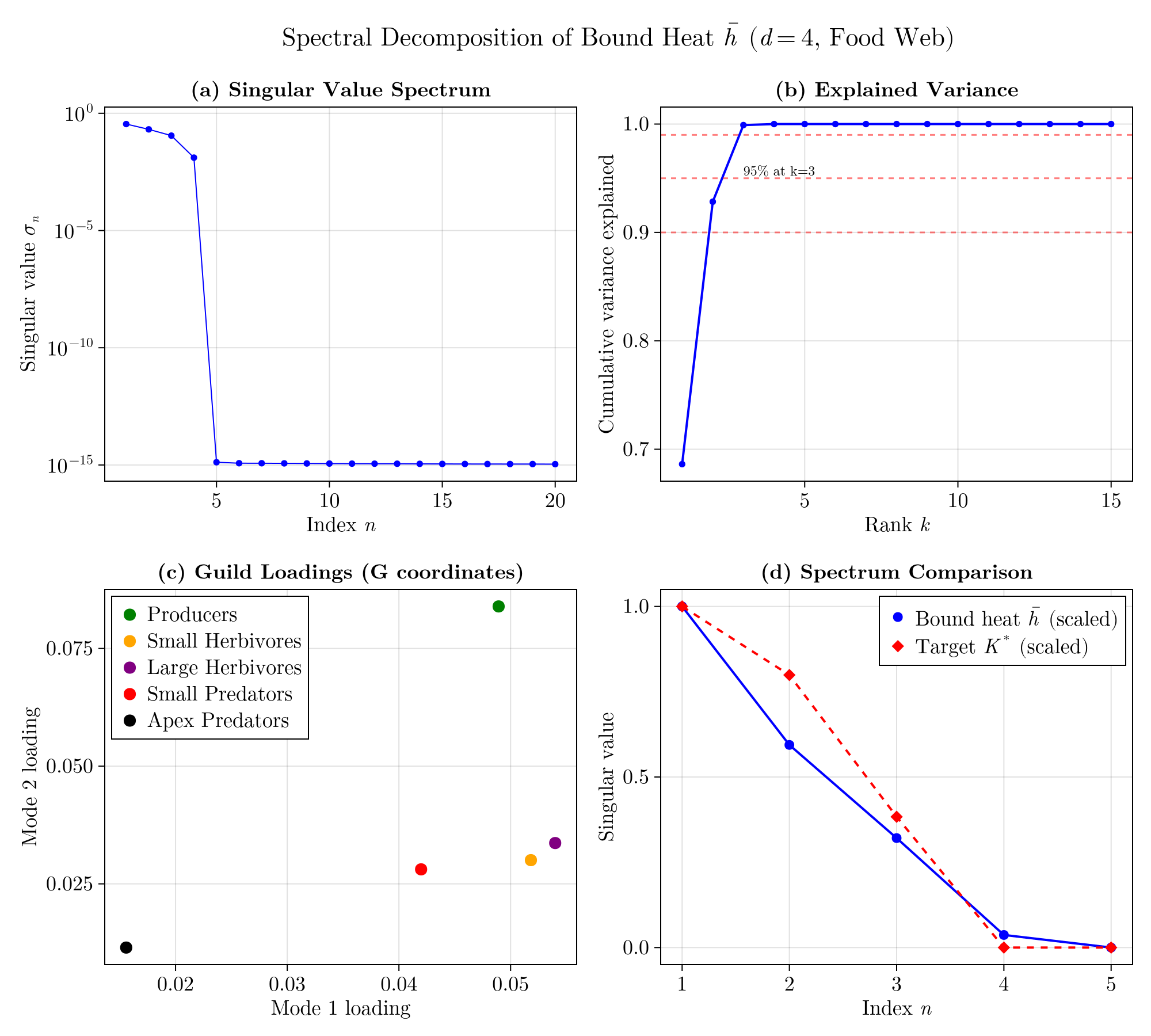}
  \caption{\textbf{Spectral decomposition of the food web heat map ($d = 4$).} (a) Singular value spectrum showing four non-zero singular values before machine precision, confirming $\mathrm{rank}(\overline{T}) = d = 4$. (b) Cumulative variance explained; three modes capture 95\% of the interaction structure. (c) Guild loadings in the space of the first two singular modes, revealing trophic organization: producers (green) and herbivores (yellow, purple) separate from predators (red, black). (d) Comparison of singular value spectra between the realized heat map and the target affinity matrix.}
  \label{fig:spectral_foodweb}
\end{figure}

\section{Time indexing and beyond}\label{sec:time-indexing}

So far we considered graphs, and their random models, as stuck in time. Yet, the motivating example of ecological food webs invite to consider what happens when evolution occurs, that is, when the intensities change in time.

This means that instead of observing one graph $G$, we observe a sequence of graphs $G_t = (V_t, E_t)$, where the index $t$ is commonly assumed to stand for time. In most RDPG time extensions, we do consider $V_t = V$ for each time $t$, that is, the set of vertices does not change, while the connections between them can change.

Each graph $G_t$ is generated by an RDPG model with parameters $\mbgreen{G}_t$ and $\mbred{R}_t$. A body of statistical results guides us to decide whether, given two graphs $G_t$ and $G_{t + \delta t}$, we can determine that $(\mbgreen{G}, \mbred{R})_t = (\mbgreen{G}, \mbred{R})_{t + \delta t}$ or not, that is, whether the change in the observation is actually induced by a movement of the points in $\mgreen{G}$ and $\mred{R}$ or by the inherent variability of the observation process.

Eventually, but this has so far received less attention, the graphs can be indexed by more than one variable, e.g., we could consider a spatiotemporal distribution of graphs $G_{x,y,t}$ where $x,y$ are some geographic coordinates and $t$ is time.

So far, there hasn't been an attempt to study from a dynamical system perspective the movement of the graph points in time and space.

\subsection{Two temporal scales}\label{sec:two-temporal-scales}

Introducing time into IDPG reveals two distinct dynamical scales:

\subsubsection{Sampling dynamics (fast scale)}

The Poisson point process describes when and where individuals appear. PPP is memoryless: individuals appear independently at rate $\mbblue{\rho}$. But richer temporal structure can arise from generalizing the sampling process:

\begin{itemize}
  \item \textbf{Hawkes processes}: self-exciting point processes where past events increase the rate of future events. An interaction at $(\mgreen{\vec{g}}, \mred{\vec{r}})$ temporarily boosts the intensity nearby, producing temporal clustering. This could model social reinforcement or predator-prey encounter dynamics.

  \item \textbf{Cox processes}: doubly stochastic processes where the intensity $\mbblue{\rho}$ is itself random, introducing additional variability in event rates.
\end{itemize}

These generalizations govern the temporal correlations in \emph{when} individuals appear, while the spatial structure of $\mbblue{\rho}$ governs \emph{where} they appear.

\subsubsection{Intensity evolution (slow scale)}

On a slower timescale, the intensity $\mbblue{\rho}$ itself can evolve. We write $\mbblue{\rho}(\mgreen{\vec{g}}, \mred{\vec{r}}, t)$ for a time-varying intensity on $\mblue{\Omega} = \mgreen{B^d_+} \times \mred{B^d_+}$. The sampling process then operates on a landscape that drifts over time.

To handle the domain constraint, we view $\mblue{\Omega} \subset \mathbb{R}^{2d}$ and use PDE operators in the full coordinate $(\mgreen{\vec{g}}, \mred{\vec{r}})$. The boundary $\partial \mblue{\Omega}$ has flat faces where some $\mgreen{g}_k = 0$ or $\mred{r}_k = 0$, and curved faces where $\|\mgreen{\vec{g}}\| = 1$ or $\|\mred{\vec{r}}\| = 1$. Three natural boundary conditions arise:

\begin{itemize}
  \item \textbf{Absorbing boundary} ($\mbblue{\rho} = 0$ on $\partial \mblue{\Omega}$): intensity that reaches the boundary vanishes. Without exogenous inputs, the total mass decreases over time, and $\mathbb{E}[N(t)]$ shrinks. This models extinction or loss of individuals that become too extreme in their interaction propensities.

  \item \textbf{Reflecting boundary} (no-flux condition $\nabla_{(\mgreen{\vec{g}}, \mred{\vec{r}})} \mbblue{\rho} \cdot \vec{n} = 0$ on $\partial \mblue{\Omega}$): intensity cannot escape $\mblue{\Omega}$. Total mass is conserved, so $\mathbb{E}[N(t)]$ remains constant even as the distribution evolves. This may be more natural for ecological applications where species cannot leave the space of viable niches and they accumulate at boundaries rather than disappearing.

  \item \textbf{Robin boundary} ($\alpha \mbblue{\rho} + \beta \nabla_{(\mgreen{\vec{g}}, \mred{\vec{r}})} \mbblue{\rho} \cdot \vec{n} = 0$ on $\partial \mblue{\Omega}$, with $\alpha, \beta \geq 0$): a linear combination of the intensity value and its normal flux vanishes at the boundary. This interpolates between absorbing ($\beta = 0$) and reflecting ($\alpha = 0$) conditions. The parameter ratio $\alpha / \beta$ controls the rate at which intensity ``leaks'' through the boundary: a small ratio yields near-reflecting behaviour with slow mass loss, while a large ratio approaches the absorbing case. Robin conditions can model partial permeability of the niche boundary, where some fraction of individuals at extreme positions are lost while others are retained.
\end{itemize}

In all three cases, no individuals are ever sampled at positions outside $\mblue{\Omega}$ (the intensity there is zero or inaccessible).

Under the product assumption $\mbblue{\rho}(\mgreen{\vec{g}}, \mred{\vec{r}}, t) = \mbgreen{\rho_G}(\mgreen{\vec{g}}, t) \cdot \mbred{\rho_R}(\mred{\vec{r}}, t)$, we can study the evolution of each marginal intensity separately. Moreover, if $\mbgreen{\rho_G}$ and $\mbred{\rho_R}$ each evolve according to independent PDEs, the product structure is preserved: the proposing and accepting landscapes evolve autonomously.

\subsection{PDE regimes on the intensity}\label{sec:pde-regimes}

Classic partial differential equations describe canonical modes of intensity evolution:

\subsubsection{Diffusion}

The heat equation
\[
  \frac{\partial \mbblue{\rho}}{\partial t} = \nu \Delta_{(\mgreen{\vec{g}}, \mred{\vec{r}})} \mbblue{\rho}
\]
where $\nu > 0$ is the diffusion coefficient, models spreading or mixing. An initially concentrated intensity diffuses outward, representing diversification or loss of specificity. In the product case, if $\mbgreen{\rho_G}$ diffuses, individuals become less specialized in their proposing behavior over time.

\subsubsection{Advection}

The transport equation
\[
  \frac{\partial \mbblue{\rho}}{\partial t} = - \nabla_{(\mgreen{\vec{g}}, \mred{\vec{r}})} \cdot (\vec{v}\, \mbblue{\rho})
\]
models directed drift. The intensity translates through the latent space at velocity $\vec{v}$, representing systematic change in interaction propensities. In ecological terms, this could model adaptation or environmental pressure shifting species' niches.

\subsubsection{Reaction-diffusion}

Combining local dynamics with spatial spreading:
\[
  \frac{\partial \mbblue{\rho}}{\partial t} = \nu \Delta_{(\mgreen{\vec{g}}, \mred{\vec{r}})} \mbblue{\rho} + f(\mbblue{\rho})
\]
where $f(\mbblue{\rho})$ captures local growth, decay, or competition. This can produce pattern formation, traveling waves, or stable heterogeneous distributions.

\subsubsection{Pursuit-evasion}

Under the product assumption, the two marginal intensities can be coupled through their centroids, modelling a predator-prey or pursuit-evasion dynamic in the latent space. The ``prey'' intensity $\mbgreen{\rho_G}$ is advected away from the ``predator'' centroid $\mbred{\tilde{\mu}_R}(t)$, while the ``predator'' intensity $\mbred{\rho_R}$ is advected toward the ``prey'' centroid $\mbgreen{\tilde{\mu}_G}(t)$. An elastic restoring term prevents either population from drifting indefinitely:
\[
  \frac{\partial \mbgreen{\rho_G}}{\partial t} = - \nabla \cdot (\vec{v}_G \cdot \mbgreen{\rho_G}), \quad \vec{v}_G = -\alpha (\mbred{\tilde{\mu}_R} - \vec{x}_0) - \gamma (\mgreen{\vec{g}} - \vec{x}_0)
\]
\[
  \frac{\partial \mbred{\rho_R}}{\partial t} = - \nabla \cdot (\vec{v}_R \cdot \mbred{\rho_R}), \quad \vec{v}_R = \beta (\mbgreen{\tilde{\mu}_G} - \vec{x}_0) - \gamma (\mred{\vec{r}} - \vec{x}_0)
\]
where $\alpha, \beta > 0$ control the evasion and pursuit speeds respectively, $\gamma > 0$ is the elastic centering strength, and $\vec{x}_0$ is a reference position. The centroids $\mbgreen{\tilde{\mu}_G}(t)$ and $\mbred{\tilde{\mu}_R}(t)$ are computed from the current intensities, making this a nonlinear, nonlocal system. The resulting dynamics produce coupled oscillatory motion in the latent space (Figure~\ref{fig:pde_reflecting}, bottom row).

\subsection{Induced dynamics on graph statistics}\label{sec:induced-dynamics}

As $\mbblue{\rho}$ evolves, so do the expected graph properties. Under the product assumption, the expected number of nodes $\mathbb{E}[N(t)] = \mgreen{c_G}(t) \cdot \mred{c_R}(t)$ and the expected edges $\mathbb{E}[|E(t)|]$ become functions of time, determined by the evolving marginal intensities.

For instance, under pure diffusion with no-flux boundary conditions on $B^d_+$, total mass is conserved: $\mgreen{c_G}(t) = \mgreen{c_G}(0)$. But the intensity-weighted means $\mbgreen{\mu_G}(t)$ and $\mbred{\mu_R}(t)$ may change, affecting expected edge counts even as expected node counts remain constant.

This connects random graph theory to the broader literature on PDE inference from stochastic observations.

\begin{figure}[htbp]
  \centering
  \includegraphics[width=\textwidth]{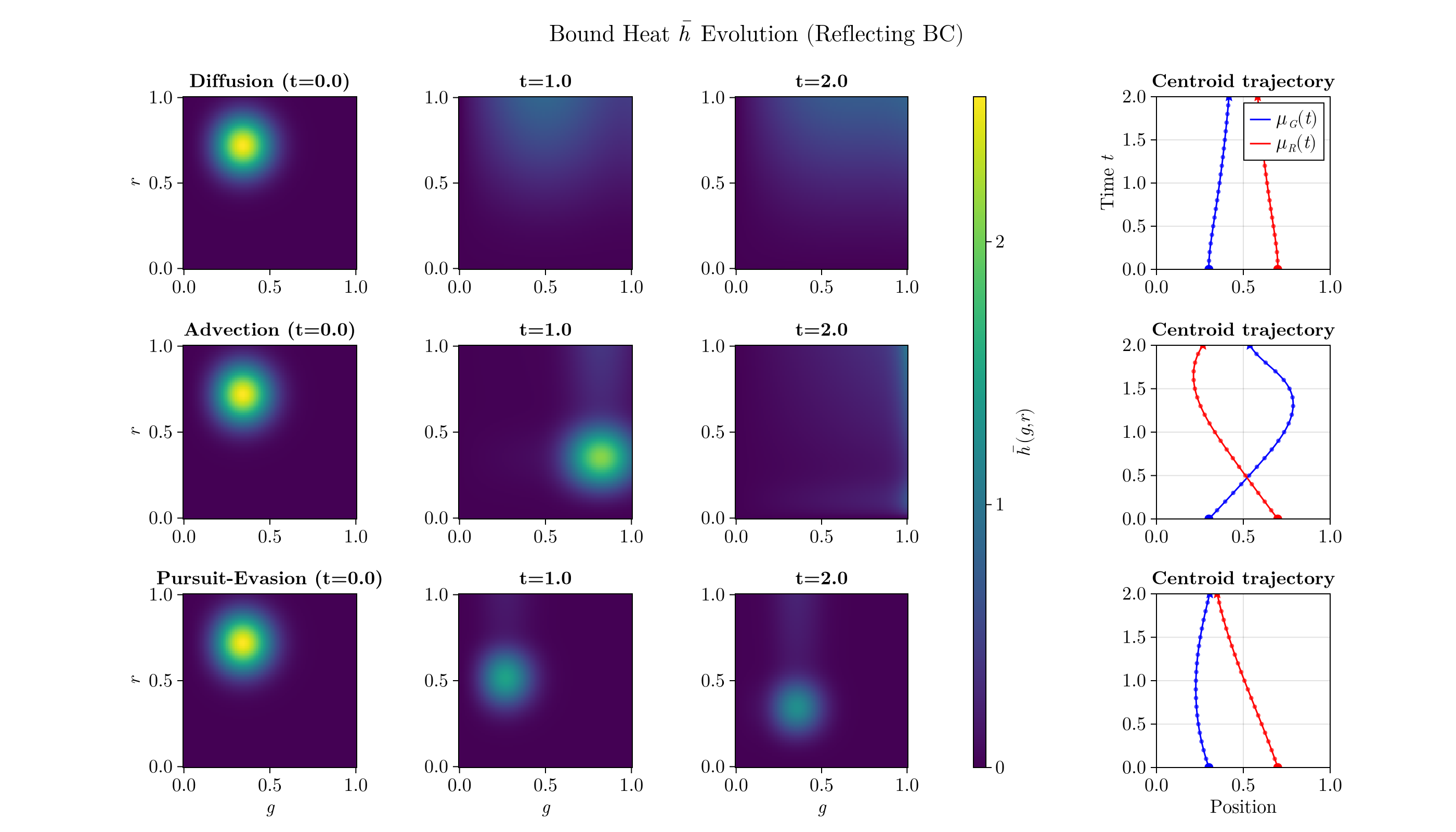}
  \caption{\textbf{Bound heat evolution under PDE dynamics (reflecting boundary conditions).} Three rows show diffusion, advection, and pursuit-evasion dynamics, each with snapshots at $t \in \{0, 1, 2\}$ and centroid trajectories $\tilde{\mu}_G(t)$, $\tilde{\mu}_R(t)$. Diffusion spreads the intensity while centroids remain stable. Advection translates both marginals rightward. Pursuit-evasion creates coupled oscillatory motion as the ``predator'' ($\mbgreen{\rho_G}$) chases the ``prey'' ($\mbred{\rho_R}$). Reflecting boundaries preserve total mass throughout.}
  \label{fig:pde_reflecting}
\end{figure}

\begin{figure}[htbp]
  \centering
  \includegraphics[width=0.95\textwidth]{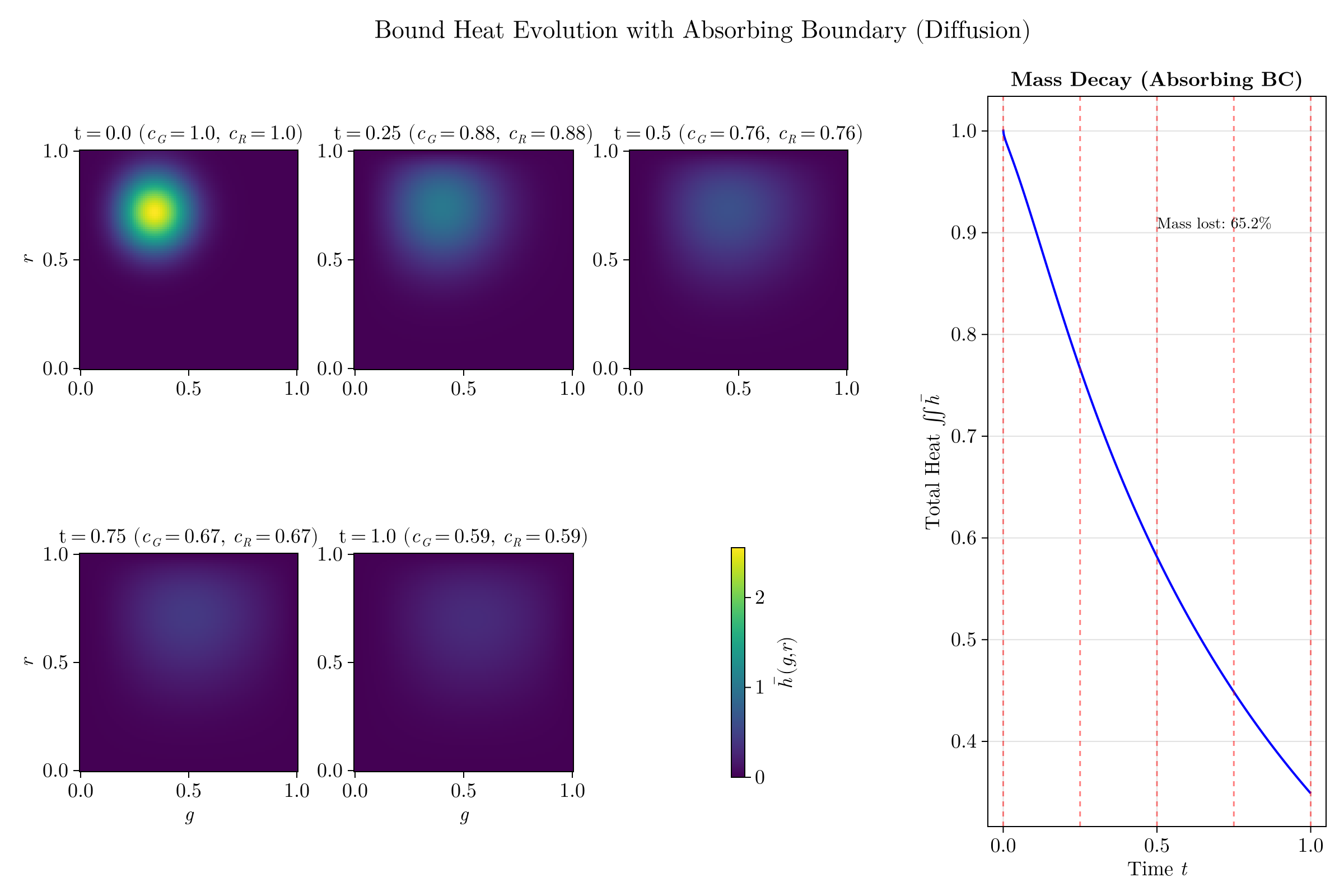}
  \caption{\textbf{Mass decay with absorbing boundary conditions.} Left: Bound heat snapshots at $t \in \{0, 0.25, 0.5, 0.75, 1.0\}$ showing intensity loss as diffusion carries mass to the boundary. Titles show marginal charges $\mgreen{c_G}(t)$ and $\mred{c_R}(t)$. Right: Total bound heat $\int \int \overline{h}\, dg\, dr$ as a function of time, decaying from 1.0 to approximately 0.35 (65\% mass loss). The bound heat decays faster than mass because it depends on the product $\mgreen{c_G} \cdot \mred{c_R}$.}
  \label{fig:pde_absorbing}
\end{figure}

\section{Computational experiments}\label{sec:simulations}

We verified the theoretical predictions through Monte Carlo simulations. Full details and additional figures are available in the supplementary materials.

\subsection{Perennial vs ephemeral scaling}

Figure~\ref{fig:sim1} confirms the fundamental scaling dichotomy. Using a product intensity on $B^2_+$ with Gaussian marginals ($\kappa = 15$, means $\mbgreen{\mu_G} = (0.6, 0.4)$ and $\mbred{\mu_R} = (0.5, 0.5)$), we generated 1000 replications at each intensity level $\mblue{\Lambda} \in \{10, 25, 50, 100, 200\}$.

\begin{figure}[htbp]
  \centering
  \includegraphics[width=0.85\textwidth]{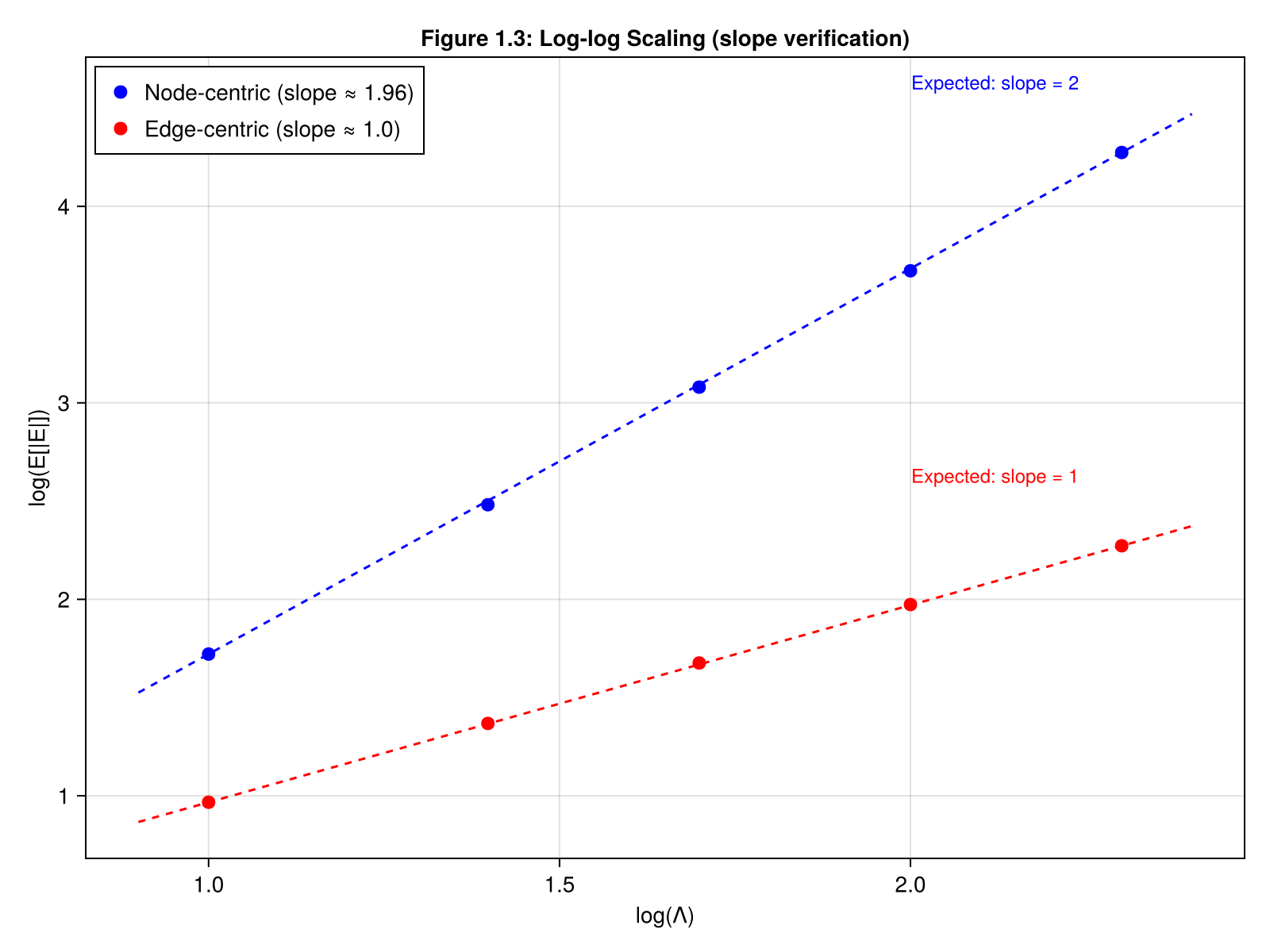}
  \caption{Log-log scaling of expected edges vs total intensity. Perennial (blue) shows slope $\approx 1.97$ (theory: 2), ephemeral (red) shows slope $\approx 1.01$ (theory: 1). Dashed lines show theoretical predictions.}
  \label{fig:sim1}
\end{figure}

The empirical slopes (1.97 for perennial and 1.01 for ephemeral) match theoretical predictions within sampling error.

The structural difference is visually striking (Figure~\ref{fig:sim1_viz}): at $\mblue{\Lambda} = 50$, a typical perennial realization has $N \approx 45$ nodes and $|E| \approx 1000$ edges (dense), while ephemeral yields a sparser graph with nodes organized into disjoint pairs.

\begin{figure}[htbp]
  \centering
  \includegraphics[width=\textwidth]{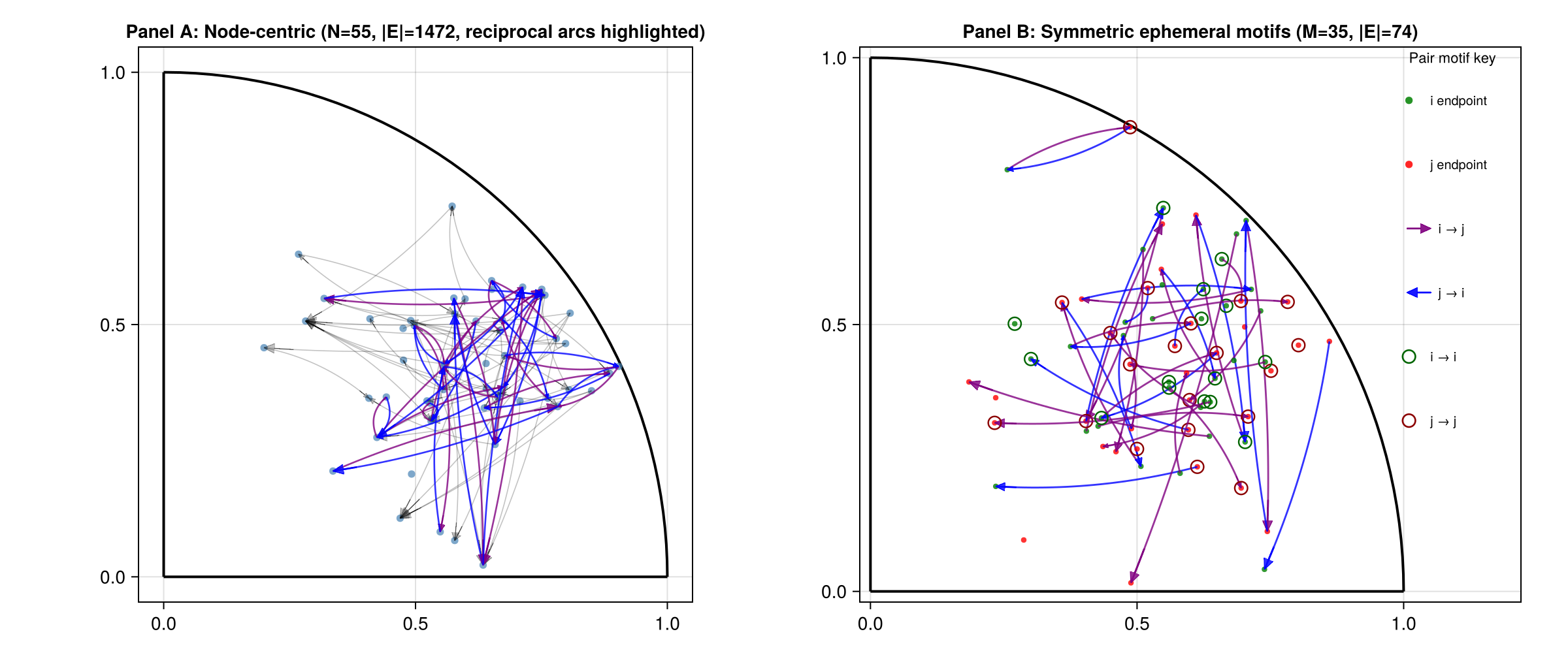}
  \caption{Sample realizations at $\mblue{\Lambda} = 50$. Panel A: Perennial produces a dense graph ($N = 45$, $|E| \approx 1000$); directed edges are curved, and reciprocated pairs are highlighted as opposite arcs (purple/blue). Panel B: Symmetric ephemeral produces disjoint 2-node components with explicit motif coding. In each pair, endpoints are colored green/red (local labels $i$/$j$), cross-edges are directional (purple: $i \to j$, blue: $j \to i$), and self-loops are shown separately (dark green: $i \to i$, dark red: $j \to j$). Both panels use identical intensity parameters; the difference arises purely from the realization rule.}
  \label{fig:sim1_viz}
\end{figure}

\subsection{Intermediate regime interpolation}

Figure~\ref{fig:sim2} verifies the overlap probability formula. With $\mblue{\Lambda} = 50$ and observation window $W = 1$, we varied the mean lifetime $\eta$ across four orders of magnitude.

\begin{figure}[htbp]
  \centering
  \includegraphics[width=0.85\textwidth]{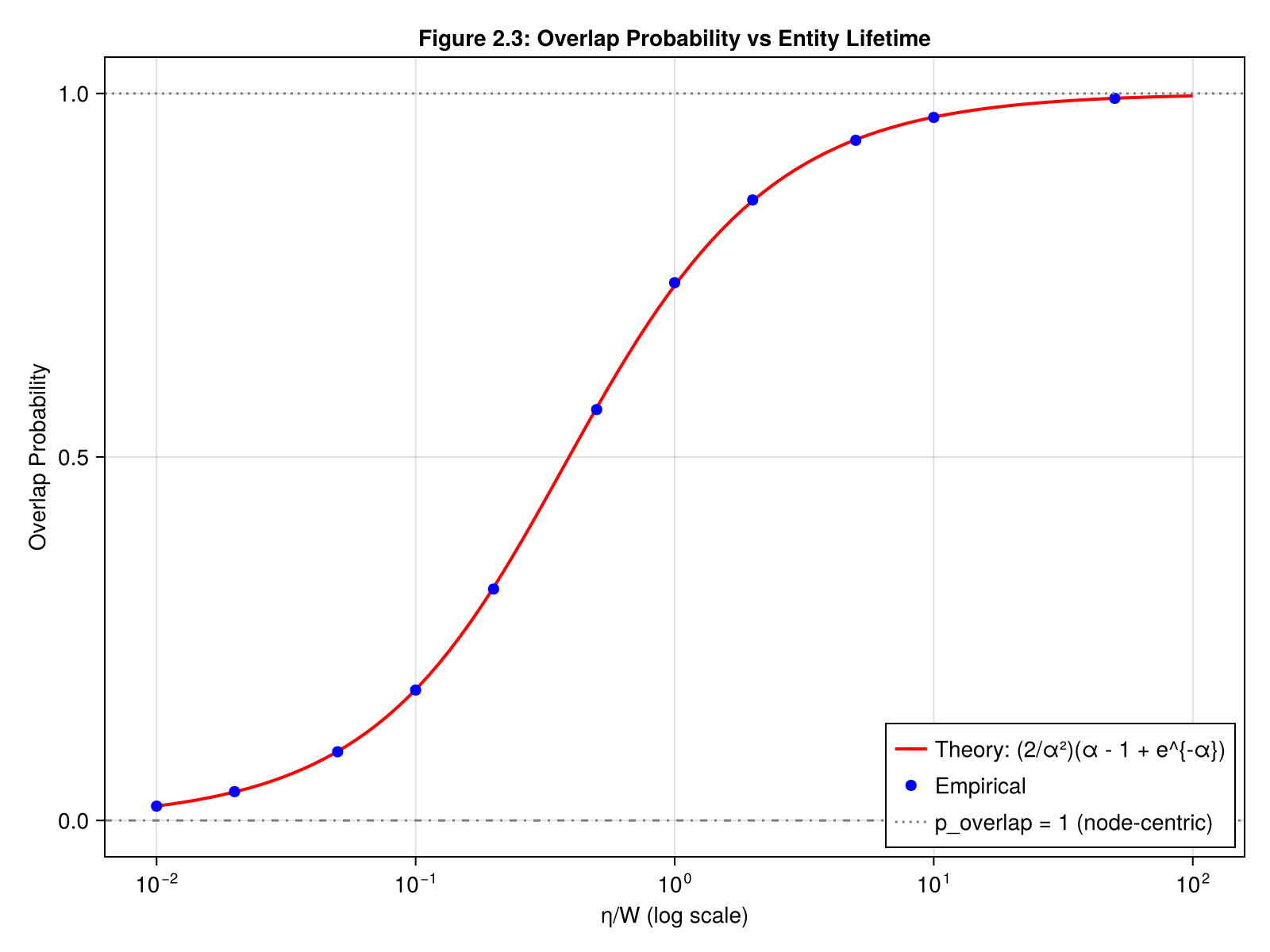}
  \caption{Overlap probability vs normalized lifetime $\eta/W$. Empirical estimates (points) match the theoretical curve $p_{\text{overlap}} = (2/u^2)(u - 1 + e^{-u})$ where $u = W/\eta$. The transition from ephemeral ($p_{\text{overlap}} \approx 0$) to perennial ($p_{\text{overlap}} \approx 1$) spans roughly two orders of magnitude in $\eta/W$.}
  \label{fig:sim2}
\end{figure}

The empirical overlap probabilities match the theoretical formula with relative errors below 1\% across all tested values. The transition occurs smoothly: at $\eta/W = 0.1$, about 18\% of individual pairs overlap; at $\eta/W = 1$, about 74\% overlap; at $\eta/W = 10$, overlap exceeds 97\%.

\begin{figure}[htbp]
  \centering
  \includegraphics[width=\textwidth]{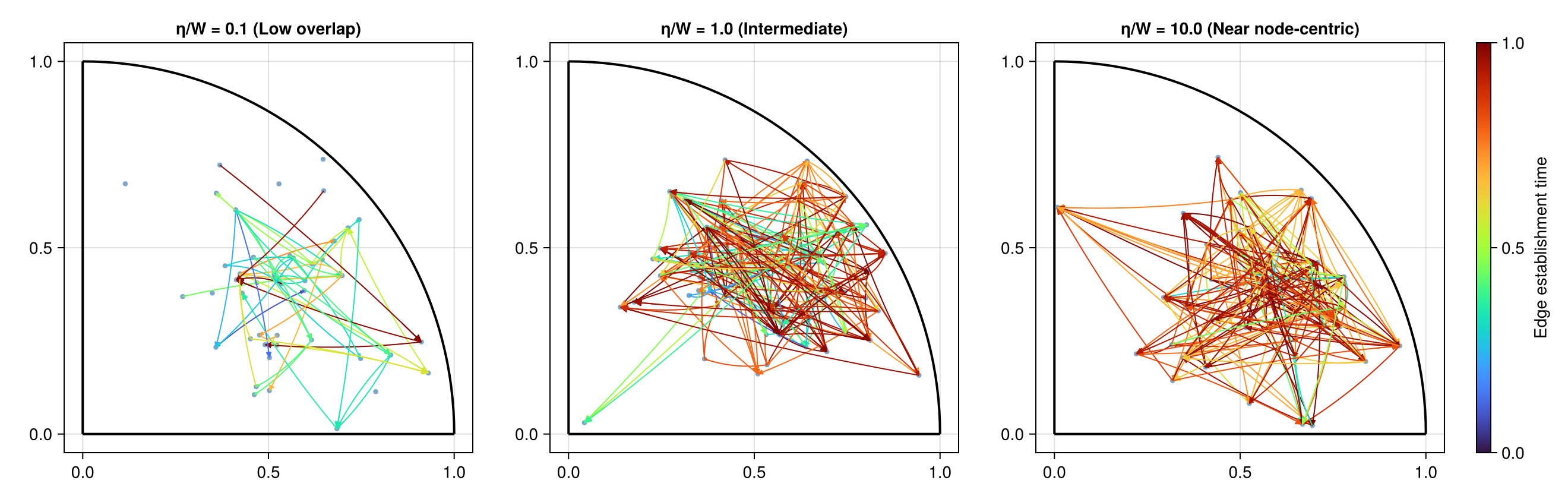}
  \caption{Sample graphs at three lifetime regimes, all with $\mblue{\Lambda} = 50$. Left: $\eta/W = 0.1$ (near ephemeral) shows sparse edges. Center: $\eta/W = 1$ (intermediate) shows moderate connectivity. Right: $\eta/W = 10$ (near perennial) shows dense connectivity approaching the $\mathbf{R}_\infty$ limit. Directed edges are plotted as curved arrows, so reciprocal interactions are visible as paired arcs. For visualization clarity, we display a fixed fraction of realized edges and color each by sampled establishment time within its overlap interval (shared colorbar).}
  \label{fig:sim2_viz}
\end{figure}

\subsection{Ratio invariance under PDE evolution}

Figure~\ref{fig:sim6} tests whether the ratio formula
\[
  \frac{\mathbb{E}[E]_{\text{perennial}}}{\mathbb{E}[E]_{\text{ephemeral}}} = \frac{\mblue{\Lambda}}{2}
\]
holds under the distinct-pair perennial convention when the intensity evolves under PDEs (symmetric ephemeral rule with four edge trials per sampled pair). We simulated a 5-species food web ($d = 4$) under four dynamics: static, diffusion, advection (with absorbing boundary), and pursuit-evasion.

\begin{figure}[htbp]
  \centering
  \includegraphics[width=0.85\textwidth]{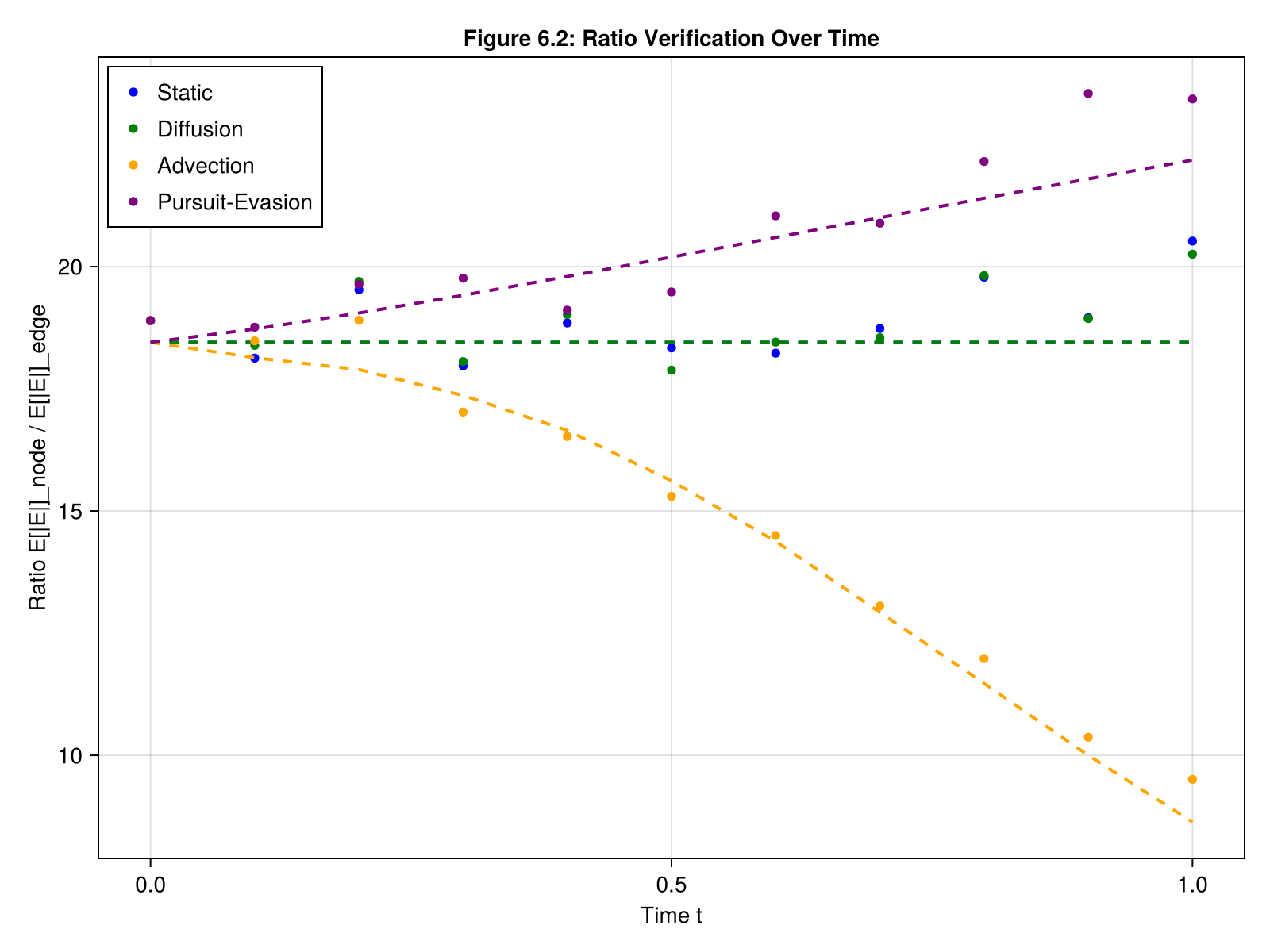}
  \caption{Ratio $\mathbb{E}[E]_{\text{perennial}} / \mathbb{E}[E]_{\text{ephemeral}}$ over time under different PDE regimes. Dashed lines show theoretical $\mblue{\Lambda}(t)/2$ (distinct-pair perennial convention). Static and diffusion maintain constant $\mblue{\Lambda}$; advection decreases $\mblue{\Lambda}$ due to absorbing boundaries; pursuit-evasion increases $\mblue{\Lambda}$. In all cases, the empirical ratio tracks $\mblue{\Lambda}(t)/2$.}
  \label{fig:sim6}
\end{figure}

Under the distinct-pair perennial convention, the ratio correctly tracks $\mblue{\Lambda}(t)/2$ in all regimes (mean absolute error $< 3$ throughout). This confirms that the fundamental relationship between realization rules persists even as the underlying intensity evolves. The ratio depends only on the instantaneous total intensity, not on the history of the dynamics.

\begin{figure}[htbp]
  \centering
  \includegraphics[width=\textwidth]{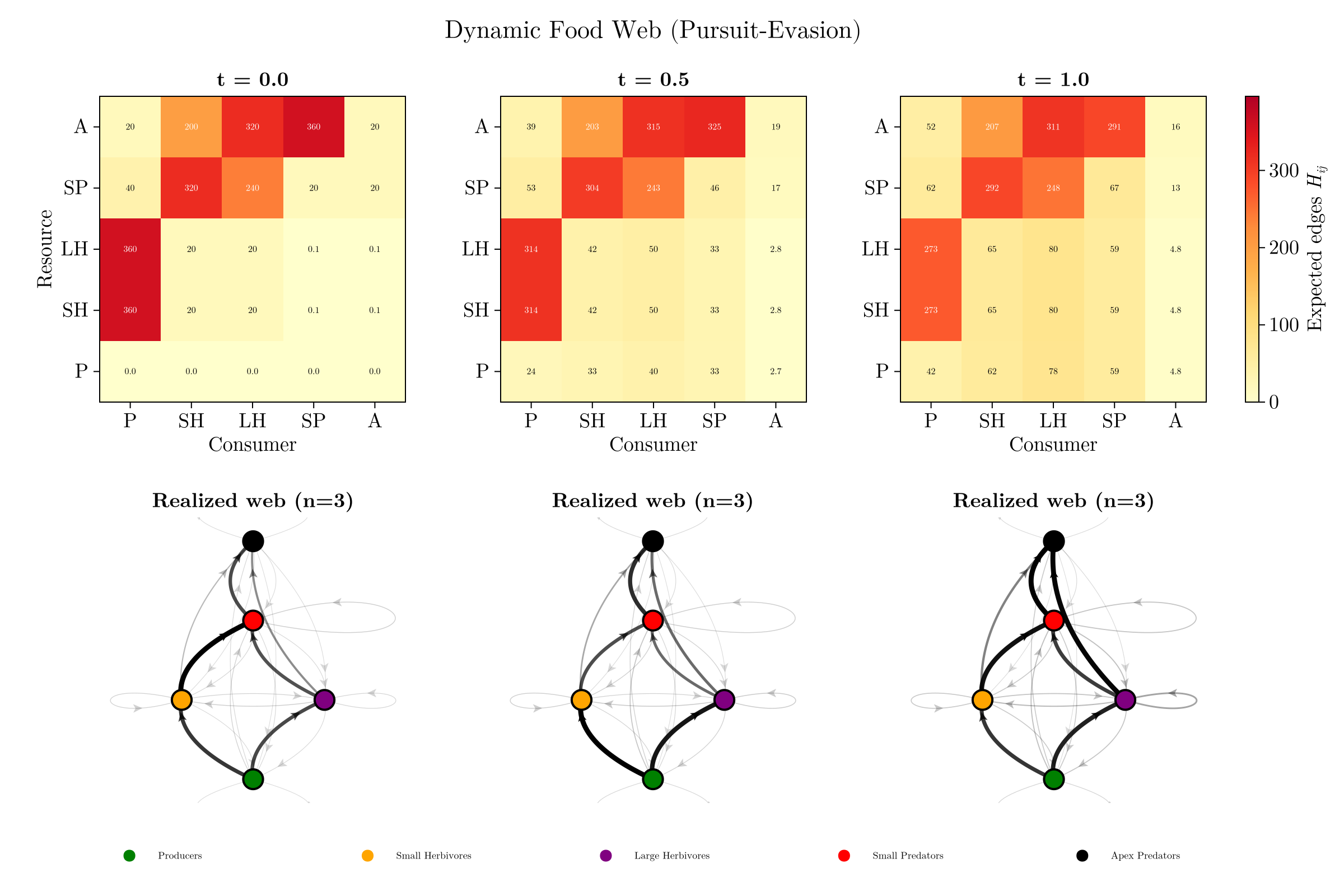}
  \caption{\textbf{Dynamic food web evolution under pursuit-evasion.} Top row: Expected edge matrices $\mathcal{H}_{ij}(t)$ at $t \in \{0, 0.5, 1.0\}$, showing how guild interactions shift as predators chase prey through the latent space. Bottom row: Realized food webs (averaged over 3 samples). Initially, herbivores (SH, LH) feed heavily on producers (P). As pursuit-evasion dynamics unfold, the interaction structure becomes more diffuse, with increased cross-guild connections and modified trophic flow. Edge weights reflect changing expected edge counts.}
  \label{fig:foodweb_dynamic}
\end{figure}

\section{Inference of an IDPG model}\label{sec:inference}

In its most generality, the inference problem for IDPG is: given one (or many) observed graph $G = (V, E)$, and eventually ancillary information such as a partial or complete position of the individuals in the latent spaces, can we recover the intensity $\mbblue{\rho}$? or at least the marginal intensities $\mbgreen{\rho_G}$ and $\mbred{\rho_R}$ under a product assumption? The problem is completely open, and here we only sketch some reflections.

An immediate complication is that, under certain observation regimes, the observed vertex set $V$ might correspond to $N_{\text{obs}}$, not $N$: we see only nodes with at least one edge. We miss the isolated nodes, those sampled from PPP($\Lambda$) but forming no connections. Since isolation probability depends on position (nodes with small $\|\mgreen{\vec{g}}\|$ or $\|\mred{\vec{r}}\|$ are more likely to be isolated), the observed positions are a biased sample from $\mbblue{\rho}$. Any inference procedure must either correct for this selection bias or acknowledge that it estimates the intensity conditional on observability.

For a perennial IDPG, a natural approach proceeds in two stages:

\begin{enumerate}
  \item \textbf{Embed the observed nodes.} Apply the standard RDPG inference procedure: compute the singular value decomposition of the adjacency matrix $\mathbf{A}$, select an appropriate dimension, and obtain estimated positions for each observed node.

  \item \textbf{Estimate the intensity.} Treat the embedded positions as a point cloud and apply density estimation techniques \cite{silverman2018density} to recover $\mbblue{\rho}$, or its marginals $\mbgreen{\rho_G}$ and $\mbred{\rho_R}$ under a product assumption.
\end{enumerate}

The feasibility and accuracy of this procedure may depend on additional assumptions about the structure of $\mbblue{\rho}$. A natural constraint is to model $\mbblue{\rho}$ as a mixture of multivariate Gaussian distributions, which offers a flexible yet tractable family for density estimation.

For ephemeral observations, the inference problem requires specifying the observation model. If interaction pairs are observed directly with their latent positions, estimating $\mbblue{\rho}$ proceeds via density estimation on the position point cloud. More commonly, one observes a discretized or aggregated version, e.g., interaction counts between clusters or categories, requiring additional modeling to relate the summary to the underlying continuous intensity.

Finally, inferring the PDE dynamics of a time dynamic IDPG from a sequence of observed graphs $G_{t_1}, G_{t_2}, \dots$ combines the IDPG inference problem at each snapshot with dynamical estimation across time.

\section{Discussion and future directions}\label{sec:discussion}

We have introduced Intensity Dot Product Graphs as a measure-theoretic generalization of Random Dot Product Graphs, where discrete nodes give way to continuous intensities and the probability matrix gives way to the heat map. This framework accommodates both perennial and ephemeral generative mechanisms, with the intermediate regime interpolating between them through individual lifetimes. The heat map (comprising raw heat and bound heat in the product case) provides a unified language for interaction structure that recovers RDPG in the Dirac limit while extending naturally to dynamic settings where the intensity evolves under PDEs.

Several directions remain open for future investigation.

\subsection{Spectral theory of the heat operator}

The spectral structure of the bound heat operator $\overline{T}$, its singular values, singular functions, and their relationship to network properties, merit deeper analysis. Key questions include:

\begin{itemize}
  \item \textbf{Explicit spectrum for Gaussian intensities.} When $\mbgreen{\rho_G}$ and $\mbred{\rho_R}$ are truncated Gaussians on $B^d_+$, the Gram matrices $A$ and $B$ are computable in terms of Gaussian moments. What is the resulting singular value distribution? How does it depend on the concentration (variance) and centering of the Gaussians?

  \item \textbf{Spectral gap of the Laplacian.} The continuous Laplacian $\mathcal{L} = D - T$ has a spectral gap controlling network connectivity. How does this gap relate to geometric properties of $\mbblue{\rho}$? Is there an analog of Cheeger's inequality relating spectral gap to an isoperimetric constant on $\mblue{\Omega}$?

  \item \textbf{Spectral clustering.} In graphon theory and spectral clustering \cite{vonluxburg2007tutorial}, eigenvectors of the Laplacian identify community structure. Do the singular functions of the bound heat operator similarly reveal latent structure in the intensity landscape? This could provide a principled approach to identifying species or functional groups in ecological applications.
\end{itemize}

\subsection{Heat kernel interpretation}

Our ``heat map,'' ``raw heat'', and ``bound heat'' terminology invites connection to the classical heat kernel $p(t, x, y)$ satisfying the heat equation. This connection is deeper than nomenclature, suggesting a program to develop genuine heat-theoretic foundations for IDPG:

\begin{itemize}
  \item \textbf{Heat semigroup generation.} A fundamental question is whether the continuous Laplacian $\mathcal{L} = D - T$ generates a strongly continuous semigroup $e^{-t \mathcal{L}}$ on $L^2(\mblue{\Omega})$ \cite{engel2000one}. If so, this semigroup would describe the evolution of ``influence'' across $\mblue{\Omega}$, with the bound heat operator $\overline{T}$ playing the role of a transition kernel. The theory of heat kernels on manifolds and graphs \cite{grigoryan2009heat} provides the analytical framework; extending these results to our infinite-dimensional setting requires verifying that $\mathcal{L}$ satisfies appropriate sectoriality or dissipativity conditions.

  \item \textbf{Heat kernel on graphs.} The graph heat kernel $H_t = e^{-t L}$ describes diffusion on a network \cite{chung1997spectral}. Its entries $(H_t)_{ij} = \sum_k e^{-\lambda_k t} \phi_k(i) \phi_k(j)$ decay exponentially in the Laplacian eigenvalues. The analogous continuous construction would yield a kernel $p(t, \mgreen{\vec{g}}, \mgreen{\vec{g}}') = \sum_n e^{-\lambda_n t} \phi_n(\mgreen{\vec{g}}) \phi_n(\mgreen{\vec{g}}')$ describing how interaction propensity diffuses through the intensity landscape.

  \item \textbf{Spectral representation.} The classical heat kernel admits the spectral representation $p(t,x,y) = \sum_n e^{-\lambda_n t} \phi_n(x) \phi_n(y)$. If a similar representation holds in our setting, equilibration time scales would be governed by the spectrum of the evolution generator (e.g., real parts of eigenvalues in the directed case), not directly by singular values.

  \item \textbf{Diffusion maps.} The diffusion maps framework \cite{coifman2006diffusion} uses heat kernels to construct embeddings that respect intrinsic geometry. Our heat map could provide a similar embedding of $\mblue{\Omega}$, where distances reflect interaction propensity rather than Euclidean distance.
\end{itemize}

\subsection{Dynamics and spectral evolution}

When the intensity $\mbblue{\rho}(t)$ evolves under a PDE (diffusion, advection, pursuit-evasion), the heat map $\mathcal{H}(t)$ and its spectrum co-evolve. This opens questions at the intersection of spectral theory and dynamical systems:

\begin{itemize}
  \item \textbf{Spectral tracking.} How do singular values $\sigma_k(t)$ evolve as $\mbblue{\rho}(t)$ changes? Perturbation theory (Weyl's inequality) guarantees Lipschitz continuity in the Hilbert-Schmidt norm, but finer structure components such as rate of change, crossing of singular values, bifurcations are unexplored.

  \item \textbf{Spectral gap dynamics.} Does the Laplacian's spectral gap increase or decrease under diffusion? Under pursuit-evasion? A shrinking gap would indicate network fragmentation; a growing gap, consolidation.

\item \textbf{Invariant spectral features.} Are some spectral quantities preserved under certain classes of PDE evolution? For example, total mass $\mblue{\Lambda}(t)$ is conserved under reflecting-boundary diffusion. There may be analogous spectral invariants that characterize the interaction structure.
\end{itemize}

\subsection{Inference and estimation}

Practical application requires inferring the intensity $\mbblue{\rho}$ from observed graphs. Key challenges include:

\begin{itemize}
  \item \textbf{Identifiability.} At the population level (exact raw heat), identifiability of $\mbblue{\rho}$ can hold under regularity assumptions (see Section~\ref{sec:heatmap}). The open challenge is statistical/coordinate identifiability from finite sampled graphs: disentangling latent-coordinate ambiguities (RDPG/SVD indeterminacy) and finite-sample noise when estimating $\mbblue{\rho}$.

  \item \textbf{Estimation from samples.} Given a single graph (or sequence of graphs) from an IDPG, how can we estimate the underlying intensity? Maximum likelihood, method of moments, and spectral methods are all candidates. The RDPG inference literature \cite{athreya2018statistical} provides a starting point, though the continuous intensity setting introduces new challenges.

  \item \textbf{Model selection.} How can we test whether an observed graph is better described by perennial or ephemeral sampling? Under the distinct-pair perennial convention, the ratio $\mblue{\Lambda}/2$ between expected edge counts provides a starting point (or $\frac{\mblue{\Lambda}+1}{2}$ if perennial self-loops are included), but distributional tests are needed.
\end{itemize}

\subsection{Extensions of the kernel}

The dot product kernel $K(s,t) = \mgreen{\vec{g}_s} \cdot \mred{\vec{r}_t}$ is natural for bilinear interactions but not universal. Extensions include:

\begin{itemize}
  \item \textbf{General bilinear forms.} Replace the dot product with $\mgreen{\vec{g}_s}^\top M \mred{\vec{r}_t}$ for a matrix $M$, allowing asymmetric weighting of coordinates.

  \item \textbf{Non-bilinear kernels.} Gaussian RBF kernels $K(s,t) = \exp(-\|s - t\|^2 / (2 \sigma^2))$ encode similarity rather than compatibility. Such kernels have infinite rank, changing the spectral picture dramatically.

  \item \textbf{Multiplex and higher-order interactions.} Multiple edge types (multiplex) or hyperedges (higher-order) require tensor-valued heat maps. The mathematical framework generalizes, but computational tractability may suffer.
\end{itemize}

\bibliographystyle{plainnat}
\bibliography{bibliography}

\appendix
\clearpage

\section{Derivations of expected edge counts}\label{appendix:derivations}

We derive expected edge counts for both realization rules under the product assumption $\mbblue{\rho}(\mgreen{\vec{g}}, \mred{\vec{r}}) = \mbgreen{\rho_G}(\mgreen{\vec{g}}) \cdot \mbred{\rho_R}(\mred{\vec{r}})$ on $\mblue{\Omega} = \mgreen{B^d_+} \times \mred{B^d_+}$.

\subsection{Key tools from point process theory}

\textbf{Campbell's formula} \cite[Prop.~2.7]{last2017lectures} \cite[Sec.~3.2]{kingman1993poisson}. For a Poisson point process with intensity measure $\lambda$, and any measurable function $f$ with $\int |f| \, d\lambda < \infty$:
\[
  \mathbb{E}\!\left[\sum_{x \in \text{PPP}(\lambda)} f(x)\right] = \int f(x)\, \lambda(dx)
\]

\textbf{Second factorial moment measure}. For a point process, the second factorial moment measure $M_{[2]}$ is defined on product sets by $M_{[2]}(A \times B) = \mathbb{E}[N(A) \cdot N(B)]$ for disjoint $A, B$, representing the expected number of ordered pairs of distinct points \cite[Sec.~5.4]{daley2003introduction}. For a Poisson process, counts in disjoint sets are independent \cite[Ex.~6.1(a)]{daley2003introduction}, so:
\[
  M_{[2]}(A \times B) = \mathbb{E}[N(A)] \cdot \mathbb{E}[N(B)] = \lambda(A) \cdot \lambda(B)
\]
Thus $M_{[2]} = \lambda \otimes \lambda$, and for any measurable $f$:
\[
  \mathbb{E}\!\left[\sum_{x \neq y} f(x,y)\right] = \iint f(x,y)\, \lambda(dx)\, \lambda(dy)
\]

\emph{Important:} The second factorial moment formula computes expectations but does not imply that the process of pairs is a PPP. In the perennial rule, edges are conditionally independent given positions but marginally dependent through shared nodes.

\textbf{Product structure and Fubini's theorem}. When $\lambda$ is a product measure, Fubini's theorem \cite{billingsley1995probability} permits iterated integration:
\[
  \int_{\mblue{\Omega}} f(\mgreen{\vec{g}}, \mred{\vec{r}})\, \lambda(d\mgreen{\vec{g}}, d\mred{\vec{r}}) = \int_{\mgreen{G}} \int_{\mred{R}} f(\mgreen{\vec{g}}, \mred{\vec{r}})\, \mbgreen{\rho_G}(d\mgreen{\vec{g}})\, \mbred{\rho_R}(d\mred{\vec{r}})
\]

\subsection{Notation (product case)}

The derivations below assume product intensity $\mbblue{\rho} = \mbgreen{\rho_G} \otimes \mbred{\rho_R}$. We use:
\begin{itemize}
  \item Total intensity: $\mblue{\Lambda} = \int_{\mblue{\Omega}} \mbblue{\rho} = \mgreen{c_G} \cdot \mred{c_R}$, also written $\mathbb{E}[N]$
  \item Marginal total intensities: $\mgreen{c_G} = \int_{\mgreen{G}} \mbgreen{\rho_G}$ and $\mred{c_R} = \int_{\mred{R}} \mbred{\rho_R}$
  \item Intensity-weighted mean positions: $\mbgreen{\mu_G} = \int_{\mgreen{G}} \mgreen{\vec{g}}\, \mbgreen{\rho_G}(\mgreen{\vec{g}})\, d\mgreen{\vec{g}}$ and $\mbred{\mu_R} = \int_{\mred{R}} \mred{\vec{r}}\, \mbred{\rho_R}(\mred{\vec{r}})\, d\mred{\vec{r}}$
  \item Normalized means: $\mbgreen{\tilde{\mu}_G} = \mbgreen{\mu_G} / \mgreen{c_G}$ and $\mbred{\tilde{\mu}_R} = \mbred{\mu_R} / \mred{c_R}$
\end{itemize}

\subsection[Perennial derivation]{Perennial derivation ($\mathbf{R}_\infty$)}\label{sec:appendix-perennial}

The perennial rule samples $N$ nodes from PPP($\mbblue{\rho}$) on $\mblue{\Omega}$. The expected number of edges is:
\[
  \mathbb{E}[E]_{\text{perennial}} = \mathbb{E}\!\left[\sum_{i \neq j} \mathds{1}[\text{edge } i \to j]\right] = \mathbb{E}\!\left[\sum_{i \neq j} (\mgreen{\vec{g}_i} \cdot \mred{\vec{r}_j})\right]
\]

Applying the second factorial moment formula with $f(s,t) = \mgreen{\vec{g}_s} \cdot \mred{\vec{r}_t}$:
\[
  \mathbb{E}[E]_{\text{perennial}} = \iint_{\mblue{\Omega} \times \mblue{\Omega}} (\mgreen{\vec{g}_s} \cdot \mred{\vec{r}_t})\, \mbblue{\rho}(s)\, \mbblue{\rho}(t)\, ds\, dt
\]

Under product intensity $\mbblue{\rho} = \mbgreen{\rho_G} \otimes \mbred{\rho_R}$, expanding $\mbblue{\rho}(s) = \mbgreen{\rho_G}(\mgreen{\vec{g}_s})\, \mbred{\rho_R}(\mred{\vec{r}_s})$ and applying Fubini to separate the four variables:
\[
  \mathbb{E}[E]_{\text{perennial}} = \iiiint (\mgreen{\vec{g}_s} \cdot \mred{\vec{r}_t})\, \mbgreen{\rho_G}(\mgreen{\vec{g}_s})\, \mbred{\rho_R}(\mred{\vec{r}_s})\, \mbgreen{\rho_G}(\mgreen{\vec{g}_t})\, \mbred{\rho_R}(\mred{\vec{r}_t})\, d\mgreen{\vec{g}_s}\, d\mred{\vec{r}_s}\, d\mgreen{\vec{g}_t}\, d\mred{\vec{r}_t}
\]

The dot product $\mgreen{\vec{g}_s} \cdot \mred{\vec{r}_t} = \sum_k (\mgreen{g}_s)_k (\mred{r}_t)_k$ depends only on $\mgreen{\vec{g}_s}$ and $\mred{\vec{r}_t}$. By linearity:
\begin{align}
  \mathbb{E}[E]_{\text{perennial}}
  &= \sum_k
  \left[\int (\mgreen{g}_s)_k\, \mbgreen{\rho_G}(\mgreen{\vec{g}_s})\, d\mgreen{\vec{g}_s}\right]
  \left[\int \mbred{\rho_R}(\mred{\vec{r}_s})\, d\mred{\vec{r}_s}\right] \nonumber \\
  &\quad \cdot
  \left[\int \mbgreen{\rho_G}(\mgreen{\vec{g}_t})\, d\mgreen{\vec{g}_t}\right]
  \left[\int (\mred{r}_t)_k\, \mbred{\rho_R}(\mred{\vec{r}_t})\, d\mred{\vec{r}_t}\right] \nonumber \\
  &= \sum_k (\mbgreen{\mu_G})_k \cdot \mred{c_R} \cdot \mgreen{c_G} \cdot (\mbred{\mu_R})_k \nonumber \\
  &= \mgreen{c_G} \cdot \mred{c_R} \cdot (\mbgreen{\mu_G} \cdot \mbred{\mu_R})
\end{align}

Rewriting:
\[
  \mathbb{E}[E]_{\text{perennial}} = \mblue{\Lambda}^2 \cdot (\mbgreen{\tilde{\mu}_G} \cdot \mbred{\tilde{\mu}_R}) = (\mathbb{E}[N])^2 \cdot (\mbgreen{\tilde{\mu}_G} \cdot \mbred{\tilde{\mu}_R})
\]

\subsection[Asymmetric ephemeral derivation]{Asymmetric ephemeral derivation ($\mathbf{R}_0$, historical)}\label{sec:appendix-asymmetric-ephemeral}

\footnote{This derivation corresponds to an alternative ``asymmetric ephemeral'' model where source and target are sampled as a directed pair. The symmetric ephemeral model defined in the main text samples unordered pairs and evaluates all four potential edges.}

For the product-form edge intensity, this asymmetric rule samples directed interactions from PPP($\mbblue{\rho_{\mathcal{E}}}$) on $\mathcal{E}$ with:
\[
  \mbblue{\rho_{\mathcal{E}}}(s,t) = \frac{\mbblue{\rho}(s) \cdot \mbblue{\rho}(t)}{2 \mblue{\Lambda}}
\]

\textbf{Verification of total mass:}
\[
  \int_{\mblue{\mathcal{E}}} \mbblue{\rho_{\mathcal{E}}} = \frac{1}{2 \mblue{\Lambda}} \left(\int_{\mblue{\Omega}} \mbblue{\rho}\right)^2 = \frac{\mblue{\Lambda}^2}{2 \mblue{\Lambda}} = \frac{\mblue{\Lambda}}{2} \;\checkmark
\]

\textbf{Expected edge count.} By Campbell's theorem:
\begin{align}
  \mathbb{E}[E]_{\text{asym}} &= \int_{\mblue{\mathcal{E}}} (\mgreen{\vec{g}_s} \cdot \mred{\vec{r}_t})\, \mbblue{\rho_{\mathcal{E}}}(s, t)\, ds\, dt \nonumber \\
  &= \frac{1}{2 \mblue{\Lambda}} \int_{\mblue{\mathcal{E}}} (\mgreen{\vec{g}_s} \cdot \mred{\vec{r}_t})\, \mbblue{\rho}(s)\, \mbblue{\rho}(t)\, ds\, dt
\end{align}

The integral is identical to the perennial case, yielding $\mblue{\Lambda}^2 (\mbgreen{\tilde{\mu}_G} \cdot \mbred{\tilde{\mu}_R})$. Therefore:
\[
  \mathbb{E}[E]_{\text{asym}} = \frac{\mblue{\Lambda}^2 (\mbgreen{\tilde{\mu}_G} \cdot \mbred{\tilde{\mu}_R})}{2 \mblue{\Lambda}} = \frac{\mblue{\Lambda}}{2} (\mbgreen{\tilde{\mu}_G} \cdot \mbred{\tilde{\mu}_R}) = \frac{\mathbb{E}[N]}{2} (\mbgreen{\tilde{\mu}_G} \cdot \mbred{\tilde{\mu}_R})
\]

\subsection[Symmetric ephemeral derivation]{Symmetric ephemeral derivation ($\mathbf{R}_0$)}

In the symmetric ephemeral model, we sample $M \sim \text{Poisson}(\mblue{\Lambda}/2)$ unordered pairs. Each pair $\{i, j\}$ has positions drawn i.i.d.\ from $\mbblue{\rho}/\mblue{\Lambda}$ and contributes four potential edges with probabilities:
\begin{itemize}
  \item $i \to j$: $\mgreen{\vec{g}_i} \cdot \mred{\vec{r}_j}$
  \item $j \to i$: $\mgreen{\vec{g}_j} \cdot \mred{\vec{r}_i}$
  \item $i \to i$: $\mgreen{\vec{g}_i} \cdot \mred{\vec{r}_i}$
  \item $j \to j$: $\mgreen{\vec{g}_j} \cdot \mred{\vec{r}_j}$
\end{itemize}

Under product intensity, the expected edges per pair is:
\begin{align}
  \mathbb{E}[\text{edges per pair}] &= \mathbb{E}[\mgreen{\vec{g}_i} \cdot \mred{\vec{r}_j}] + \mathbb{E}[\mgreen{\vec{g}_j} \cdot \mred{\vec{r}_i}] + \mathbb{E}[\mgreen{\vec{g}_i} \cdot \mred{\vec{r}_i}] + \mathbb{E}[\mgreen{\vec{g}_j} \cdot \mred{\vec{r}_j}] \nonumber \\
  &= 2(\mbgreen{\tilde{\mu}_G} \cdot \mbred{\tilde{\mu}_R}) + 2(\mbgreen{\tilde{\mu}_G} \cdot \mbred{\tilde{\mu}_R}) = 4(\mbgreen{\tilde{\mu}_G} \cdot \mbred{\tilde{\mu}_R})
\end{align}

Therefore:
\[
  \mathbb{E}[E]_{\text{ephemeral}} = \mathbb{E}[M] \cdot 4(\mbgreen{\tilde{\mu}_G} \cdot \mbred{\tilde{\mu}_R}) = \frac{\mblue{\Lambda}}{2} \cdot 4(\mbgreen{\tilde{\mu}_G} \cdot \mbred{\tilde{\mu}_R}) = 2 \mblue{\Lambda} (\mbgreen{\tilde{\mu}_G} \cdot \mbred{\tilde{\mu}_R})
\]

\subsection{Ratio of expected edges}

\[
  \frac{\mathbb{E}[E]_{\text{perennial}}}{\mathbb{E}[E]_{\text{ephemeral}}} = \frac{\mblue{\Lambda}^2 (\mbgreen{\tilde{\mu}_G} \cdot \mbred{\tilde{\mu}_R})}{2 \mblue{\Lambda} (\mbgreen{\tilde{\mu}_G} \cdot \mbred{\tilde{\mu}_R})} = \frac{\mblue{\Lambda}}{2} = \frac{\mathbb{E}[N]}{2}
\]

This ratio reflects the fundamentally different generative mechanisms:
\begin{itemize}
  \item Perennial: $N^2$ interaction opportunities from all pairs of persistent nodes
  \item Ephemeral: $M = N/2$ interaction pairs, each contributing up to 4 edges
\end{itemize}

The scaling difference persists: perennial produces $O(N^2)$ edges (dense), ephemeral produces $O(N)$ edges (sparse).

\subsection{Derivation of overlap probability}\label{appendix:overlap}

For the intermediate regime $\mathbf{R}_\eta$, we derive the probability that two independently sampled individuals have overlapping lifetimes.

\textbf{Setup.} Two individuals with:
\begin{itemize}
  \item Birth times $T_1, T_2 \stackrel{\text{i.i.d.}}{\sim} \text{Uniform}(0, W)$
  \item Lifetimes $\tau_1, \tau_2 \stackrel{\text{i.i.d.}}{\sim} \text{Exp}(\eta)$ (exponential with mean $\eta$)
\end{itemize}

Entity $i$ is alive during $[T_i, T_i + \tau_i]$. The intervals overlap iff $\max(T_1, T_2) < \min(T_1 + \tau_1, T_2 + \tau_2)$.

\textbf{Derivation.} By symmetry, we condition on $T_1 \leq T_2$:
\[
  P(\text{overlap}) = 2 \cdot P(T_2 < T_1 + \tau_1 \mid T_1 \leq T_2) = 2 \cdot P(\Delta \in [0, \tau_1))
\]
where $\Delta = T_2 - T_1$.

The gap $\Delta$ has triangular density on $[-W, W]$:
\[
  f_\Delta(\delta) = \frac{W - |\delta|}{W^2}
\]

For $\Delta \geq 0$, we need $\Delta < \tau_1$. Conditioning on $\tau_1 = t$:
\[
  P(\Delta \in [0, t)) = \begin{cases}
    (W t - t^2 / 2) / W^2 & \text{if } t \leq W, \\
    1/2 & \text{if } t > W
  \end{cases}
\]

Integrating over $\tau_1 \sim \text{Exp}(\eta)$:
\[
  P(\text{overlap}) = 2 \left[\int_0^W \frac{W t - t^2 / 2}{W^2} \cdot \frac{e^{-t/\eta}}{\eta}\, dt + \int_W^\infty \frac{1}{2} \cdot \frac{e^{-t/\eta}}{\eta}\, dt\right]
\]

The second integral evaluates to $\frac{1}{2} e^{-W/\eta}$.

For the first integral, using standard formulas for $\int t^n e^{-t/\eta}\, dt$ and letting $u = W / \eta$:
\[
  \int_0^W t\, e^{-t/\eta}\, dt = \eta^2 [1 - (1 + u) e^{-u}]
\]
\[
  \int_0^W t^2\, e^{-t/\eta}\, dt = 2 \eta^3 [1 - (1 + u + u^2 / 2) e^{-u}]
\]

After algebra, combining terms:
\[
  p_{\text{overlap}}(\eta, W) = \frac{2}{u^2} (u - 1 + e^{-u})
\]

\textbf{Verification of limits.}

\emph{Long-lived} ($\eta \gg W$, so $u \to 0$): Using Taylor expansion $e^{-u} \approx 1 - u + u^2 / 2$:
\[
  u - 1 + e^{-u} \approx u^2 / 2
\]
\[
  p_{\text{overlap}} \approx \frac{2}{u^2} \cdot \frac{u^2}{2} = 1
\]

\emph{Ephemeral} ($\eta \ll W$, so $u \to \infty$):
\[
  u - 1 + e^{-u} \approx u
\]
\[
  p_{\text{overlap}} \approx \frac{2u}{u^2} = \frac{2}{u} = \frac{2\eta}{W} \to 0
\]

\subsection{A note on self-loops}\label{appendix:selfloops}

Throughout this work, we have been a bit sloppy about whether interactions, connections, and edges happen only between distinct individuals or not. Usually this translates to either having $N(N-1)$ or $N^2$ links, and the difference is often negligible for large enough graphs.

\subsubsection{Ephemeral rule}

In the symmetric ephemeral rule defined in the main text, each interaction pair $\{i, j\}$ naturally generates self-loop opportunities: $i \to i$ and $j \to j$ are evaluated alongside the cross-edges $i \to j$ and $j \to i$. Self-loops are thus included by construction in the ephemeral model.

\footnote{In the asymmetric ephemeral variant (see Section~\ref{sec:appendix-asymmetric-ephemeral}), where directed interactions $(s, t)$ are sampled from a continuous PPP on $\mblue{\mathcal{E}}$, self-loops are automatically excluded because the diagonal $\{s = t\}$ has measure zero under any absolutely continuous intensity.}

\subsubsection{Perennial rule}

The perennial rule samples $N$ nodes, then considers all $N^2$ ordered pairs of nodes as edge opportunities. Self-loops correspond to pairs $(i, i)$.

The second factorial moment formula \cite[Sec.~5.4, Ex.~6.1(a)]{daley2003introduction} we use for perennial derivations:
\[
  \mathbb{E}\!\left[\sum_{i \neq j} f(x_i, x_j)\right] = \iint f(x,y)\, \lambda(dx)\, \lambda(dy)
\]
naturally counts only distinct ordered pairs. This identity is stated in terms of the factorial moment measure and is valid irrespective of whether $\lambda$ has atoms.

If one includes self-loops in the perennial model (consistent with the generative interpretation ``all ordered pairs, including $i=i$''), an additional Campbell term is required:
\[
  \mathbb{E}[\text{self-loops}] = \int (\mgreen{\vec{g}} \cdot \mred{\vec{r}})\, \mbblue{\rho}(d\mgreen{\vec{g}}, d\mred{\vec{r}})
\]

Under product intensity $\mbblue{\rho} = \mbgreen{\rho_G} \otimes \mbred{\rho_R}$, this simplifies to:
\[
  \mathbb{E}[\text{self-loops}] = \mathbb{E}[N] \cdot (\mbgreen{\tilde{\mu}_G} \cdot \mbred{\tilde{\mu}_R})
\]

This is $O(N)$ compared to $O(N^2)$ for distinct-pair edges, so for moderate-to-large $\mathbb{E}[N]$ the self-loop contribution is negligible.

Therefore, with product intensity and including self-loops explicitly:
\[
  \mathbb{E}[E]_{\text{perennial+loops}} = \mathbb{E}[E]_{\text{perennial}} + \mathbb{E}[\text{self-loops}] = (\mblue{\Lambda}^2 + \mblue{\Lambda}) \cdot (\mbgreen{\tilde{\mu}_G} \cdot \mbred{\tilde{\mu}_R})
\]
and the exact ratio against symmetric ephemeral is:
\[
  \frac{\mathbb{E}[E]_{\text{perennial+loops}}}{\mathbb{E}[E]_{\text{ephemeral}}} = \frac{\mblue{\Lambda} + 1}{2}
\]
which reduces to $\mblue{\Lambda}/2$ asymptotically.

\subsection{Per-dimension concentration for boundary positioning}\label{appendix:perdim}

When using Gaussian kernels centered near the boundary of $B^d_+$, truncation biases the effective mean toward the interior. For species that should be precisely positioned at boundary regions (e.g., producers with resource coordinate near the edge of niche space), this can be problematic.

\textbf{Per-dimension concentration.} Instead of a scalar concentration parameter $\kappa$, use a vector $\bm{\kappa} = (\kappa_1, \dots, \kappa_d)$:
\[
  \rho_m(\mgreen{\vec{g}}) \propto \prod_{i=1}^{d} \exp\!\left(-\kappa_{g,i} (\mgreen{g}_i - \mu_{g,i})^2 / 2\right) \cdot \mathds{1}(\mgreen{\vec{g}} \in B^d_+)
\]

The interpretation is that $\sigma_i = 1 / \sqrt{\kappa_i}$ gives the standard deviation in dimension $i$:

\begin{table}[htbp]
\centering
\begin{tabular}{|c|c|c|}
\hline
$\kappa_i$ & \textbf{Interpretation} & \textbf{Std.\ dev.\ $\sigma_i$} \\
\hline
30 & Normal variation & $\approx 0.18$ \\
100 & Tight & $\approx 0.10$ \\
500 & Very tight & $\approx 0.045$ \\
1000 & Nearly fixed & $\approx 0.032$ \\
\hline
\end{tabular}
\end{table}

\textbf{Ecological example (4D).} Producers with strong resource presence in dimension 1 and consumer role in ``null'' dimension 4:
\[
  \bm{\mu}_g = [0.90, 0.10, 0.02, 0.00], \quad \bm{\kappa}_g = [500, 30, 30, 30]
\]
\[
  \bm{\mu}_r = [0.00, 0.00, 0.00, 0.95], \quad \bm{\kappa}_r = [30, 30, 30, 500]
\]

Dimensions 1 of $\mgreen{\vec{g}}$ and 4 of $\mred{\vec{r}}$ are \emph{structural} (high $\kappa$, defining trophic level), others allow within-species variation.

\textbf{PDE compatibility.} Using high $\kappa$ rather than fixed values maintains smooth distributions compatible with PDE evolution. Under isotropic diffusion:
\[
  \kappa_i(t) = \kappa_i(0) / (1 + 2 \nu \kappa_i(0)\, t)
\]

High-$\kappa$ dimensions decay slower, preserving structural traits while allowing variable traits to spread faster.

\textbf{Reduced boundary bias.} Narrow Gaussians (high $\kappa$) near boundaries lose minimal mass to truncation, so effective mean $\approx$ specified mean. Wide Gaussians suffer more bias as significant mass extends beyond the boundary.

\subsection{Spectral decomposition of the bound heat operator}\label{appendix:spectral}

We develop the singular value decomposition of the bound heat operator in detail. The mathematical foundations follow the theory of Hilbert--Schmidt integral operators \cite{schmidt1907theorie} \cite[Ch.~VI]{reed1980methods} \cite[Ch.~28]{lax2002functional}.

\subsubsection{Setup}

Under product intensity $\mbblue{\rho} = \mbgreen{\rho_G} \otimes \mbred{\rho_R}$, the bound heat density is:
\[
  \overline{h}(\mgreen{\vec{g}}, \mred{\vec{r}}) = (\mgreen{\vec{g}} \cdot \mred{\vec{r}}) \cdot \mbgreen{\rho_G}(\mgreen{\vec{g}}) \cdot \mbred{\rho_R}(\mred{\vec{r}})
\]

Define component functions $\alpha_k: B^d_+ \to \mathbb{R}$ and $\beta_k: B^d_+ \to \mathbb{R}$ by:
\[
  \alpha_k(\mgreen{\vec{g}}) = \mgreen{g}_k \cdot \mbgreen{\rho_G}(\mgreen{\vec{g}}), \quad \beta_k(\mred{\vec{r}}) = \mred{r}_k \cdot \mbred{\rho_R}(\mred{\vec{r}})
\]

Then:
\[
  \overline{h}(\mgreen{\vec{g}}, \mred{\vec{r}}) = \sum_{k=1}^d \alpha_k(\mgreen{\vec{g}}) \beta_k(\mred{\vec{r}})
\]

This represents the bound heat as a sum of $d$ separable (rank-1) kernels.

\subsubsection{Gram matrices}

The spectral structure depends on the Gram matrices of $\{\alpha_k\}$ and $\{\beta_k\}$:
\[
  A_{jk} = \langle \alpha_j, \alpha_k \rangle_{L^2(\mgreen{B^d_+})} = \int_{\mgreen{B^d_+}} \mgreen{g}_j \mgreen{g}_k [\mbgreen{\rho_G}(\mgreen{\vec{g}})]^2 \, d\mgreen{\vec{g}}
\]
\[
  B_{jk} = \langle \beta_j, \beta_k \rangle_{L^2(\mred{B^d_+})} = \int_{\mred{B^d_+}} \mred{r}_j \mred{r}_k [\mbred{\rho_R}(\mred{\vec{r}})]^2 \, d\mred{\vec{r}}
\]

Both $A$ and $B$ are $d \times d$ symmetric positive semi-definite matrices.

\subsubsection{Singular value decomposition}

\begin{theorem}[Singular value decomposition of bound heat operator]\label{thm:svd}
The bound heat operator $\overline{T}$ has at most $d$ non-zero singular values. Let $A = U_A \Sigma_A U_A^\top$ and $B = U_B \Sigma_B U_B^\top$ be eigendecompositions. Define:
\[
  C = \Sigma_A^{1/2} U_A^\top U_B \Sigma_B^{1/2}
\]

The singular values of $\overline{T}$ are the singular values of the $d \times d$ matrix $C$.
\end{theorem}

\emph{Proof sketch.} The operator $\overline{T}$ maps $f \in L^2(\mred{B^d_+})$ to:
\[
  (\overline{T} f)(\mgreen{\vec{g}}) = \sum_{k=1}^d \alpha_k(\mgreen{\vec{g}}) \langle \beta_k, f \rangle
\]

This factors as $\overline{T} = \mathcal{A} \mathcal{B}^*$ where $\mathcal{A}: \mathbb{R}^d \to L^2(\mgreen{B^d_+})$ is $\mathcal{A} \mathbf{c} = \sum_k c_k \alpha_k$ and $\mathcal{B}: \mathbb{R}^d \to L^2(\mred{B^d_+})$ is $\mathcal{B} \mathbf{c} = \sum_k c_k \beta_k$.

The operators $\mathcal{A}^* \mathcal{A}$ and $\mathcal{B}^* \mathcal{B}$ are represented by the Gram matrices $A$ and $B$ respectively. The SVD of $\overline{T}$ follows from the SVD of $\mathcal{A}$ and $\mathcal{B}$ combined. \hfill $\square$

\subsubsection{Gaussian intensity example}

When $\mbgreen{\rho_G}$ and $\mbred{\rho_R}$ are truncated Gaussians with means $\boldsymbol{\mu}_G, \boldsymbol{\mu}_R$ and covariance matrices $\Sigma_G, \Sigma_R$, the Gram matrices involve Gaussian moment integrals.

For a scalar Gaussian $\rho(x) \propto \exp(-(x - \mu)^2 / (2\sigma^2))$ on $\mathbb{R}$ (ignoring truncation for simplicity):
\[
  \int x^2 [\rho(x)]^2 \, dx \propto \sigma \cdot (\mu^2 + \sigma^2 / 2)
\]

The Gram matrix entries are weighted second moments of the intensity, capturing both the centering ($\mu$) and spread ($\sigma$) of the population in each coordinate.

For isotropic Gaussians centered at the origin with $\sigma^2 I$ covariance, the Gram matrices are proportional to identity: $A \propto \sigma^2 I$, $B \propto \sigma^2 I$. The singular values are then all equal, reflecting the rotational symmetry.

For anisotropic or off-center Gaussians, the Gram matrices develop structure, and singular values separate. The dominant singular value corresponds to the direction of maximal intensity-weighted coordinate product.

\subsubsection{Hilbert--Schmidt norm}

The Hilbert--Schmidt norm of $\overline{T}$ satisfies:
\[
  \|\overline{T}\|_{HS}^2 = \sum_{n=1}^d \sigma_n^2 = \operatorname{tr}(AB)
\]

This provides a scalar measure of total interaction intensity, computable directly from the Gram matrices without explicit Singular Value decomposition.

\subsubsection{Connection to total bound heat}

The total bound heat is:
\[
  \overline{\mathcal{H}}(B^d_+, B^d_+) = \int \int \overline{h}(\mgreen{\vec{g}}, \mred{\vec{r}}) \, d\mgreen{\vec{g}} \, d\mred{\vec{r}} = \sum_k \langle \alpha_k, \mathds{1} \rangle \langle \beta_k, \mathds{1} \rangle = \mbgreen{\mu_G} \cdot \mbred{\mu_R}
\]

where $\mbgreen{\mu_G} = \int \mgreen{\vec{g}} \mbgreen{\rho_G} \, d\mgreen{\vec{g}}$ and $\mbred{\mu_R} = \int \mred{\vec{r}} \mbred{\rho_R} \, d\mred{\vec{r}}$ are the (unnormalized) mean positions. This is the sum over coordinates of products of intensity-weighted means, the ``bulk'' interaction propensity.

\subsection{The desire operator}

\subsubsection{Spectral decomposition of the desire operator}\label{appendix:spectral-desire}

\begin{proposition}[Spectral structure of Desire]\label{prop:spectral-consistency}
  Let $\tilde{D}: L^2(\mred{B^d_+}, \tilde{\rho}_R) \to L^2(\mgreen{B^d_+}, \tilde{\rho}_G)$ be the desire operator.
  The squared singular values $\sigma_k(\tilde{D})^2$ are exactly the eigenvalues of the matrix product $\Sigma_G \Sigma_R$, where $\Sigma_G$ and $\Sigma_R$ are the second moment matrices of the normalized intensities.
\end{proposition}

\begin{proof}
  The singular values of $\tilde{D}$ are the square roots of the eigenvalues of the self-adjoint composition $\tilde{D} \tilde{D}^*$. We derive the action of this composition explicitly.

  \textbf{Step 1: The Adjoint.}
  The adjoint operator $\tilde{D}^*: L^2(\mgreen{B^d_+}, \tilde{\rho}_G) \to L^2(\mred{B^d_+}, \tilde{\rho}_R)$ is defined by the duality condition $\langle \tilde{D} f, u \rangle_{\tilde{\rho}_G} = \langle f, \tilde{D}^* u \rangle_{\tilde{\rho}_R}$. Expanding the weighted inner products reveals that the adjoint has the symmetric form:
  \[
    (\tilde{D}^* u)(\mred{\vec{r}}) = \int_{\mgreen{B^d_+}} (\mred{\vec{r}} \cdot \mgreen{\vec{g}}) u(\mgreen{\vec{g}}) \tilde{\rho}_G(\mgreen{\vec{g}}) \, d\mgreen{\vec{g}}
  \]

  \textbf{Step 2: Finite Rank Subspace.}
  Notice that for any input function $f$, the output $(\tilde{D} f)(\mgreen{\vec{g}})$ is a linear projection onto the coordinate functions of $\mgreen{\vec{g}}$. Specifically:
  \[
    (\tilde{D} f)(\mgreen{\vec{g}}) = \mgreen{\vec{g}} \cdot w_f, \quad \text{where } w_f = \int_{\mred{B^d_+}} \mred{\vec{r}} f(\mred{\vec{r}}) \tilde{\rho}_R(\mred{\vec{r}}) \, d\mred{\vec{r}}
  \]
  Thus, the image of $\tilde{D}$ lies in the finite-dimensional subspace spanned by the coordinate maps $\{g_1, \ldots, g_d\}$. Any eigenfunction $u$ of $\tilde{D} \tilde{D}^*$ corresponding to a non-zero eigenvalue must lie in this subspace. We can therefore write the eigenfunction as $u(\mgreen{\vec{g}}) = \mgreen{\vec{g}} \cdot x$ for some vector $x \in \mathbb{R}^d$.

  \textbf{Step 3: Action of the Composition.}
  First, apply the adjoint $\tilde{D}^*$ to the ansatz $u(\mgreen{\vec{g}}) = \mgreen{\vec{g}} \cdot x$:
  \begin{align}
    (\tilde{D}^* u)(\mred{\vec{r}}) &= \int (\mred{\vec{r}} \cdot \mgreen{\vec{g}}) (\mgreen{\vec{g}} \cdot x) \tilde{\rho}_G(\mgreen{\vec{g}}) \, d\mgreen{\vec{g}} \nonumber \\
    &= \mred{\vec{r}} \cdot \left(\int \mgreen{\vec{g}} \mgreen{\vec{g}}^\top \tilde{\rho}_G(\mgreen{\vec{g}}) \, d\mgreen{\vec{g}}\right) x \nonumber \\
    &= \mred{\vec{r}} \cdot (\Sigma_G x) \nonumber
  \end{align}
  Next, apply the forward operator $\tilde{D}$ to this intermediate result $v(\mred{\vec{r}}) = \mred{\vec{r}} \cdot (\Sigma_G x)$:
  \begin{align}
    (\tilde{D} v)(\mgreen{\vec{g}}) &= \int (\mgreen{\vec{g}} \cdot \mred{\vec{r}}) (\mred{\vec{r}} \cdot (\Sigma_G x)) \tilde{\rho}_R(\mred{\vec{r}}) \, d\mred{\vec{r}} \nonumber \\
    &= \mgreen{\vec{g}} \cdot \left(\int \mred{\vec{r}} \mred{\vec{r}}^\top \tilde{\rho}_R(\mred{\vec{r}}) \, d\mred{\vec{r}}\right) (\Sigma_G x) \nonumber \\
    &= \mgreen{\vec{g}} \cdot (\Sigma_R \Sigma_G x) \nonumber
  \end{align}

  \textbf{Step 4: Eigenvalue Equation.}
  The operator eigenvalue equation $\tilde{D} \tilde{D}^* u = \sigma^2 u$ thus becomes the vector equation:
  \[
    \mgreen{\vec{g}} \cdot (\Sigma_R \Sigma_G x) = \sigma^2 (\mgreen{\vec{g}} \cdot x)
  \]
  Since this must hold for all $\mgreen{\vec{g}}$, it implies $\Sigma_R \Sigma_G x = \sigma^2 x$.

  \textbf{Conclusion:} The squared singular values $\sigma^2$ are the eigenvalues of $\Sigma_R \Sigma_G$ (which are identical to those of $\Sigma_G \Sigma_R$). Although this matrix product is generally non-symmetric, it is similar to the symmetric positive semi-definite matrix $\Sigma_G^{1/2} \Sigma_R \Sigma_G^{1/2}$, ensuring that all eigenvalues are real and non-negative.
\end{proof}

\textbf{Reality of singular values.}
The matrix product $M = \Sigma_R \Sigma_G$ appearing in the eigenvalue equation is generally not symmetric. To see why its eigenvalues are nevertheless real and non-negative, assume without loss of generality that $\Sigma_G$ is positive definite (i.e., the intensity $\tilde{\rho}_G$ has full-rank support).

Consider the similarity transformation using the square root matrix $\Sigma_G^{1/2}$:
\begin{align}
  S &= \Sigma_G^{1/2} M \Sigma_G^{-1/2} \nonumber \\
  &= \Sigma_G^{1/2} (\Sigma_R \Sigma_G) \Sigma_G^{-1/2} \nonumber \\
  &= \Sigma_G^{1/2} \Sigma_R \Sigma_G^{1/2} \nonumber
\end{align}

The resulting matrix $S$ is symmetric (as a product of symmetric matrices in a sandwich form). Furthermore, it is positive semi-definite: for any vector $v$, letting $u = \Sigma_G^{1/2} v$, we have:
\[
  v^\top S v = v^\top \Sigma_G^{1/2} \Sigma_R \Sigma_G^{1/2} v = u^\top \Sigma_R u \geq 0
\]
Since similar matrices share the same spectrum, the eigenvalues of $\Sigma_R \Sigma_G$ are identical to those of $S$, ensuring they are real and non-negative.

\subsubsection{Spectral consistency of the desire operator}

\begin{proposition}[Spectral Consistency for Multiple Independent IDPGs]\label{prop:multi-graph-consistency}
  Assume bounded latent support and regularity conditions ensuring local Lipschitz dependence of singular values on the empirical second-moment matrices. Let $G_1, \ldots, G_m$ be $m$ independent \emph{non-empty} directed graphs generated by the following truncated-size perennial protocol with intensity $\mblue{\rho}$ and total intensity $\mblue{\Lambda}$: for each graph $l$,
  \begin{enumerate}
    \item The number of nodes is Poisson conditioned to be positive: $N_l \sim \text{Poisson}(\mblue{\Lambda})$ given $N_l \geq 1$.
    \item Latent positions $(\mgreen{\vec{g}_i}, \mred{\vec{r}_i})$ are i.i.d.\ draws from $\tilde{\rho}$.
    \item Directed edges form independently: $A^{(l)}_{ij} \sim \text{Bernoulli}(\mgreen{\vec{g}_i} \cdot \mred{\vec{r}_j})$.
  \end{enumerate}

  Let $\overline{\sigma}_k = \frac{1}{m} \sum_{l=1}^m \sigma_k(A^{(l)}) / N_l$ be the averaged spectral estimator. Then:
  \[
    |\overline{\sigma}_k - \sigma_k(\tilde{D})| = \mathcal{O}\!\left(\frac{1}{\sqrt{\mblue{\Lambda}}}\right) + \mathcal{O}_p\!\left(\frac{1}{\sqrt{m \mblue{\Lambda}}}\right)
  \]
\end{proposition}

\begin{proof}
  Let $\hat{\sigma}_k^{(l)} = \sigma_k(A^{(l)}) / N_l$. We decompose the error into bias and fluctuation:
  \[
    |\overline{\sigma}_k - \sigma_k(\tilde{D})| \leq \underbrace{|\mathbb{E}[\hat{\sigma}_k^{(l)}] - \sigma_k(\tilde{D})|}_{\text{Bias}} + \underbrace{\left|\frac{1}{m} \sum_{l=1}^m \hat{\sigma}_k^{(l)} - \mathbb{E}[\hat{\sigma}_k^{(l)}]\right|}_{\text{Fluctuation}}
  \]

  \textbf{Step 1: The Bias.}
  The bias measures how far the expected finite-graph estimator is from the continuous operator limit.
  As established in Proposition~\ref{prop:spectral-consistency}, the true singular values are determined by the population second moments $\Sigma_G$ and $\Sigma_R$.
  Similarly, the singular values of the probability matrix $P^{(l)} / N_l$ are determined by the \emph{sample} second moments (e.g., $\hat{\Sigma}_G = \frac{1}{N_l} \sum_{i=1}^{N_l} \mgreen{\vec{g}_i} \mgreen{\vec{g}_i}^\top$).

  Since $\hat{\Sigma}_G$ is an average of $N_l$ i.i.d.\ bounded rank-1 matrices, the Central Limit Theorem guarantees it converges to $\Sigma_G$ with error scaling as $1/\sqrt{N_l}$. Combining this with the Bernoulli noise (which also scales as $1/\sqrt{N_l}$), the conditional bias is:
  \[
    |\mathbb{E}[\hat{\sigma}_k^{(l)} \mid N_l] - \sigma_k(\tilde{D})| \leq \frac{C}{\sqrt{N_l}}
  \]

  Averaging this over the truncated Poisson law (where $N_l$ is Poisson conditioned on $N_l \geq 1$, still centered at scale $\mblue{\Lambda}$):
  \[
    \text{Bias} = \mathcal{O}\!\left(\frac{1}{\sqrt{\mblue{\Lambda}}}\right)
  \]

  \textbf{Step 2: The Fluctuation (Averaging Graphs).}
  The estimators $\hat{\sigma}_k^{(1)}, \ldots, \hat{\sigma}_k^{(m)}$ are $m$ independent random variables.
  The variance of a single estimator scales as $\operatorname{Var}(\hat{\sigma}_k^{(l)}) = \mathcal{O}(1 / \mblue{\Lambda})$.
  Averaging $m$ such copies reduces the standard deviation by $1/\sqrt{m}$:
  \[
    \text{Fluctuation} = \mathcal{O}_p\!\left(\sqrt{\frac{\operatorname{Var}(\hat{\sigma}_k^{(1)})}{m}}\right) = \mathcal{O}_p\!\left(\frac{1}{\sqrt{m \mblue{\Lambda}}}\right)
  \]

  \textbf{Conclusion:}
  The error is dominated by the bias (finite $\mblue{\Lambda}$) unless the system is dense. Increasing $m$ reduces the fluctuation but cannot fix the resolution limit imposed by $\mblue{\Lambda}$.
\end{proof}

The theoretical predictions of this proposition, together with Theorem~\ref{thm:spectral-consistency-adjacency}, are verified numerically in Figure~\ref{fig:desire-spectral}.

\clearpage

\subsection{A review of the notation adopted}

Throughout the paper we tried to adopt a consistent notation, although we sometime abused it or used a symbol with a specific meaning in a particular section, where it was clear by the context what we meant. You can find the notation collected in the following table.

\begin{longtable}{p{0.22\textwidth} p{0.65\textwidth}}
\hline
\textbf{Notation} & \textbf{Meaning} \\
\hline
\endfirsthead
\hline
\textbf{Notation} & \textbf{Meaning} \\
\hline
\endhead
\hline
\endfoot
$G$ & A graph \\
$V$ & The set of nodes of a graph \\
$E$ & The set of edges of a graph: (ordered) pairs of nodes \\
$i$, $j$ & Two nodes in a graph, or individuals \\
$s$, $t$ & Generally, the source and target of an interaction, connection, edge, \ldots \\
$i \to j$ & The edge given by the ordered pair $(i,j)$ \\
$B^d_+$ & The slice of the $d$-dimensional ball with norm 1, centred in the origin, and having all positive coordinates \\
$\mgreen{G}$, $\mred{R}$ & Respectively the position spaces that define an individual propensity to give and receive connections \\
$\mgreen{B^d_+}$, $\mred{B^d_+}$ & A canonical choice for an absolute coordinate version of the giving and receiving spaces \\
$\mgreen{\vec{g}_i}$, $\mred{\vec{r}_i}$ & The position of an individual $i$, respectively in the \emph{giving} and \emph{receiving} spaces \\
$\mgreen{c_G}$, $\mred{c_R}$ & Marginal total intensities (total mass in green/red spaces) \\
$\mbgreen{\mu_G}$, $\mbred{\mu_R}$ & Intensity-weighted mean positions \\
$\mblue{\Omega}$ & The full position space, that is $\mgreen{G} \times \mred{R}$, canonically $\mblue{\Omega} = \mgreen{B^d_+} \times \mred{B^d_+}$ \\
$\mblue{\mathcal{E}}$ & The space of edges, that is $\mblue{\mathcal{E}} = \mblue{\Omega} \times \mblue{\Omega}$ \\
$\mblue{\rho}$ & The intensity on $\mblue{\Omega}$ \\
$\mblue{\rho_{\mblue{\mathcal{E}}}}$ & The intensity on $\mblue{\mathcal{E}}$ \\
$\mblue{\Lambda}$ & Total intensity over $\mblue{\Omega}$ \\
$\mbblue{\tilde{\rho}}$ & The normalized probability measure on the latent space ($\mbblue{\tilde{\rho}} = \mbblue{\rho} / \mblue{\Lambda}$) \\
Interaction & The precursor of a connection \\
Connection & An established interaction \\
$K(s,t)$ & The affinity kernel, giving the probability of connection between interacting individuals $s$ and $t$ \\
$\mathbf{R}$ & Realization rule determining how intensity becomes interactions: $\mathbf{R}_\infty$ (perennial), $\mathbf{R}_0$ (ephemeral), $\mathbf{R}_\eta$ (intermediate) \\
$\eta$; $W$; $u$ & Mean lifetime of an individual (governs the transition between regimes); Observation window duration; The ratio $W / \eta$ \\
$p_{\text{overlap}}$ & Probability that two independent lifetimes overlap (function of $\eta$ and $W$) \\
$W(u, v)$ & Digraphon kernel function $W: [0,1]^2 \to [0,1]$ \\
$\phi$ & Measure-preserving map $\phi: [0,1] \to \mblue{\Omega}$ from digraphons' label space to position space \\
$h(s, t)$ & Raw heat density: $h = K(s,t) \cdot \mbblue{\rho}(s) \cdot \mbblue{\rho}(t)$ \\
$\mathcal{H}(A, B)$ & Raw heat map: expected edges from $A$ to $B$ under perennial sampling \\
$\overline{\mathcal{H}}$, $\overline{h}$ & Bound heat map and density (projection onto active coordinates $\mgreen{G} \times \mred{R}$) \\
$\overline{T}$ & Bound Heat Operator mapping $L^2(\mred{R})$ to $L^2(\mgreen{G})$ \\
$\tilde{D}$, $\tilde{D}^*$ & Desire Operator: integral operator with kernel $K$ weighted by population densities; and its adjoint \\
$\Sigma_G$, $\Sigma_R$; $\hat{\Sigma}_G$ & Population second moment matrices (Gramians of the desire operator); Sample second moment matrix derived from observed nodes \\
$\mathcal{L}$ & Continuous Laplacian operator defined by the heat map \\
$\sigma_k$; $\overline{\sigma}_k$ & Singular values; Averaged singular value estimator across $m$ independent graphs \\
\hline
\end{longtable}

\end{document}